\documentclass[twoside]{article}

\usepackage[accepted]{aistats2021}
\usepackage[round]{natbib}

\usepackage[colorlinks=true,citecolor=blue, urlcolor=blue]{hyperref}       
\usepackage{url}            
\usepackage{amsfonts}       
\usepackage{amsmath,eqnarray} 
\usepackage{array}
\usepackage{enumerate,textpos}
\usepackage[noend]{algorithmic} 
\usepackage{algorithm}
\usepackage{amssymb,amsthm,amsfonts,nicefrac}
\usepackage{color}
\usepackage{bm}
\usepackage{graphicx}
\usepackage{caption}
\usepackage{subcaption}
\usepackage{float}
\usepackage{bbm}
\usepackage{mathrsfs}

\newtheorem{lemma}{Lemma}[section]

\newtheorem{theorem}[lemma]{Theorem}

\newtheorem{definition}[lemma]{Definition}
\newtheorem{assumption}[lemma]{Assumption}

\definecolor{darkgray}{rgb}{0.5, 0.5, 0.5}

\renewcommand{\hat}{\widehat}

\newcommand{\removed}[1]{}

\newcommand{\cA}{\mathcal{A}}

\newcommand{\cE}{\mathcal{E}}

\newcommand{\cT}{\mathcal{T}}

\newcommand{\bP}{\mathbb{P}}

\newcommand{\A}{\mathrm{A}}

\newcommand{\e}{\mathrm{e}}

\newcommand{\I}{\mathrm{I}}

\newcommand{\Ber}{\text{Ber}}
\newcommand{\Unif}{\text{Unif}}

\renewcommand{\log}[1]{\operatorname{log}\left(#1\right)}
\renewcommand{\exp}[1]{\operatorname{exp}\left(#1\right)}

\newcommand{\prob}[1]{\mathbb{P}\left[#1\right]}

\DeclareMathOperator*{\argmax}{\rm argmax}
\DeclareMathOperator*{\argmin}{\rm argmin}

\mathchardef\mhyphen="2D

\newcommand{\ignore}[1]{}

\makeatletter
\g@addto@macro{\UrlBreaks}{\UrlOrds}
\makeatother

\begin{document}

\runningtitle{Corralling Stochastic Bandit Algorithms}

\twocolumn[

\aistatstitle{Corralling Stochastic Bandit Algorithms}

\aistatsauthor{ Raman Arora \And Teodor V. Marinov \And  Mehryar Mohri }

\aistatsaddress{ Johns Hopkins University\\\href{mailto:arora@cs.jhu.edu}{arora@cs.jhu.edu}  \And  Johns Hopkins University\\\href{mailto:tmarino2@jhu.edu}{tmarino2@jhu.edu} \And Courant Institute and Google Research\\\href{mailto:mohri@google.com}{mohri@google.com} } ]

\begin{abstract}
  We study the problem of corralling stochastic bandit algorithms,
  that is combining multiple bandit algorithms designed for a
  stochastic environment, with the goal of devising a corralling
  algorithm that performs almost as well as the best base algorithm.
  We give two general algorithms for this setting, which we show
  benefit from favorable regret guarantees. We show that the regret of
  the corralling algorithms is no worse than that of the best
  algorithm containing the arm with the highest reward, and depends on
  the gap between the highest reward and other rewards.
\end{abstract}

\section{Introduction}

We study the problem of \emph{corralling} multi-armed bandit
algorithms in a stochastic environment. This consists of selecting, at
each round, one out of a fixed collection of bandit algorithms and
playing the action returned by that algorithm.
Note that the corralling algorithm does not
directly select an arm, but only a base algorithm. It 
never requires knowledge of the action set of each base algorithm. 
The objective of the
corralling algorithm is to achieve a large cumulative reward or a small
pseudo-regret, over the course of its interactions with the
environment.
This problem was first introduced and studied by
\cite{agarwal2016corralling}. Here, we are guided by the same
motivation but consider the stochastic setting and seek more favorable
guarantees.  Thus, we assume that the reward, for each arm, is drawn
from an unknown distribution.

In the simplest setting of our study, we assume that each base bandit
algorithm has access to a distinct set of arms.  This scenario appears
in several applications. As an example, consider the online contractual 
display ads allocation problem \citep{dispads2020}: when users visit a 
website, say some page of the online site of a national newspaper, 
an ads allocation algorithm
chooses an ad to display at each specific slot with the goal of
achieving the largest value. This could be an ad for a clothing item,
which could be meant for the banner of the online front page of that
newspaper.
To do so, the ads allocation algorithm chooses one out of a large set
of advertisers, 
each a clothing brand or company in this case, which have signed a
contract with the ads allocation company. Each
clothing company has its own marketing strategy and thus its own
bandit algorithm with its own separate set of clothing items or
arms. There is no sharing of information between these companies which
are typically competitors.  Furthermore, the ads allocation algorithm
is not provided with any detailed information about the base bandits
algorithms of these companies, since that is proprietary information
private to each company. The allocation algorithm cannot choose a
specific arm or clothing item, it can only choose a base
advertiser. The number of ads or arms can be very large. The number of
advertisers can also be relatively large in practice, depending on the
domain. The number of times the ads allocation is run is in the order
of millions or even billions per day, depending on the category of
items.

A similar problem arises with online mortgage broker companies
offering loans to new applicants. The mortgage broker algorithm must
choose a bank, each with different mortgage products.  The broker
brings a new application exclusively to one of the banks, as part of
the contract, which also entitles them to incentives.  The bank's
online algorithm can be a bandit algorithm proposing a product, and
the details of the algorithm are not accessible to the broker; for
instance, the bank's credit rate and incentives may depend on the
financial and credit history of the applicant. The number of mortgage
products is typically fairly large, and the number of online loan
requests per day is in the order of several thousands.  Other
instances of this problem appear when an algorithm can only select one
of multiple bandit algorithms and, for privacy or regulatory reasons,
it cannot directly select an arm or receive detailed information about
the base algorithms.

In the most general setting we study, there may be an arbitrary
sharing of arms between the bandit algorithms.  We will only assume that 
only one algorithm has

access to the arm with maximal expected
reward, which implies a positive gap between the expected reward of the
best arm of any algorithm and that of the best algorithm. This is
because we seek to devise a corralling algorithm with favorable
gap-dependent pseudo-regret guarantees.

\textbf{Related work.} The previous work most closely related to this
study is the seminal contribution by \cite{agarwal2016corralling} who
initiated the general problem of corralling bandit algorithms. The
authors gave a general algorithm for this problem, which is an
instance of the generic Mirror Descent algorithm with an appropriate
mirror map ({\sc Log-Barrier-OMD}),
\citep{FosterLiLykourisSridharanTardos2016,WeiLuo2018}, and which
includes a carefully constructed non-decreasing step-size schedule,
also used by \cite{bubeck2017kernel}. The algorithm of
\cite{agarwal2016corralling}, however, cannot in general achieve
regret bounds better than $\tilde O(\sqrt{T})$ in the time horizon,
unless optimistic instance-dependent regret bounds are known for the
corralled algorithms. Prior to their work,
\cite{arora2012deterministic} presented an algorithm for learning
deterministic Markov decision processes (MDPs) with adversarial
rewards, using an algorithm for corralling bandit linear optimization
algorithms. In an even earlier work, \cite{MaillardMunos2011}
attempted to corral {\sc EXP3} algorithms
\citep{auer2002nonstochastic} with a top algorithm that is a slightly
modified version of {\sc EXP4}. The resulting regret bounds are in
$\tilde O(T^{2/3})$.

Our work can also be viewed as selecting the best algorithm for a
given unknown environment and, in this way, is similar in spirit to
the literature solving the \emph{best of both worlds}
problem~\citep{audibert2009minimax,bubeck2012best,seldin2014one,
  auer2016algorithm,seldin2017improved,WeiLuo2018,zimmert2018optimal,
ZimmertLuoWei2019}
and the model selection problem for linear bandit
\citep{foster2019model,chatterji2019osom}.

Very recently, \cite{pacchiano2020model} also considered the problem
of corralling stochastic bandit algorithms. The authors seek to treat
the problem of model selection, where multiple algorithms might share
the best arm. More precisely, the authors consider a setting in which
there are $K$ stochastic contextual bandit algorithms and try to
minimize the regret with respect to the best overall policy belonging
to any of the bandit algorithms. They propose two corralling
algorithms, one based on the work of \citep{agarwal2016corralling} and
one based on EXP3.P \citep{auer2002nonstochastic}. The main novelty in
their work is a smoothing technique for each of the base algorithms,
which avoids having to restart the base algorithms throughout the $T$
rounds, as was proposed in \citep{agarwal2016corralling}.  The
proposed regret bounds are of the order $\tilde\Theta(\sqrt{T})$.  We
expect that the smoothing technique is also applicable to one of the
corralling algorithms we propose. Since \cite{pacchiano2020model}
allow for algorithms with shared best arms, their main results do not
discuss the optimistic setting in which there is a gap between the
reward of the optimal policy and all other competing policies, and do
not achieve the \emph{optimistic guarantees} we provide.  Further,
they show a min-max lower bound which states that even if one of the
base algorithms is optimistic and contains the best arm, there is
still no hope to achieve regret better than $\tilde\Omega(\sqrt{T})$
if the best arm is shared by an algorithm with regret
$\tilde\Omega(\sqrt{T})$.
We view their contributions as complementary to ours.

In general, some caution is needed when designing a corralling
algorithm, since aggressive strategies may discard or disregard a base
learner that admits an arm with the best mean reward if it performs
poorly in the initial rounds. Furthermore, as noted 
by \cite{agarwal2016corralling}, 
additional assumptions are required on each of the base learners 
if one hopes to achieve non-trivial corralling guarantees. 

\textbf{Contributions.} 
We first motivate our key assumption that all of the corralled
algorithms must have favorable regret guarantees during all rounds. To
do so, in Section~\ref{sec:lower_bounds}, we show that if one does not
assume anytime regret guarantees, then even when corralling simple
stochastic bandit algorithms, each with $o(\sqrt{T})$ regret, any
corralling strategy will have to incur $\Omega(\sqrt{T})$
regret. Therefore, for the rest of the paper we assume that each base
learner admits anytime guarantees.  In Section~\ref{sec:ucb_boost} and
Section~\ref{sec:tsallis}, we present two general corralling
algorithms whose pseudo-regret guarantees admit a dependency on the
gaps between base learners, that is their best arms, and only
poly-logarithmic dependence on time horizon. These bounds are
syntactically similar to the instance-dependent guarantees for the
stochastic multi-armed bandit problem~\citep{auer2002finite}. Thus,
our corralling algorithm performs almost as well as the best base
learner, if it were to be used on its own, modulo gap-dependent terms
and logarithmic factors. The algorithm in Section~\ref{sec:ucb_boost}
uses the standard UCB ideas combined with a boosting technique, which
runs multiple copies of the same base learner.
In Section~\ref{sec:lower_bound_without_boosting}, we show that simply
using UCB-style corralling without boosting can incur linear regret.
If, additionally, we assume that each of the base learners satisfy the
stability condition adopted in \citep{agarwal2016corralling}, then, in
Section~\ref{sec:tsallis} we show that it suffices to run a single
copy of each base learner by using a corralling approach based on
OMD. We show that UCB-I~\citep{auer2002finite} can be made to satisfy
the stability condition, as long as the confidence bound is rescaled
and changed by an additive factor. In
Section~\ref{sec:experiments}, to further examine the 
properties of our algorithms, we report the results of experiments
with our algorithms for synthetic datasets. Finally, while
our main motivation is not model selection, in Section~\ref{sec:model-selection}, we briefly discuss some
related matters and show that our algorithms
can help recover several known results in that area.

\section{Preliminaries}

We consider the problem of corralling $K$ stochastic multi-armed
bandit algorithms $\cA_1, \ldots, \cA_K$, which we often refer to as
\emph{base algorithms} (base learners).  At each round $t$, a
\emph{corralling algorithm} selects a base algorithm $\cA_{i_t}$,
which plays action $a_{i_t,j_t}$. The corralling algorithm is not
informed of the identity of this action but it does observe its reward
$r_t(a_{i_t, j_t})$. The top algorithm then updates its decision rule
and provides feedback to each of the base learners $\cA_i$. We note
that the feedback may be just the empty set, in which case the base
learners do not update their state. We will also assume access to the
parameters controlling the behavior of each $\cA_i$ such as the step
size for mirror descent-type algorithms, or the confidence bounds for
UCB-type algorithms. Our goal is to minimize the cumulative
pseudo-regret of the corralling algorithm as defined in
Equation~\ref{eq:pseudo_reg}:\footnote{For conciseness, from now on,
  we will simply write \emph{regret} instead of \emph{pseudo-regret}.}
\begin{equation}
\label{eq:pseudo_reg}
\mathbb{E}[R(T)]
= T\mu_{1,1} - \mathbb{E}\left[\sum_{t=1}^T r_t(a_{i_t,j_t})\right],
\end{equation}
where $\mu_{1,1}$ is the mean reward of the best arm.

\textbf{Notation.} We denote by $\e_i$ the $i$th standard basis vector, 
by $\pmb{0}_K \in \mathbb{R}^{K}$ the vector of all $0$s, and 
by $\pmb{1}_K \in \mathbb{R}^{K}$ the vector of all $1$s. For two vectors 
$x, y \in \mathbb{R}^K$, $x \odot y$ denotes their Hadamard product. 
We also denote the line segment between $x$ and $y$ as $[x, y]$. 
$w_{t,i}$ denotes the $i$-th entry of a vector $w_t \in \mathbb{R}^K$.
$\Delta^{K - 1}$ denotes the probability simplex in $\mathbb{R}^K$, 
$D_{\Psi}(x, y)$ the Bregman divergence induced by the potential $\Psi$, 
whose conjugate function we denote by $\Psi^*$. We use $\I_{C}$ to 
denote the indicator function of a set $C$. 
For any $k \in \mathbb{N}$, we use the shorthand
$[k] := \{1, 2, \ldots, k\}$. 

For the base algorithms $\cA_1, \ldots, \cA_K$, let $T_{i}(t)$ be the
number of times algorithm $\cA_i$ has been played until time $t$. Let
$T_{i,j}(t)$ be the number of times action $j$ has been proposed by
algorithm $\cA_i$ until time $t$.  Let $[k_i]$ denote the set of arms
or action set of algorithm $\cA_i$.  We denote the reward of arm $j$
in the action set of algorithm $i$ at time $t$ as $r_t(a_{i,j})$ and
denote its mean reward by $\mu_{i,j}$. We also use $a_{i,j_t}$ to
denote the arm proposed by algorithm $\cA_i$ during time $t$. Further,
the algorithm played at time $t$ is denoted as $i_t$, its action
played at time $t$ is $a_{i_t,j_t}$ and the reward for that action is
$r_t(a_{i_t,j_t})$ with mean $\mu_{i_t,j_t}$. Let $i^*$ denote the
index of the base algorithm that contains the arm with the highest
mean reward. Without loss of generality, we will assume that
$i^* = 1$. Similarly, we assume that $a_{i,1}$ is the arm with highest
reward in algorithm $\cA_i$. We assume that the best arm of the best
algorithm has a gap to the best arm of every other algorithm. We
denote the gap between the best arm of $\cA_1$ and the best arm of
$\cA_i$ as $\Delta_i$: $\Delta_i = \mu_{i^*,1} - \mu_{i,1} > 0$ for
$i\neq i^*$. Further, we denote the intra-algorithm gaps by
$\Delta_{i,j} = \mu_{i,1} - \mu_{i,j}$. We denote by $\bar R_{i}(t)$
an upper bound on the regret of algorithm $\cA_i$ at time $t$ and by
$R_{i}(t)$ the actual regret of $\cA_i$, so that
$\mathbb{E}[R_{i}(t)]$ is the expected regret of algorithm $\cA_i$ at
time $t$.  The asymptotic notations $\tilde \Omega$ and $\tilde O$ are
equal to $\Omega$ and $O$ up to poly-logarithmic factors.

\section{Lower bounds without anytime regret guarantees}
\label{sec:lower_bounds}

We begin by showing a simple and yet instructive lower bound that
helps guide our intuition regarding the information needed from the
base algorithms $\{\cA_i\}_{i = 1}^K$ in the design of a corralling
algorithm. Our lower bound is based on corralling base algorithms that
only admit a fixed-time horizon regret bound and do not enjoy anytime
regret guarantees.  We further assume that the corralling strategy
cannot simulate anytime regret guarantees on the base algorithms, say
by using the so-called doubling trick.  This result suggests that the
base algorithms must admit a strong regret guarantee during every
round of the game.

The key idea behind our construction is the following. Suppose one of
the corralled algorithms, $\cA_i$, incurs a linear regret over the
first $R_{i}(T)$ rounds. In that case, the corralling algorithm is
unable to distinguish between $\cA_i$ and an another algorithm that
mimics the linear regret behavior of $\cA_i$ throughout all $T$
rounds, unless the corralling algorithm plays $\cA_i$ at least
$R_{i}(T)$ times.  The successive elimination algorithm
\citep{even2002pac} benefits from gap-dependent bounds and can have
the behavior just described for a base algorithm. Thus, our lower
bound is presented for successive elimination base algorithms, all
with regret $O(T^{1/4})$.  It shows that, with constant probability,
no corralling strategy can achieve a more favorable regret than
$\tilde\Omega(\sqrt{T})$ in that case.
\begin{theorem}
\label{thm:lower_bound2}
Let the corralled algorithms be instances of successive elimination
defined by a parameter $\alpha$. With probability $1/4$ over the
random sampling of $\alpha$, any corralling strategy will incur regret
at least $\tilde\Omega(\sqrt{T})$, while the gap, $\Delta$, between
the best and second best reward is such that
$\Delta > \omega(T^{-1/4})$ and all algorithms have a regret bound of
$\tilde O(1/\Delta)$.
\end{theorem}
This theorem shows that, even when corralling natural algorithms that
benefit from asymptotically better regret bounds, corralling can incur
$\tilde\Omega(\sqrt{T})$ regret.
It can be further proven (Theorem~\ref{thm:lower_bound3},
Appendix~\ref{app:lower_bounds}) that, even if the worst case upper
bounds on the regret of the base algorithms were known, achieving an
optimistic regret guarantee for corralling would not be possible,
unless some additional assumptions were made.

\section{UCB-style corralling algorithm}
\label{sec:ucb_boost}

The negative result of Section~\ref{sec:lower_bounds} hinge on the
fact that the base algorithms do not admit anytime regret guarantees.
Therefore, we assume, for the rest of the paper, that the base
algorithms, $\{\cA_i\}$, satisfy the following:
\begin{equation}
  \label{eq:pseudo_regret_assumption}
  \mathbb{E}\left[t\mu_{i,1}-\sum_{s=1}^t r_{s}(a_{i_s,j_s})\right] \leq \bar R_{i}(t),
\end{equation}
for any time $t\in [T]$.  For UCB-type algorithms, such bounds can be
derived from the fact that the expected number of pulls, $T_{i,j}(t)$,
of a suboptimal arm $j$, is bounded as
$\mathbb{E}[T_{i,j}(t)] \leq c\frac{\log{t}}{(\Delta_{i,j})^2}$, for
some time and gap-independent constant $c$ (e.g.,
\cite{bubeck2010bandits}), and take the following form,
$\bar R_{i}(t) \leq c' \sqrt{k_i t \log{t}}$, for some constant $c'$.

Suppose that the bound in Equation~\ref{eq:pseudo_regret_assumption}
holds with probability $1 - \delta_t$. Note that such bounds are
available for some UCB-type
algorithms~\citep{audibert2009exploration}. We can then adopt the
optimism in the face of uncertainty principle for each $\mu_{i,1}$ by
overestimating it with
$\frac{1}{t}\sum_{s=1}^{t} r_s(a_{i,j_s}) + \frac{1}{t}\bar
R_{i}(t)$. As long as this occurs with high enough probability, we can
construct an upper confidence bound for $\mu_{i,1}$ and use it in a
UCB-type algorithm. Unfortunately, the upper confidence bounds
required for UCB-type algorithms to work need to hold with high enough
probability, which is not readily available from
Equation~\ref{eq:pseudo_regret_assumption} or from probabilistic
bounds on the pseudo-regret of anytime stochastic bandit
algorithms. In fact, as discussed in
Section~\ref{sec:lower_bound_without_boosting}, we expect it to be
impossible to corral any-time stochastic MAB algorithms with a
standard UCB-type strategy. However, a simple boosting technique, in
which we run $2\log{1/\delta}$ copies of each algorithm $\cA_i$, gives
the following high probability version of the bound in
Equation~\ref{eq:pseudo_regret_assumption}.

\begin{lemma}
\label{lem:boosted_regret_bound}
Suppose we run $2\log{1/\delta}$ copies of algorithm $\cA_i$ which
satisfies Equation~\ref{eq:pseudo_regret_assumption}.  If
$\cA_{med_i}$ is the algorithm with median cumulative reward at time
$t$, then
$\mathbb{P}[t\mu_{i,1}-\sum_{s=1}^t r_{s}(a_{med_i,j_s}) \geq 2\bar
R_{i}(t)] \leq \delta$.
\end{lemma}

\begin{algorithm}[t]
\caption{UCB-C}
\label{alg:ucb-c-gen}
\begin{algorithmic}[1]
\REQUIRE{Stochastic bandit algorithms $\cA_1,\ldots,\cA_K$}
\ENSURE{Sequence of algorithms $(i_t)_{t=1}^T$.}
\STATE $t=1$
\FOR{$i=1,\ldots,K$}
\STATE $\mathbb{A}_i = \emptyset$ \% contains all copies of $\cA_i$
\FOR{$s=1,\ldots, \lceil2\log{T}\rceil$}
\STATE Initialize $\cA_{i}(s)$ as a copy of $\cA_i$, $\hat\mu_{i}(s) = 0$ 
\STATE Append $(\cA_{i}(s),\hat\mu_{i}(s))$ to $\mathbb{A}_i$
\ENDFOR
\ENDFOR
\FOR{$i = 1,\ldots,K$}
\STATE Foreach $(\cA_{i}(s),\hat\mu_{i}(s)) \in \mathbb{A}_i$, play $\cA_{i}(s)$, update empirical mean $\hat\mu_{i}(s)$, $t=t+2\log{T}$
\STATE $\hat\mu_{med_i} = \texttt{Median}(\{\hat\mu_{i}(s)\}_{s=1}^{\lceil2\log{T}\rceil})$
\ENDFOR
\WHILE{$t\leq T$}
\STATE \small
$b_\ell(t) = \frac{\sqrt{2 \bar R_{med_\ell}(T_{med_\ell}(t))} + \sqrt{2T_{med_\ell}(t)\log{t}}}{T_{med_\ell}(t)},\forall \ell\in[K]$
\STATE \small
$i=\argmax_{\ell \in [K]}\left\{\hat\mu_{med_\ell} + b_\ell(t)\right\}$
\STATE Foreach $(\cA_{i}(s),\hat\mu_{i}(s)) \in \mathbb{A}_i$, play $\cA_{i}(s)$, update empirical mean $\hat\mu_{i}(s)$, $t=t+2\log{T}$ 
\STATE $\hat\mu_{med_i} = \texttt{Median}(\{\hat\mu_{i}(s)\}_{s=1}^{\lceil2\log{T}\rceil})$
\ENDWHILE
\end{algorithmic}
\end{algorithm}

We consider the following variant of the standard UCB algorithm for
corralling. We initialize $2\log{T}$ copies of each base algorithm
$\cA_i$.  Each $\cA_i$ is associated with the median empirical average
reward of its copies. At each round, the corralling algorithm picks
the $\cA_i$ with the highest sum of median empirical average reward
and an upper confidence bound based on
Lemma~\ref{lem:boosted_regret_bound}.  The pseudocode is given in
Algorithm~\ref{alg:ucb-c-gen}. The algorithm admits the following
regret guarantees.

\begin{theorem}
\label{thm:regret_bound_ucbc}
Suppose that algorithms $\cA_1,\ldots,\cA_K$ satisfy the following
regret bound $\mathbb{E}[R_{i}(t)] \leq \sqrt{\alpha k_i t\log{t}}$,
respectively for $i \in [K]$. Algorithm~\ref{alg:ucb-c-gen} selects a
sequence of algorithms $i_1,\ldots,i_T$ which take actions
$a_{i_1,j_1},\ldots,a_{i_T,j_T}$, respectively, such that
\begin{align*}
    \mathbb{E}[R(T)]&\leq O\left(\sum_{i \neq i^*} \frac{k_i\log{T}^2}{\Delta_i} + \log{T}\mathbb{E}\left[R_{i^*}(T)\right]\right),\\
    \mathbb{E}[R(T)]&\leq O\left(\log{T}\sqrt{K T\log{T} \max_{i\in [K]} (k_i)}\right).
\end{align*}
\end{theorem}
We note that both the optimistic and the worst case regret bounds
above involve an additional factor that depends on the number of arms,
$k_i$, of the base algorithm $\cA_i$. %
This dependence reflects the complexity of the decision space of
algorithm $\cA_i$. We conjecture that a complexity-free bound is not
possible, in general. To see this, consider a setting where each
$\cA_i$, for $i\neq i^*$, only plays arms with equal means
$\mu_i=\mu_{1,1} - \Delta_i$.  Standard stochastic bandit regret lower
bounds, e.g.~\citep{garivier2018explore}, state that any strategy on
the combined set of arms of all algorithms will incur regret at least
$\Omega(\sum_{i\neq i^*} k_i\log{T}/\Delta_i)$. The $\log{T}$ factor
in front of the regret of the best algorithm comes from the fact that
we are running $\Omega(\log{T})$ copies of it.

\subsection{Discussion regarding tightness of bounds}
\label{sec:lower_bound_without_boosting}

A natural question is if it is possible to achieve bounds that do not
have a $\log{T}^2$ scaling. After all, for the simpler stochastic MAB
problem, regret upper bounds only scale as $O(\log{T})$ in terms of
the time horizon.  As already mentioned, the extra logarithmic factor
comes from the boosting technique, or, more precisely, the need for
exponentially fast concentration of the true regret to its expected
value, when using a UCB-type corralling strategy.  We now show that,
in the absence of such strong concentration guarantees, if only a
single copy of each of the base algorithms in
Algorithm~\ref{alg:ucb-c-gen} is run, then linear regret is
unavoidable.

\begin{theorem}
\label{thm:ucb_lower_bound_boosting}
There exist instances $\cA_1$ and $\cA_2$ of UCB-I and a reward
distribution, such that, if Algorithm~\ref{alg:ucb-c-gen} runs a
single copy of $\cA_1$ and $\cA_2$, then
$\mathbb{E}[R(T)] \geq \tilde \Omega(\Delta_2 T).$

Further, for any algorithm $\cA_1$ such that
$\prob{R_{1}(t) \geq \frac{1}{2}\Delta_{1,2}\tau} \geq
\frac{1}{\tau^c}$, there exists a reward distribution such that if
Algorithm~\ref{alg:ucb-c-gen} runs a single copy of $\cA_1$ and
$\cA_2$, then
$\mathbb{E}[R(T)] \geq \tilde \Omega((\Delta_{1,2})^c\Delta_2 T)$.
\end{theorem}
The proof of the above theorem and further discussion can be found in
Appendix~\ref{sec:ucb_lower_bound}.  The requirement that the regret
of the best algorithm satisfies
$\mathbb{P}[R_1(t) \geq \frac{1}{2}\Delta_{1,2}\tau] \geq
\frac{1}{\tau^c}$ in Theorem~\ref{thm:ucb_lower_bound_boosting} is
equivalent to the condition that the regret of the base algorithms
admit only a polynomial concentration.  Results in
\citep{salomon2011deviations} suggest that there cannot be a tighter
bound on the tail of the regret for anytime algorithms. It is
therefore unclear if the $\log{T}^2$ rate can be improved upon or if
there exists a matching information-theoretic lower bound.

\section{Corralling using Tsallis-INF}
\label{sec:tsallis}

In this section, we consider an alternative approach, based on the
work of \cite{agarwal2016corralling}, which avoids running multiple
copies of base algorithms. Since the approach is based on the OMD
framework, which is naturally suited to losses instead of rewards, for
the rest of the section we switch to losses.

We design a corralling algorithm that maintains a probability
distribution $w \in \Delta^{K-1}$ over the base algorithms,
$\{\cA_i\}_{i=1}^K$.  At each round, the corralling algorithm samples
$i_t \sim w$. Next, $\cA_{i_t}$ plays $a_{i_t,j_t}$ and the corralling
algorithm observes the loss $\ell_t(a_{i_t,j_t})$. The corralling
algorithm updates its distribution over the base algorithms using the
observed loss and provides an unbiased estimate
$\hat\ell_t(a_{i,j_t})$ of $\ell_t(a_{i,j_t})$ to algorithm $\cA_i$:
the feedback provided to $\cA_i$ is
$\hat\ell_t(a_{i_t,j_t}) = \frac{\ell_t(a_{i_t,j_t})}{w_{t,i_t}}$, and
for all $a_{i,j_t}\neq a_{i_t,j_t}$, $\hat\ell_t(a_{i,j_t}) =
0$. Notice that $\hat\ell_t \in \mathbb{R}^{K}$, as opposed to
$\ell_t \in [0,1]^{\prod_{i} k_i}$. Essentially, the loss fed to
$\cA_{i}$, with probability $w_{t,i}$, is the true loss rescaled by
the probability $w_{t,i}$ to observe the loss, and is equal to $0$
with probability $1 - w_{t,i}$.

The change of environment induced by the rescaling of the observed
losses is analyzed in \cite{agarwal2016corralling}. Following
\cite{agarwal2016corralling}, we denote the environment of the
original losses $(\ell_t)_t$ as $\cE$ and that of the rescaled losses
$(\hat\ell_t)_t$ as $\cE'$. Therefore, in environment $\cE$, algorithm
$\cA_i$ observes $\ell_t(a_{i_t,j_t})$ and in environment $\cE'$,
$\cA_i$ observes $\hat\ell_t(a_{i_t,j_t})$. A few important remarks
are in order. As in \citep{agarwal2016corralling}, we need to assume
that the base algorithms admit a \emph{stability property} under the
change of environment. In particular, if
$w_{s,i} \geq \frac{1}{\rho_t}$ for all $s \leq t$ and some
$\rho_t\in\mathbb{R}$, then $\mathbb{E}[R_{i}(t)]$ under environment
$\cE'$ is bounded by $\mathbb{E}[\sqrt{\rho_t}R_{i}(t)]$. For
completeness, we provide the definition of stability by
\cite{agarwal2016corralling}.
\begin{definition}
\label{def:stability}
Let $\gamma \in (0,1]$ and let
$R\colon \mathbb{N} \rightarrow \mathbb{R}_{+}$ be a non-decreasing
function. An algorithm $\cA$ with action space $\A$ is
$(\gamma, R(\cdot))$-stable with respect to an environment $\cE$ if
its regret under $\cE$ is $R(T)$ and its regret under $\cE'$ induced
by the importance weighting is
$\max_{a\in \A}\mathbb{E}\left[\sum_{t=1}^T \hat \ell_t(a_{i_t,j_t}) -
  \ell_t(a)\right] \leq \mathbb{E}[(\rho_T)^\gamma R(T)].$
\end{definition}
We show that UCB-I~\citep{auer2002finite} satisfies the stability
property above with $\gamma = \frac 12$. The techniques used in the
proof are also applicable to other UCB-type algorithms. Other
algorithms for stochastic bandits like Thompson sampling and OMD/FTRL
variants have been shown to be $1/2$-stable in
\citep{agarwal2016corralling}.

The corralling algorithm of \cite{agarwal2016corralling} is based on
Online Mirror Descent (OMD), where a key idea is to increase the step
size whenever the probability of selecting some algorithm $\cA_{i}$
becomes smaller than some threshold. This induces a negative regret
term which, coupled with a careful choice of step size (dependent on
regret upper bounds of the base algorithms), provides regret bounds
that scale as a function of the regret of the best base algorithm.

Unfortunately, the analysis of the corralling algorithm always leads
to at least a regret bound of $\tilde \Omega(\sqrt{T})$ and also
requires knowledge of the regret bound of the best algorithm.  Since
our goal is to obtain instance-dependent regret bounds, we cannot
appeal to this type of OMD approach. Instead, we draw inspiration from
the recent work of \cite{zimmert2018optimal}, who use a
Follow-the-Regularized-Leader (FTRL) type of algorithm to design an
algorithm that is simultaneously optimal for both stochastic and
adversarially generated losses, without requiring knowledge of
instance-dependent parameters such as the sub-optimality gaps to the
loss of the best arm. The overall intuition for our algorithm is as
follows. We use the FTRL-type algorithm proposed by
\cite{zimmert2018optimal} until the probability to sample some arm
falls below a threshold. Next, we run an OMD step with an increasing
step size schedule which contributes a negative regret term. After the
OMD step, we resume the normal step size schedule and updates from the
FTRL algorithm. After carefully choosing the initial step size rate,
which can be done in an instance-independent way, the accumulated
negative regret terms are enough to compensate for the increased
regret due to the change of environment.

\subsection{Algorithm and the main result}
We now describe our corralling algorithm in more detail.  The
potential function $\Psi_t$ used in all of the updates is defined by
$\Psi_t(w) = -4\sum_{i\in [K]}\frac{1}{\eta_{t,i}}\left(\sqrt{w_i} -
  \frac{1}{2} w_i\right)$, where $\eta_t = \left[
\begin{matrix}
\eta_{t,1}, \eta_{t,2}, \ldots, \eta_{t,K}
\end{matrix} \right]$ is the step-size schedule during time $t$. The
algorithm proceeds in epochs and begins by running each base algorithm
for $\log{T} + 1$ rounds. Each epoch is twice as large as the
preceding, so that the number of epochs is bounded by
$\operatorname{log}_2(T)$, and the step size schedule remains
non-increasing throughout the epochs, except when an OMD step is
taken. The algorithm also maintains a set of thresholds,
$\rho_1, \rho_2, \ldots, \rho_n$, where $n= O(\log{T})$.  These
thresholds are used to determine if the algorithm executes an OMD
step, while increasing the step size:
\begin{equation}
\label{eq:omd_step2}
    \begin{aligned}
        w_{t+1} &= \argmin_{w\in \Delta^{K-1}} \langle \hat \ell_t, w \rangle + D_{\Psi_t}(w,w_t),\\
       \eta_{t+1,i} &= \beta\eta_{t,i} \text{ (for $i$ : $w_{t,i} \leq 1/\rho_{s_i}$)},\\
       w_{t+2} &= \argmin_{w\in \Delta^{K-1}} \langle \hat \ell_{t+1}, w \rangle + D_{\Psi_{t+1}}(w, w_{t+1}), \rho_{s_i} = 2\rho_{s_i}
    \end{aligned}
\end{equation}
or the algorithm takes an FTRL step
\begin{equation}
\label{eq:ftrl_step2}
    w_{t+1} = \argmin_{w\in \Delta^{K-1}} \langle \hat L_t, w \rangle + \Psi_{t+1}(w),
\end{equation}
where $\hat L_t = \hat L_{t-1} + \hat \ell_t$, unless otherwise
specified by the algorithm. We note that the algorithm can only
increase the step size during the OMD step. For technical reasons, we
require an FTRL step after each OMD step. Further, we require that the
second step of each epoch be an OMD step if there exists at least one
$w_{t,i} \leq \frac{1}{\rho_1}$. The algorithm also can enter an OMD
step during an epoch if at least one $w_{t,i}$ becomes smaller than a
threshold $\frac{1}{\rho_{s_i}}$ which has not been exceeded so far.

\begin{algorithm}[t]
\caption{Corralling with Tsallis-INF}
\label{alg:tsallis_inf_incr_dcr}
\begin{algorithmic}[1]
\REQUIRE{Mult. constant $\beta$, thresholds $\{\rho_i\}_{i=1}^n$, initial step size $\eta$, epochs $\{\tau_i\}_{i=1}^m$, algorithms $\{\cA_i\}_{i=1}^K$.}
\ENSURE{Algorithm selection sequence $(i_t)_{t=1}^T$.}
\STATE Initialize $t=1$, $w_1 = Unif(\Delta^{K-1})$, $\eta_1 = \eta$
\STATE Initialize current threshold list $\theta \in [n]^K$ to $\bf 1$
\WHILE{$t\leq K\log{T}+K$}
\FOR{$i\in [K]$}
\STATE $\cA_i$ plays $a_{i,j_t}$, $\hat L_{1,i} += \ell_{t}(a_{i,j_t}),t+=1$
\ENDFOR
\ENDWHILE
\STATE $t=2, w_2 = \nabla \Phi_2(-\hat L_1),1/\eta_{t+1}^2 = 1/\eta_{t}^2 +1$
\WHILE{$j \leq m$}
\FOR{$t\in \tau_j$}
\STATE $\mathcal{R}_t = \emptyset$,$\hat\ell_t = \texttt{PLAY-ROUND}(w_t)$
\IF{$t$ is first round of $\tau_j$ and $\exists w_{t,i} \leq \frac{1}{\rho_1}$}
\FOR{$i \colon w_{t,i} \leq \frac{1}{\rho_1}$}
\STATE $\theta_i = \min\{s \in [n]\colon w_{t,i} > \frac{1}{\rho_s}\}$, $\mathcal{R}_t = \mathcal{R}_t\bigcup \{i\}$.
\ENDFOR
\STATE $(w_{t+3},\hat L_{t+2}) = \texttt{NRS}(w_{t}, \hat\ell_{t}, \eta_{t}, \mathcal{R}_t, \hat L_{t-1})$, $t = t+2$, $\hat\ell_t = \texttt{PLAY-ROUND}(w_t)$
\ENDIF
\IF{$\exists i\colon w_{t,i} \leq \frac{1}{\rho_{\theta_i}}$ and prior step was not \texttt{NRS}}
\FOR{$i \colon w_{t,i} \leq \frac{1}{\rho_{\theta_i}}$}
\STATE $\theta_i +=1$, $\mathcal{R}_t = \mathcal{R}_t\bigcup \{i\}$.
\ENDFOR
\STATE $(w_{t+3},\hat L_{t+2}) = \texttt{NRS}(w_{t}, \hat\ell_{t}, \eta_{t}, \mathcal{R}_t, \hat L_{t-1})$, $t = t+2$, $\hat\ell_t = \texttt{PLAY-ROUND}(w_t)$
\ELSE
\STATE $1/\eta_{t+1}^2 = 1/\eta_{t}^2 +1$, $w_{t+1} = \nabla \Phi_{t+1}(-\hat L_t)$
\ENDIF
\ENDFOR
\ENDWHILE
\end{algorithmic}
\end{algorithm}

\begin{algorithm}[t]
\caption{$\texttt{NEG-REG-STEP} (\texttt{NRS})$}
\label{alg:neg_regret}
\begin{algorithmic}[1]
\REQUIRE{Prior iterate $w_t$, loss $\hat\ell_t$, step size $\eta_t$, set of rescaled step-sizes $\mathcal{R}_t$, cumulative loss $\hat L_{t-1}$}
\ENSURE{Plays two rounds of the game and returns distribution $w_{t+3}$ and cumulative loss $\hat L_{t+2}$}
\STATE $(w_{t+1},\hat L_{t}) = \texttt{OMD-STEP}(w_{t}, \hat\ell_{t}, \eta_{t}, \mathcal{R}_t, \hat L_{t-1})$
\STATE $\hat\ell_{t+1} = \texttt{PLAY-ROUND}(w_{t+1})$, $\hat L_{t+1} = \hat L_t + \hat\ell_{t+1}$
\FOR{ all $i$ such that $w_{t,i} \leq \frac{1}{\rho_1}$}
\STATE $\eta_{t+2,i} = \beta\eta_{t,i}$, $\mathcal{R}_t = \mathcal{R}_t \cup \{i\}$ and restart $\cA_i$ with updated environment $\theta_i = \frac{1}{2w_{t,i}}$
\ENDFOR
\STATE $w_{t+2} = \nabla \Phi_{t+2}(-\hat L_{t+1})$
\STATE $\hat\ell_{t+2} = \texttt{PLAY-ROUND}(w_{t+2})$
\STATE $\hat L_{t+2} = \hat L_{t+1} + \hat\ell_{t+2}, \eta_{t+3} = \eta_{t+2}, t = t+2$
\STATE $w_{t+1} = \nabla \Phi_{t+1}(-\hat L_t),t=t+1$
\end{algorithmic}
\end{algorithm}

We set the probability thresholds so that $\rho_1 = O(1)$,
$\rho_j = 2 \rho_{j-1}$ and $\frac{1}{\rho_n} \geq \frac{1}{T}$, so
that $n \leq \operatorname{log}_2(T)$.  In the beginning of each
epoch, except for the first epoch, we check if
$w_{t,i} < \frac{1}{\rho_1}$. If it is, we increase the step size as
$\eta_{t+1,i} = \beta \eta_{t,i}$ and run the OMD step. The pseudocode
for the algorithm is given in
Algorithm~\ref{alg:tsallis_inf_incr_dcr}. The routines
\texttt{OMD-STEP} and \texttt{PLAY-ROUND} can be found in
Algorithm~\ref{alg:omd_step} and Algorithm~\ref{alg:play_round}
(Appendix~\ref{app:tsallis-inf}) respectively. \texttt{OMD-STEP}
essentially does the update described in Equation~\ref{eq:omd_step2}
and \texttt{PLAY-ROUND} samples and plays an algorithm, after which
constructs an unbiased estimator of the losses and feeds these back to
all of the sub-algorithms. We show the following regret bound for the
corralling algorithm.
\begin{theorem}
\label{thm:regret_bound}
Let $\bar R_{i}(\cdot)$ be a function upper bounding the expected regret, $\mathbb{E}[R_{i}(\cdot)]$, of $\cA_{i}$ for all $i\in [K]$. 
For $\beta = e^{1/\log{T}^2}$ and for $\eta$ such that for all $i \in [K]$, $\eta_{1, i} \leq \min_{t\in [T]} \frac{\left(1-\exp{-\frac{1}{\log{T}^2}}\right)\sqrt{t}}{50\bar R_{i}(t)}$,  the expected regret of Algorithm~\ref{alg:tsallis_inf_incr_dcr} is bounded as follows:
    $\mathbb{E}\left[R(T)\right] \leq O\left(\sum_{i\neq i^*} \frac{\log{T}}{\eta_{1,i}^2 \Delta_i} + \mathbb{E}[R_{i^*}(T)]\right).$
\end{theorem}

To parse the bound above, suppose $\{\cA_i\}_{i\in[K]}$ are standard
stochastic bandit algorithms such as UCB-I. In
Theorem~\ref{thm:ucb_stable}, we show that UCB-I is indeed
$\frac{1}{2}$-stable as long as we are allowed to rescale and
introduce an additive factor to the confidence bounds. In this case, a
worst-case upper bound on the regret of any $\cA_{i}$ is
$\mathbb{E}[R_{i}(t)] \leq c\sqrt{k_{i}\log{t} t}$ for all $t\in[T]$
and some universal constant $c$. We note that the min-max regret bound
for the stochastic multi-armed bandit problem is $\Theta(\sqrt{KT})$
and most known any-time algorithms solving the problem achieve this
bound up to poly-logarithmic factors. Further we note that
$\left(1-\exp{-\frac{1}{\log{T}^2}}\right) >
\frac{1}{e\log{T}^2}$. This suggests that the bound in
Theorem~\ref{thm:regret_bound} on the regret of the corralling
algorithm is at most
$O\big(\sum_{i\neq i^*} \frac{k_i\log{T}^5}{\Delta_{i}}
+\mathbb{E}[R_{i^*}(T)]\big)$. In particular, if we instantiate
$\mathbb{E}[R_{i^*}(T)]$ to the instance-dependent bound of
$O\Big(\sum_{j \neq 1} \frac{\log{T}}{\Delta_{i^*,j}}\Big)$, the
regret of Algorithm~\ref{alg:tsallis_inf_incr_dcr} is bounded by
$O\Big(\sum_{i\neq i^*} \frac{k_i\log{T}^5}{\Delta_{i}} +\sum_{j \neq
  1} \frac{\log{T}}{\Delta_{i^*,j}}\Big).$ In general we cannot
exactly compare the current bound with that of UCB-C
(Algorithm~\ref{alg:ucb-c-gen}), as the regret bound in
Theorem~\ref{thm:regret_bound} has worse scaling in the time horizon
on the gap-dependent terms, compared to the regret bound in
Theorem~\ref{thm:regret_bound_ucbc}, but has no additional scaling in
front of the $\mathbb{E}[R_{i^*}(T)]$ term. In practice we observe
that Algorithm~\ref{alg:tsallis_inf_incr_dcr} outperforms
Algorithm~\ref{alg:ucb-c-gen}.

Since essentially all stochastic multi-armed bandit algorithms enjoy a
regret bound, in time horizon, of the order $\tilde O(\sqrt{T})$, we
are guaranteed that $1/\eta_{t,i}^2$ scales only poly-logarithmically
with the time horizon. What happens, however, if algorithm $\cA_i$ has
a worst case regret bound of the order $\omega(\sqrt{T})$? For the
next part of the discussion, we only focus on time horizon
dependence. As a simple example, suppose that $\cA_i$ has worst case
regret of $T^{2/3}$ and that $\cA_{i^*}$ has a worst case regret of
$\sqrt{T}$. In this case, Theorem~\ref{thm:regret_bound} tells us that
we should set $\eta_{1,i} = \tilde O(1/T^{1/6})$ and hence the regret
bound scales at least as
$\Omega(T^{1/3}/\Delta_i + \mathbb{E}[R_{i^*}(T)])$. In general, if
the worst case regret bound of $\cA_i$ is in the order of $T^{\alpha}$
we have a regret bound scaling at least as
$T^{2\alpha-1}/\Delta_{i}$. This is not unique to
Algorithm~\ref{alg:tsallis_inf_incr_dcr} and a similar scaling of the
regret would occur in the bound for Algorithm~\ref{alg:ucb-c-gen} due
to the scaling of confidence intervals.

\textbf{Corralling in an adversarial environment.}  Because
Algorithm~\ref{alg:tsallis_inf_incr_dcr} is based on a best of both
worlds algorithm, we can further handle the case when the
losses/rewards are generated adversarially or whenever the best
overall arm is shared across multiple algorithms, similarly to the
settings studied by \cite{agarwal2016corralling,pacchiano2020model}.
\begin{theorem}
\label{thm:regret_bound_adv}
Let $\bar R_{i^*}(\cdot)$ be a function upper bounding the expected regret of $\cA_{i^*}$, $\mathbb{E}[R_{i^*}(\cdot)]$. For any $\eta_{1,i^*} \leq \min_{t\in [T]} \frac{\left(1-\exp{-\frac{1}{\log{T}^2}}\right)\sqrt{t}}{50\bar R_{i^*}(t)}$ and $\beta = e^{1/\log{T}^2}$ it holds that the expected regret of Algorithm~\ref{alg:tsallis_inf_incr_dcr} is bounded as follows:
    $\mathbb{E}\left[R(T)\right] \leq O\bigg(\max_{w\in \Delta^{K-1}}\sqrt{T}\sum_{i=1}^K\frac{\sqrt{w_i}}{\eta_{1,i}}+\mathbb{E}[R_{i^*}(T)]\bigg).$
\end{theorem}
The bound in Theorem~\ref{thm:regret_bound_adv} essentially evaluates
to
$O(\max(\sqrt{TK},\max_{i\in [K]}\bar R_{i}(T)) +
\mathbb{E}[R_{i^*}(T)])$. Unfortunately, this is not quite enough to
recover the results in
\citep{agarwal2016corralling,pacchiano2020model}. This is attributed
to the fact that we use the $\frac{1}{2}$-Tsallis entropy as the
regularizer instead of the log-barrier function. It is possible to
improve the above bound for algorithms with stability $\gamma < 1/2$,
however, because model selection is not the primary focus of this
work, we will not present such results here.

\textbf{Stability of UCB-I.} We now briefly discuss how the regret
bounds of UCB-I and similar algorithms change whenever the variance of
the stochastic losses is rescaled by
Algorithm~\ref{alg:tsallis_inf_incr_dcr}. Let us focus on base learner
$\cA_i$ during epoch $\tau_j$. During epoch $\tau_j$, there is some
largest threshold $\rho_{s_i}$ which is never exceeded by the inverse
probabilities, i.e., $\min_{t \in \tau_j} w_{t,i} \geq
1/\rho_{s_i}$. This implies that the rescaled losses are in
$[0,\rho_{s_i}]$. Further, their variance is bounded by
$\mathbb{E}[\hat \ell_t(i)^2] =
\mathbb{E}[\ell_t(a_{i,j_t})^2/w_{t,i}] \leq\rho_{s_i}$. Using a
version of Freedman's inequality~\citep{Freedman1975}, we show the
following.
\begin{theorem}[Informal]
\label{thm:ucb_stable}
Suppose that during epoch $\tau_j$ of size $\cT_j$, UCB-I 
\citep{auer2002finite} uses an upper confidence bound $\sqrt{\frac{4\rho_{s_i}\log{t}}{T_{i,j}(t)}} +\frac{4\rho_{s_i}\log{t}}{3T_{i,j}(t)}$ for arm $j$ at time $t$. Then, the expected regret of $\cA_i$ under the rescaled rewards is at most $\mathbb{E}[R_{i}(\cT_j)] \leq \sqrt{8\rho_{s_i} k_i\cT_j\log{\cT_j}}.$
\end{theorem}
We expect that other UCB-type algorithms~\citep{audibert2009exploration,garivier2011kl,bubeck2013bandits,garivier2018kl} should also be $\frac{1}{2}$-stable.

\section{Empirical results}
\label{sec:experiments}

In this section, we further examine the empirical properties of our
algorithms via experiments on synthetically generated datasets.  We
compare Algorithm~\ref{alg:ucb-c-gen} and
Algorithm~\ref{alg:tsallis_inf_incr_dcr} to the Corral algorithm
\citep{agarwal2016corralling}[Algorithm 1], which is also used in
\citep{pacchiano2020model}.
We note that \cite{pacchiano2020model} also use Exp3.P as a corralling
algorithm.  Recent work \citep{lee2020bias} suggests that Corral
exhibits similar high probability regret guarantees as Exp3.P and that
Corral would completely outperform Exp3.P.

\paragraph{Experimental setup.} 
The algorithms that we corral are UCB-I, Thompson sampling (TS), and
FTRL with $\frac{1}{2}$-Tsallis entropy reguralizer (Tsallis-INF).
When implementing Algorithm~\ref{alg:tsallis_inf_incr_dcr} and Corral,
we make an important deviation from what theory prescribes: we
\emph{never} restart the corralled algorithms and run them with their
default parameters.
In all our experiments, we corral two instances of UCB-I, TS, and
Tsallis-INF for a total of six algorithms.  The algorithm containing
the best arm plays over $10$ arms.  Every other algorithm plays over
$5$ arms.
The rewards for each base algorithm are Bernoulli random variables
with expectations set so that for all $i>2$ and $j>1$,
$\Delta_{i, j} = 0.01$. We run two sets of experiments with $\Delta_i$
equal to either $0.2$ or $0.02$. This setting implies that Algorithm~1
always contains the best arm and that the best arm of each base
algorithm is arm one. Even though $\Delta_{i,j} = 0.01$ implies large
regret for all sub-optimal algorithms, it also reduces the variance of
the total reward for these algorithms thereby making the corralling
problem harder.
Finally, the time horizon is set to $T = 10^6$. For a more extensive discussion,
about our choice of algorithms and parameters for the experimental setup we refer 
the reader to Appendix~\ref{app:experiments}.

\paragraph{Large gap experiments.}

\begin{table}[t]
\centering
\begin{tabular}{@{\hspace{0cm}}|@{\hspace{0cm}}c@{\hspace{0cm}}|@{\hspace{0cm}}c@{\hspace{0cm}}|@{\hspace{0cm}}c@{\hspace{0cm}}|@{\hspace{0cm}}}
\hline
\multicolumn{1}{|c|}{\textsc{Corral}} 
& Algorithm~\ref{alg:tsallis_inf_incr_dcr} 
& Algorithm~\ref{alg:ucb-c-gen} \\
\hline
\includegraphics[width=0.33\linewidth]{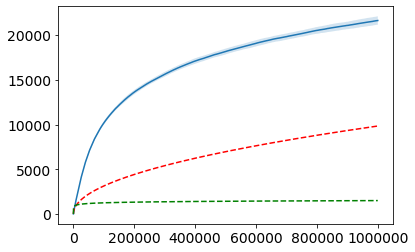}       
& \includegraphics[width=0.33\linewidth]{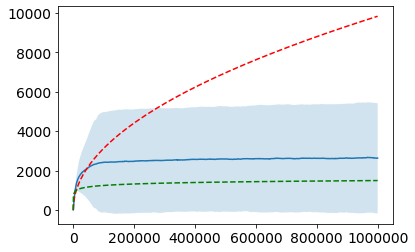}
& \includegraphics[width=0.33\linewidth]{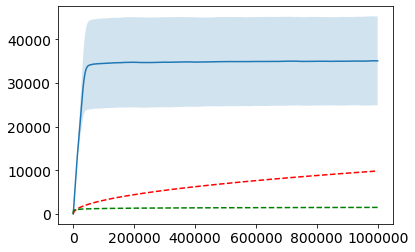}
\\ \hline
\includegraphics[width=0.33\linewidth]{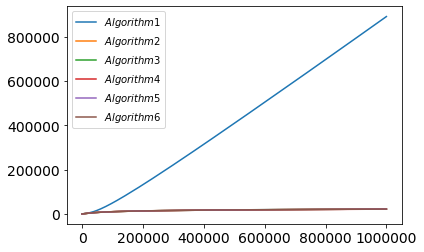}       
& \includegraphics[width=0.33\linewidth]{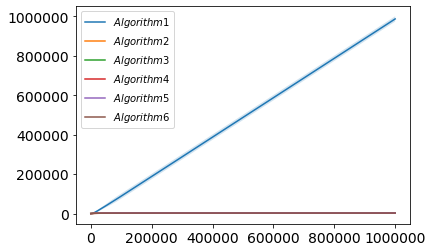}
& \includegraphics[width=0.33\linewidth]{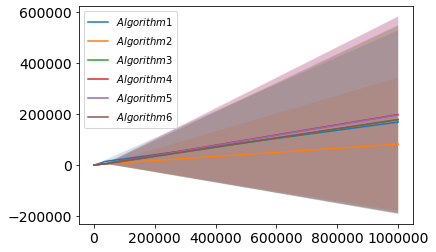}
\\ \hline
\end{tabular}
\caption{Regret for corralling when $\Delta_i = 0.2$}
\label{table:large_gap_reg}
\end{table}

Table~\ref{table:large_gap_reg} reports the regret (top) and number of plays of each algorithm found in our
experiments when $\Delta_i = 0.2$.  The plots represent the average
regret, in blue, and the average number of pulls of each algorithm (color according to the legend) over $75$ runs of each experiment. The standard
deviation is represented by the shaded blue region. The algorithm that
contains the optimal arm is $\cA_1$ and is an instance of UCB-I. The
red dotted line in the top plots is given by $4\sqrt{KT} + \mathbb{E}[R_1(T)]$, and the
green dotted line is given by
$4 \sum_{i\neq 1} \frac{k_i\log{T}}{\Delta_i} +
\mathbb{E}[R_1(T)]$. These lines serve as a reference across experiments and we believe they are more accurate upper bounds for the regret of the proposed
and existing algorithms. As expected, we see that, in the large gap
regime, the Corral algorithm exhibits $\Omega(\sqrt{T})$ regret, while
the regret of Algorithm~\ref{alg:tsallis_inf_incr_dcr} remains bounded
in $O(\log{T})$. Algorithm~\ref{alg:ucb-c-gen} admits two regret
phases. In the initial phase, its regret is linear, while in the
second phase it is logarithmic.  This is typical of UCB strategies in
the stochastic MAB problem~\citep{garivier2018explore}.

\paragraph{Small gap experiments}
\begin{table}[t]
\centering
\begin{tabular}{@{\hspace{0cm}}|@{\hspace{0cm}}c@{\hspace{0cm}}|@{\hspace{0cm}}c@{\hspace{0cm}}|@{\hspace{0cm}}c@{\hspace{0cm}}|@{\hspace{0cm}}}
\hline
\multicolumn{1}{|c|}{\textsc{Corral}}& Algorithm~\ref{alg:tsallis_inf_incr_dcr} & Algorithm~\ref{alg:ucb-c-gen} \\
\hline
\includegraphics[width=0.33\linewidth]{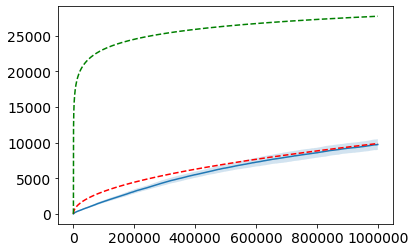}
& \includegraphics[width=0.33\linewidth]{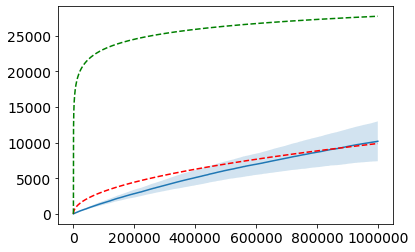}
& \includegraphics[width=0.33\linewidth]{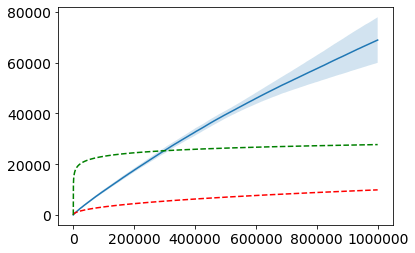}\\ 
\hline
\includegraphics[width=0.33\linewidth]{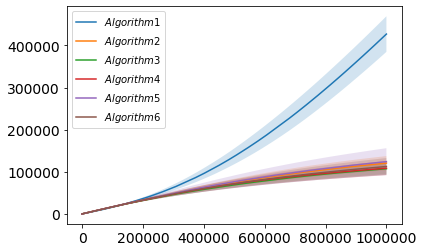}
& \includegraphics[width=0.33\linewidth]{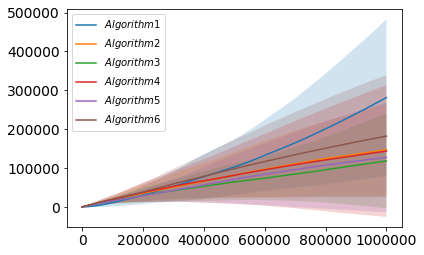}
& \includegraphics[width=0.33\linewidth]{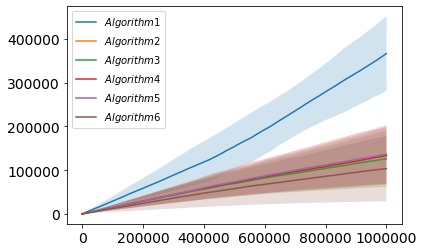}\\ 
\hline
\end{tabular}
\caption{Regret for corralling when $\Delta_i = 0.02$}
\label{table:small_gap_reg}
\end{table}

Table~\ref{table:small_gap_reg} reports the results of our experiments
for $\Delta_i = 0.02$. The setting of the experiments is the same as
in the large gap case. We observe that both Corral and
Algorithm~\ref{alg:tsallis_inf_incr_dcr} behave according to the
$O(\sqrt{T})$ bounds. This is expected since, when $\Delta_{i}=0.02$,
the optimistic bound dominates the $\sqrt{T}$-bound. The result for
Algorithm~\ref{alg:ucb-c-gen} might be somewhat surprising, as its
regret exceeds both the green and red lines.  We emphasize that this
experiment does not contradict Theorem~\ref{thm:regret_bound_ucbc}.
Indeed, if we were to plot the green and red lines according to the
bounds of Theorem~\ref{thm:regret_bound_ucbc}, the regret would remain
below both lines.

Our experiments suggest that Algorithm~\ref{alg:tsallis_inf_incr_dcr}
is the best corralling algorithm. A tighter analysis would potentially
yield optimistic regret bounds in the order of
$O\left(\sum_{i\neq i^*} \frac{k_i\log{T}}{\Delta_i} +
  \mathbb{E}[R_{i^*}(T)]\right)$. Furthermore, we expect that the
bounds of Theorem~\ref{thm:regret_bound_ucbc} are tight.  For more
detailed experiments, we refer the reader to
Appendix~\ref{app:experiments}.

\section{Model selection for linear bandits}
\label{sec:model-selection}

While the main focus of the paper is corralling MAB base learners when there exists a best overall base algorithm, we now demonstrate that
several known model selection results can be recovered
using  Algorithm~\ref{alg:tsallis_inf_incr_dcr}.

We begin by recalling the model selection problem for linear bandits. The learner is given access to a set of loss functions $\mathcal{F}\colon \mathcal{X}\times\mathcal{A} \rightarrow \mathbb{R}$ mapping from contexts $\mathcal{X}$ and actions $\mathcal{A}$ to losses. In the linear bandits setting, $\mathcal{F}$ is structured as a nested sequence of classes $\mathcal{F}_1\subseteq \mathcal{F}_2\subseteq\ldots\subseteq\mathcal{F}_K = \mathcal{F}$, where each $\mathcal{F}_i$ is defined as
\begin{align*}
    \mathcal{F}_i = \{(x,a) \rightarrow \langle\beta_i, \phi_i(x,a)\rangle: \beta_i \in \mathbb{R}^{d_i}\},
\end{align*}
for some feature embedding $\phi_i \colon \mathcal{X}\times\mathcal{A} \rightarrow \mathbb{R}^{d_i}$. It is assumed that each feature embedding $\phi_i$ contains $\phi_{i-1}$ as its first $d_{i-1}$ coordinates. It is further assumed that there exists a smallest $i^* \leq K$ to which the optimal parameter $\beta^*$ belongs, that is the observed losses for each context-action pair $(x,a)$ satisfy $\ell_t(x,a) = \mathbb{E}\left[\langle\beta^*,\phi_{d_{i^*}}(x,a) \rangle\right]$. The goal in the model selection problem is to identify $i^*$ and compete against the smallest loss for the $t$-th context in $\mathbb{R}^{d_{i^*}}$ by minimizing the regret:
\begin{align*}
    & R_{i^*}(T) = \sum_{t=1}^T \Big( \mathbb{E}[\langle \beta^*, \phi_{i^*}(x_t,a_t) \rangle]  \\
    & \hspace*{100pt}
    - \min_{a\in\mathcal{A}}\mathbb{E}[\langle \beta^*, \phi_{i^*}(x_t,a) \rangle] \Big),
\end{align*}
where the expectation is with respect to all randomness in the sampling of the contexts $x_t\sim\mathcal{D}$, actions and additional noise. We adopt the standard assumption that, given $x_t$, the observed loss for any $a$ can be expressed as follows: $\langle \beta, \phi_i(x_t,a) \rangle + \xi_t$, where $\xi_t$ is a zero-mean, sub-Gaussian random variable with variance proxy $1$ and for each of the context-action pairs it holds that $\langle \beta, \phi_i(x_t,a) \rangle \in [0,1]$.

\subsection{Algorithm and main result}

We assume that there are $K$ base learners $\{\mathcal{A}_i\}_{i=1}^K$ such that the regret of $\mathcal{A}_i$, for $i\geq i^*$, is bounded by $\tilde O(d_i^{\alpha}\sqrt{T})$. That is, whenever the model is correctly specified, the $i$-th algorithm admits a meaningful regret guarantee. In the setting of \cite{foster2019model}, $\mathcal{A}_i$ can be instantiated as \textsc{LinUCB} and in that case $\alpha = 1/2$. Further, in the setting of infinite arms, $\mathcal{A}_i$ can be instantiated as \textsc{OFUL} \citep{abbasi2011improved}, in which case $\alpha = 1$. Both $\alpha=1/2$ and $\alpha=1$ govern the min-max optimal rates in the respective settings. Our algorithm is now a simple modification of Algorithm~\ref{alg:tsallis_inf_incr_dcr}. At every time-step $t$, we update $\hat L_t = \hat L_{t-1} + \hat\ell_{t} + \mathbf{d}$, where $\mathbf{d}_i = \frac{d_i^{2\alpha}}{\sqrt{T}}$. Intuitively, our modification creates a gap between the losses of $\mathcal{A}_{i^*}$ and any $\mathcal{A}_i$ for $i>i^*$ of the order $d_i^{2\alpha}$. On the other hand for any $i<i^*$, perturbing the loss can result in at most additional $d_{i^*}^{2\alpha}\sqrt{T}$ regret. With the above observations, the bound guaranteed by Theorem~\ref{thm:regret_bound} implies that the modified algorithm should incur at most $\tilde O(d_{i^*}^{2\alpha}\sqrt{T})$ regret. In Appendix~\ref{app:model_selection}, we show the following regret bound.
\begin{theorem}
\label{thm:model_selection}
Assume that every base learner $\mathcal{A}_i$, 
$i\geq i^*$, admits a $\tilde O(d_{i}^{\alpha}\sqrt{T})$ regret. Then, there exists a corralling strategy with expected regret bounded by $\tilde O(d_{i^*}^{2\alpha}\sqrt{T} + K\sqrt{T})$. Moreover, under the additional assumption 
that the following holds for any $i<i^*$, for all $(x,a) \in \mathcal{X}\times\mathcal{A}$
\begin{align*}
    \mathbb{E}[\langle \beta_i,\phi_i(x,a) \rangle] - \min_{a \in \mathcal{A}} \mathbb{E}[\langle \beta^*,\phi_{i^*}(x,a)\rangle] \geq 2\frac{d_{i^*}^{2\alpha} - d_{i}^{2\alpha}}{\sqrt{T}},
\end{align*}
the expected regret of the same strategy is bounded as $\tilde O(d_{i^*}^{\alpha}\sqrt{T} + K\sqrt{T})$.
\end{theorem}

Typically, we have $K = O(\log{T})$ 
and thus Theorem~\ref{thm:model_selection} guarantees a regret of at most $\tilde O(d_{i^*}^{2\alpha}\sqrt{T})$. Furthermore, 
under a gap-assumption, which implies that the value of the smallest loss for the optimal embedding $i^*$ is sufficiently smaller compared to the value of any sub-optimal embedding $i<i^*$, we can actually achieve a corralling regret of the order $R_{i^*}(T)$. In particular, for the setting of \cite{foster2019model}, our strategy yields the desired $\tilde O(\sqrt{d_{i^*}T})$ regret bound. Notice that the regret guarantees are only meaningful as long as $d_{i^*} = o(T^{1/(2\alpha)})$. In such a case, 
the second assumption on the gap is that the gap is lower bounded by $o(1)$. This is a completely problem-dependent assumption and in general we expect that it cannot be satisfied. 

\section{Conclusion}

We presented an extensive analysis of the problem of corralling
stochastic bandits. Our algorithms are applicable to a number of
different contexts where this problem arises. There are also several
natural extensions and related questions relevant to our study. One
natural extension is the case where the set of arms accessible to the
base algorithms admit some overlap and where the reward observed by
one algorithm could serve as side-information to another
algorithm. Another extension is the scenario of corralling online
learning algorithms with feedback graphs. In addition to these and
many other interesting extensions, our analysis may have some
connection with the study of other problems such as model selection in
contextual bandits \citep{foster2019model} or active learning.

\section*{Acknowledgements}
This research was supported in part by NSF BIGDATA awards
IIS-1546482, IIS-1838139, NSF CAREER award IIS-1943251, and by NSF
CCF-1535987, NSF IIS-1618662, and a Google Research Award. RA would
like to acknowledge support provided by Institute for Advanced Study
and the Johns Hopkins Institute for Assured Autonomy. We 
warmly thank Julian Zimmert for insightful discussions 
regarding the Tsallis-INF approach.

\bibliographystyle{plainnat}
\bibliography{arXiv_bib}
\appendix
\onecolumn
\section{Additional experiments}
\label{app:experiments}
We now provide more detailed plots for our experiments, including number of times each corralled algorithm has been played and the distribution over corralled distribution each of the corralling algorithm keeps (in the case of Algorithm~\ref{alg:ucb-c-gen} this is just the empirical distribution of played algorithms). We additionally present experiments in which the corralled algorithm containing the best arm is FTRL with $\frac{1}{2}$-Tsallis entropy regularization and Thompson sampling. 

\paragraph{Detailed experimental setup.}
The algorithms which we corral are UCB-I, Thompson sampling (TS), and FTRL with $\frac{1}{2}$-Tsallis entropy reguralizer (Tsallis-INF). We chose these algorithms as they all come with regret guarantees for the stochastic multi-armed problem and they broadly represent three different classes of algorithms, i.e, algorithms based on the optimism in the face of uncertainty principle, algorithms based on posterior sampling, and algorithms based on online mirror descent. As already discussed in Section~\ref{sec:experiments}, when implementing Algorithm~\ref{alg:tsallis_inf_incr_dcr} and Corral, we \emph{never} restart the corralled algorithms and run them with their default parameters. Even though, there are no theoretical guarantees for this modification of the corralling algorithms, we will see that the regret bounds remain meaningful in practice. In all of the experiments we corral two instances of UCB-I, TS, and FTRL for a total of six algorithms. The best algorithm plays over 10 arms. Every other algorithm plays over 5 arms. Intuitively, the higher the number of arms implies higher complexity of the best algorithm which would lead to higher regret and a harder corralling problem. The rewards for each algorithm are Bernoulli random variables setup according to the following parameters: \textsc{base\_reward}, \textsc{in\_gap}, \textsc{out\_gap}, and \textsc{low\_reward}. The best overall arm has expected reward $\textsc{base\_reward} + \textsc{in\_gap} + \textsc{out\_gap}$. Every other arm of Algorithm 1 has expected reward equal to \textsc{low\_reward}. For all other algorithms the best arm has reward $\textsc{base\_reward} + \textsc{in\_gap}$ and other arms have reward $\textsc{base\_reward}$. 
In all of the experiments we set $\textsc{base\_reward} = 0.5,\textsc{in\_gap} =0.01, \textsc{low\_reward}=0.2$. 
While a small $\textsc{in\_gap}$ implies a large regret for the algorithms containing sub-optimal arms, it also reduces the likelihood that said algorithms would have small average reward. Combined with setting $\textsc{low\_reward} = 0.2$, this will make the average reward of $\cA_1$ look small in the initial number of rounds, compared to the average reward of $\cA_i,i>1$ and hence makes the corralling problem harder. We run two set of experiments, an easy set for which $\textsc{out\_gap}=0.19$, which translates to gaps $\Delta_i = 0.2$ in our regret bounds, and a hard set for which $\textsc{out\_gap}=0.01$ which implies $\Delta_i=0.02$. Finally time horizon is set to $T=1000000$.

\subsection{UCB-I contains best arm}
Experiments can be found in Figure~\ref{fig:large_gap_ucb} for $\Delta_i = 0.2$ and in Figure~\ref{fig:small_gap_ucb} for $\Delta_i = 0.02$.
\begin{figure*}[h!]
    \centering
\begin{subfigure}{0.25\textwidth}
  \includegraphics[width=\linewidth]{plots/UCBbest/Regret_largeGap_LogBar.png}
  \caption{Corral regret}
\end{subfigure}\hfil
\begin{subfigure}{0.25\textwidth}
  \includegraphics[width=\linewidth]{plots/UCBbest/NumPulls_largeGap_LogBar.png}
  \caption{Corral number of pulls}
\end{subfigure}\hfil
\begin{subfigure}{0.25\textwidth}
  \includegraphics[width=\linewidth]{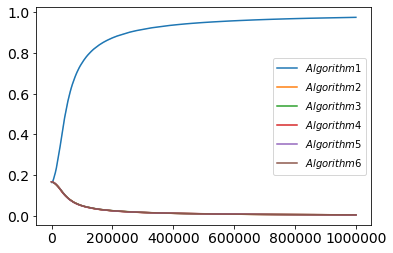}
  \caption{Corral distribution}
\end{subfigure}

\begin{subfigure}{0.25\textwidth}
  \includegraphics[width=\linewidth]{plots/UCBbest/Regret_largeGap_Tsallis.png}
  \caption{Algorithm~\ref{alg:tsallis_inf_incr_dcr} regret}
\end{subfigure}\hfil
\begin{subfigure}{0.25\textwidth}
  \includegraphics[width=\linewidth]{plots/UCBbest/NumPulls_largeGap_Tsallis.png}
  \caption{Algorithm~\ref{alg:tsallis_inf_incr_dcr} number of pulls}
\end{subfigure}\hfil
\begin{subfigure}{0.25\textwidth}
  \includegraphics[width=\linewidth]{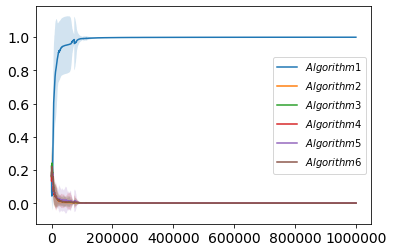}
  \caption{Algorithm~\ref{alg:tsallis_inf_incr_dcr} distribution}
\end{subfigure}

\begin{subfigure}{0.25\textwidth}
  \includegraphics[width=\linewidth]{plots/UCBbest/Regret_largeGap_UCB.png}
  \caption{Algorithm~\ref{alg:ucb-c-gen} regret}
\end{subfigure}\hfil
\begin{subfigure}{0.25\textwidth}
  \includegraphics[width=\linewidth]{plots/UCBbest/NumPulls_largeGap_UCB.png}
  \caption{Algorithm~\ref{alg:ucb-c-gen} number of pulls}
\end{subfigure}\hfil
\begin{subfigure}{0.25\textwidth}
  \includegraphics[width=\linewidth]{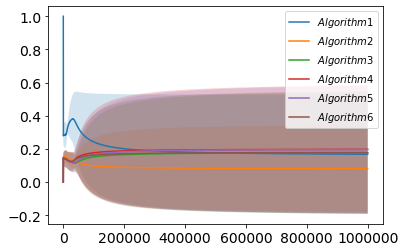}
  \caption{Algorithm~\ref{alg:ucb-c-gen} distribution}
\end{subfigure}
\caption{\text{UCB-I} contains best arm,$\Delta_i=0.2,\textsc{ALG}_{1:2} = \textsc{UCB-I},\textsc{ALG}_{3:4} = \textsc{Tsallis-INF},\textsc{ALG}_{5:6} = \textsc{TS}$.}
\label{fig:large_gap_ucb}
\end{figure*}
\newpage

\begin{figure*}[h!]
    \centering
\begin{subfigure}{0.25\textwidth}
  \includegraphics[width=\linewidth]{plots/UCBbest/Regret_smallGap_LogBar.png}
  \caption{Corral regret}
\end{subfigure}\hfil
\begin{subfigure}{0.25\textwidth}
  \includegraphics[width=\linewidth]{plots/UCBbest/NumPulls_smallGap_LogBar.png}
  \caption{Corral number of pulls}
\end{subfigure}\hfil
\begin{subfigure}{0.25\textwidth}
  \includegraphics[width=\linewidth]{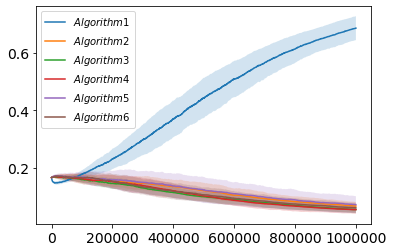}
  \caption{Corral distribution}
\end{subfigure}

\begin{subfigure}{0.25\textwidth}
  \includegraphics[width=\linewidth]{plots/UCBbest/Regret_smallGap_Tsallis.png}
  \caption{Algorithm~\ref{alg:tsallis_inf_incr_dcr} regret}
\end{subfigure}\hfil
\begin{subfigure}{0.25\textwidth}
  \includegraphics[width=\linewidth]{plots/UCBbest/NumPulls_smallGap_Tsallis.png}
  \caption{Algorithm~\ref{alg:tsallis_inf_incr_dcr} number of pulls}
\end{subfigure}\hfil
\begin{subfigure}{0.25\textwidth}
  \includegraphics[width=\linewidth]{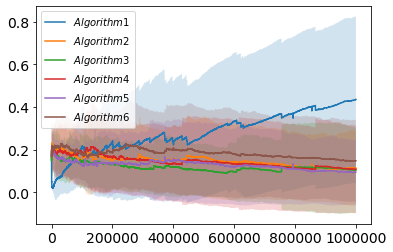}
  \caption{Algorithm~\ref{alg:tsallis_inf_incr_dcr} distribution}
\end{subfigure}

\begin{subfigure}{0.25\textwidth}
  \includegraphics[width=\linewidth]{plots/UCBbest/Regret_smallGap_UCB.png}
  \caption{Algorithm~\ref{alg:ucb-c-gen} regret}
\end{subfigure}\hfil
\begin{subfigure}{0.25\textwidth}
  \includegraphics[width=\linewidth]{plots/UCBbest/NumPulls_smallGap_UCB.png}
  \caption{Algorithm~\ref{alg:ucb-c-gen} number of pulls}
\end{subfigure}\hfil
\begin{subfigure}{0.25\textwidth}
  \includegraphics[width=\linewidth]{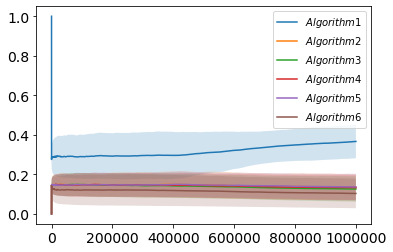}
  \caption{Algorithm~\ref{alg:ucb-c-gen} distribution}
\end{subfigure}
\caption{\text{UCB-I} contains best arm,$\Delta_i=0.02,\textsc{ALG}_{1:2} = \textsc{UCB-I},\textsc{ALG}_{3:4} = \textsc{Tsallis-INF},\textsc{ALG}_{5:6} = \textsc{TS}$.}
\label{fig:small_gap_ucb}
\end{figure*}
\newpage

\subsection{Tsallis-INF contains best arm}
Experiments can be found in Figure~\ref{fig:large_gap_tsallis} for $\Delta_i = 0.2$ and in Figure~\ref{fig:small_gap_tsallis} for $\Delta_i = 0.02$.

\begin{figure*}[h!]
    \centering
\begin{subfigure}{0.25\textwidth}
  \includegraphics[width=\linewidth]{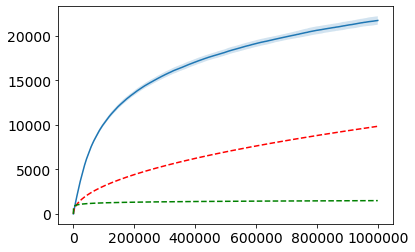}
  \caption{Corral regret}
\end{subfigure}\hfil
\begin{subfigure}{0.25\textwidth}
  \includegraphics[width=\linewidth]{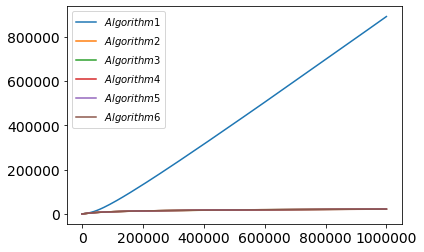}
  \caption{Corral number of pulls}
\end{subfigure}\hfil
\begin{subfigure}{0.25\textwidth}
  \includegraphics[width=\linewidth]{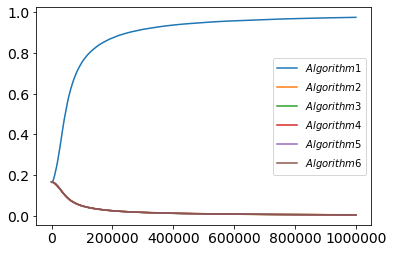}
  \caption{Corral distribution}
\end{subfigure}

\begin{subfigure}{0.25\textwidth}
  \includegraphics[width=\linewidth]{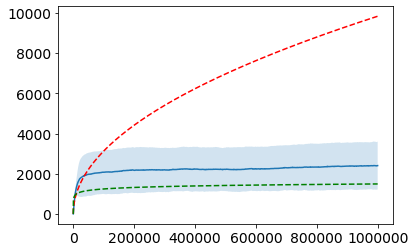}
  \caption{Algorithm~\ref{alg:tsallis_inf_incr_dcr} regret}
\end{subfigure}\hfil
\begin{subfigure}{0.25\textwidth}
  \includegraphics[width=\linewidth]{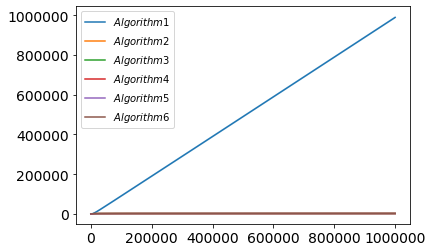}
  \caption{Algorithm~\ref{alg:tsallis_inf_incr_dcr} number of pulls}
\end{subfigure}\hfil
\begin{subfigure}{0.25\textwidth}
  \includegraphics[width=\linewidth]{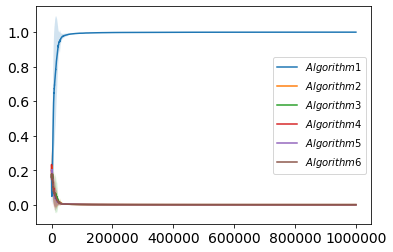}
  \caption{Algorithm~\ref{alg:tsallis_inf_incr_dcr} distribution}
\end{subfigure}

\begin{subfigure}{0.25\textwidth}
  \includegraphics[width=\linewidth]{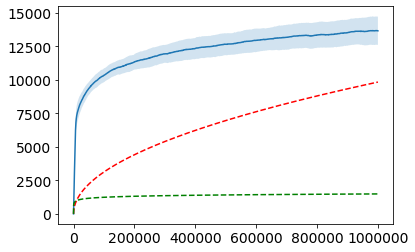}
  \caption{Algorithm~\ref{alg:ucb-c-gen} regret}
\end{subfigure}\hfil
\begin{subfigure}{0.25\textwidth}
  \includegraphics[width=\linewidth]{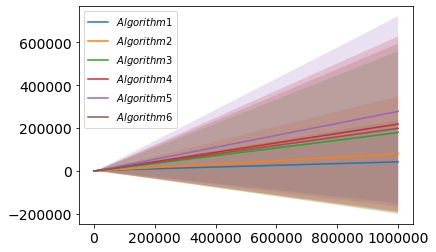}
  \caption{Algorithm~\ref{alg:ucb-c-gen} number of pulls}
\end{subfigure}\hfil
\begin{subfigure}{0.25\textwidth}
  \includegraphics[width=\linewidth]{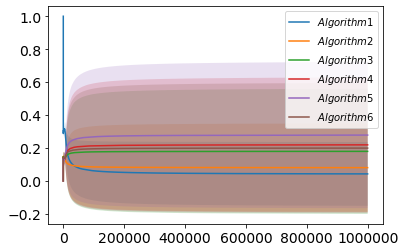}
  \caption{Algorithm~\ref{alg:ucb-c-gen} distribution}
\end{subfigure}
\caption{\text{Tsallis-INF} contains best arm,$\Delta_i=0.2,\textsc{ALG}_{1:2} = \textsc{Tsallis-INF},\textsc{ALG}_{3:4} = \textsc{UCB-I},\textsc{ALG}_{5:6} = \textsc{TS}$.}
\label{fig:large_gap_tsallis}
\end{figure*}
\newpage

\begin{figure*}[h!]
    \centering
\begin{subfigure}{0.25\textwidth}
  \includegraphics[width=\linewidth]{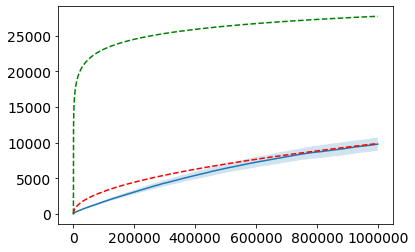}
  \caption{Corral regret}
\end{subfigure}\hfil
\begin{subfigure}{0.25\textwidth}
  \includegraphics[width=\linewidth]{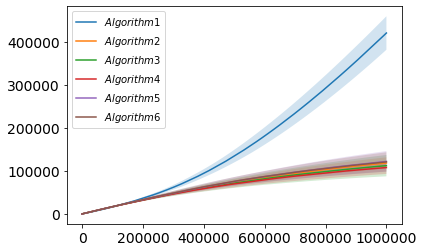}
  \caption{Corral number of pulls}
\end{subfigure}\hfil
\begin{subfigure}{0.25\textwidth}
  \includegraphics[width=\linewidth]{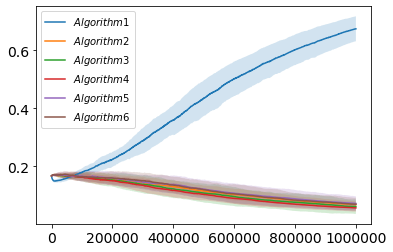}
  \caption{Corral distribution}
\end{subfigure}

\begin{subfigure}{0.25\textwidth}
  \includegraphics[width=\linewidth]{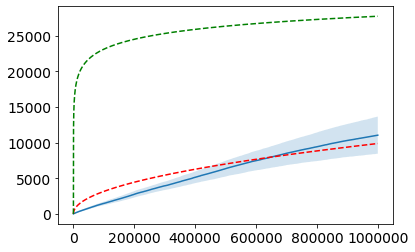}
  \caption{Algorithm~\ref{alg:tsallis_inf_incr_dcr} regret}
\end{subfigure}\hfil
\begin{subfigure}{0.25\textwidth}
  \includegraphics[width=\linewidth]{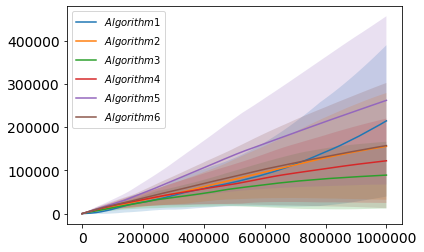}
  \caption{Algorithm~\ref{alg:tsallis_inf_incr_dcr} number of pulls}
\end{subfigure}\hfil
\begin{subfigure}{0.25\textwidth}
  \includegraphics[width=\linewidth]{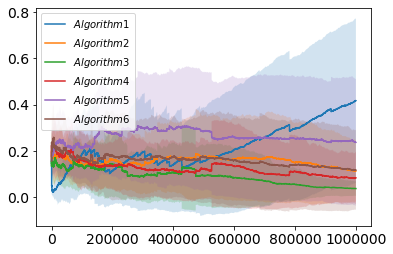}
  \caption{Algorithm~\ref{alg:tsallis_inf_incr_dcr} distribution}
\end{subfigure}

\begin{subfigure}{0.25\textwidth}
  \includegraphics[width=\linewidth]{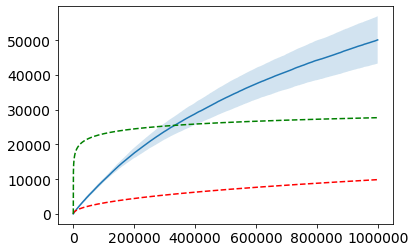}
  \caption{Algorithm~\ref{alg:ucb-c-gen} regret}
\end{subfigure}\hfil
\begin{subfigure}{0.25\textwidth}
  \includegraphics[width=\linewidth]{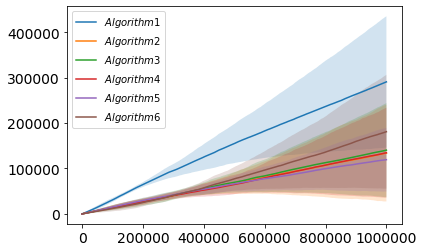}
  \caption{Algorithm~\ref{alg:ucb-c-gen} number of pulls}
\end{subfigure}\hfil
\begin{subfigure}{0.25\textwidth}
  \includegraphics[width=\linewidth]{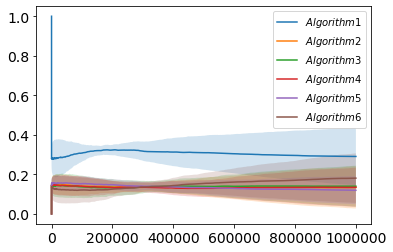}
  \caption{Algorithm~\ref{alg:ucb-c-gen} distribution}
\end{subfigure}
\caption{\text{Tsallis-INF} contains best arm,$\Delta_i=0.02,\textsc{ALG}_{1:2} = \textsc{Tsallis-INF},\textsc{ALG}_{3:4} = \textsc{UCB-I},\textsc{ALG}_{5:6} = \textsc{TS}$.}
\label{fig:small_gap_tsallis}
\end{figure*}
\newpage

\subsection{Thompson sampling contains best arm}
Experiments can be found in Figure~\ref{fig:large_gap_TS} for $\Delta_i = 0.2$ and in Figure~\ref{fig:small_gap_TS} for $\Delta_i = 0.02$.

\begin{figure*}[h!]
    \centering
\begin{subfigure}{0.25\textwidth}
  \includegraphics[width=\linewidth]{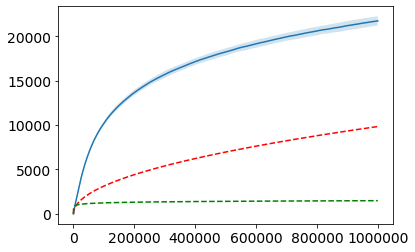}
  \caption{Corral regret}
\end{subfigure}\hfil
\begin{subfigure}{0.25\textwidth}
  \includegraphics[width=\linewidth]{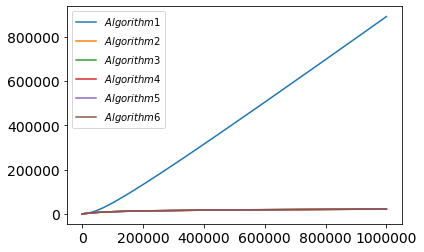}
  \caption{Corral number of pulls}
\end{subfigure}\hfil
\begin{subfigure}{0.25\textwidth}
  \includegraphics[width=\linewidth]{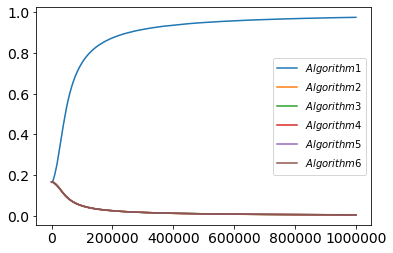}
  \caption{Corral distribution}
\end{subfigure}

\begin{subfigure}{0.25\textwidth}
  \includegraphics[width=\linewidth]{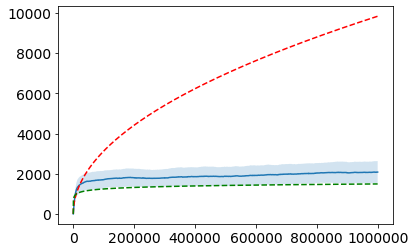}
  \caption{Algorithm~\ref{alg:tsallis_inf_incr_dcr} regret}
\end{subfigure}\hfil
\begin{subfigure}{0.25\textwidth}
  \includegraphics[width=\linewidth]{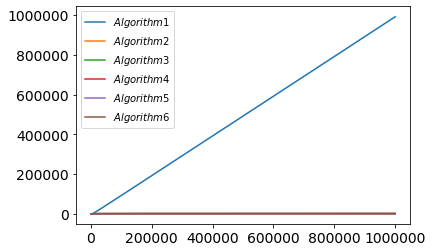}
  \caption{Algorithm~\ref{alg:tsallis_inf_incr_dcr} number of pulls}
\end{subfigure}\hfil
\begin{subfigure}{0.25\textwidth}
  \includegraphics[width=\linewidth]{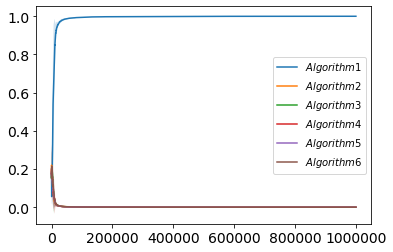}
  \caption{Algorithm~\ref{alg:tsallis_inf_incr_dcr} distribution}
\end{subfigure}

\begin{subfigure}{0.25\textwidth}
  \includegraphics[width=\linewidth]{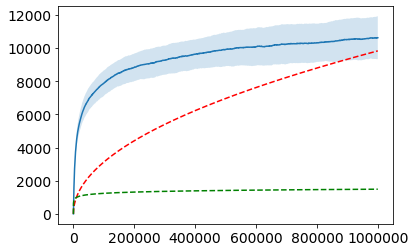}
  \caption{Algorithm~\ref{alg:ucb-c-gen} regret}
\end{subfigure}\hfil
\begin{subfigure}{0.25\textwidth}
  \includegraphics[width=\linewidth]{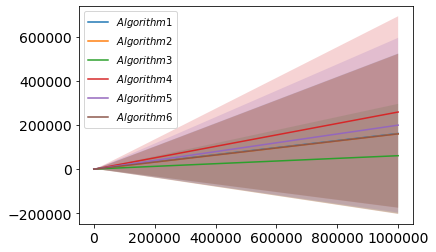}
  \caption{Algorithm~\ref{alg:ucb-c-gen} number of pulls}
\end{subfigure}\hfil
\begin{subfigure}{0.25\textwidth}
  \includegraphics[width=\linewidth]{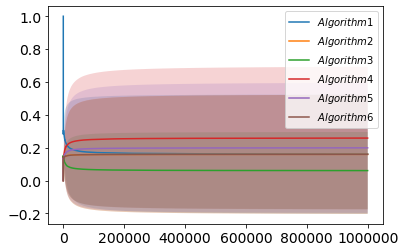}
  \caption{Algorithm~\ref{alg:ucb-c-gen} distribution}
\end{subfigure}
\caption{\text{Thompson sampling (TS)} contains best arm,$\Delta_i=0.2,\textsc{ALG}_{1:2} = \textsc{TS},\textsc{ALG}_{3:4} = \textsc{UCB-I},\textsc{ALG}_{5:6} = \textsc{Tsallis-INF}$.}
\label{fig:large_gap_TS}
\end{figure*}
\newpage

\begin{figure*}[h!]
    \centering
\begin{subfigure}{0.25\textwidth}
  \includegraphics[width=\linewidth]{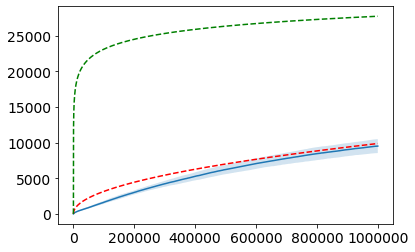}
  \caption{Corral regret}
\end{subfigure}\hfil
\begin{subfigure}{0.25\textwidth}
  \includegraphics[width=\linewidth]{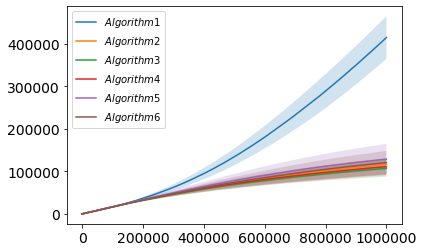}
  \caption{Corral number of pulls}
\end{subfigure}\hfil
\begin{subfigure}{0.25\textwidth}
  \includegraphics[width=\linewidth]{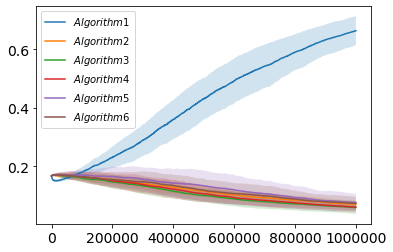}
  \caption{Corral distribution}
\end{subfigure}

\begin{subfigure}{0.25\textwidth}
  \includegraphics[width=\linewidth]{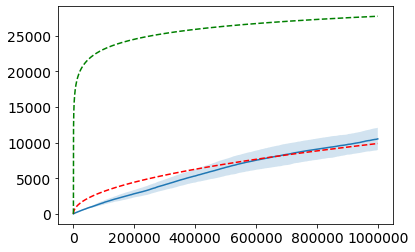}
  \caption{Algorithm~\ref{alg:tsallis_inf_incr_dcr} regret}
\end{subfigure}\hfil
\begin{subfigure}{0.25\textwidth}
  \includegraphics[width=\linewidth]{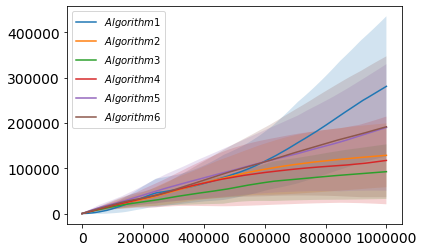}
  \caption{Algorithm~\ref{alg:tsallis_inf_incr_dcr} number of pulls}
\end{subfigure}\hfil
\begin{subfigure}{0.25\textwidth}
  \includegraphics[width=\linewidth]{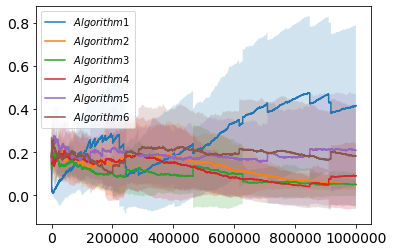}
  \caption{Algorithm~\ref{alg:tsallis_inf_incr_dcr} distribution}
\end{subfigure}

\begin{subfigure}{0.25\textwidth}
  \includegraphics[width=\linewidth]{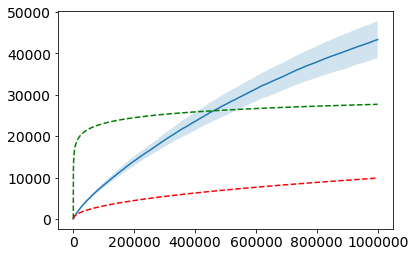}
  \caption{Algorithm~\ref{alg:ucb-c-gen} regret}
\end{subfigure}\hfil
\begin{subfigure}{0.25\textwidth}
  \includegraphics[width=\linewidth]{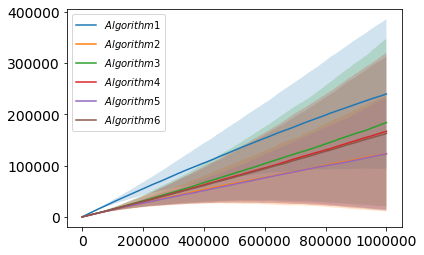}
  \caption{Algorithm~\ref{alg:ucb-c-gen} number of pulls}
\end{subfigure}\hfil
\begin{subfigure}{0.25\textwidth}
  \includegraphics[width=\linewidth]{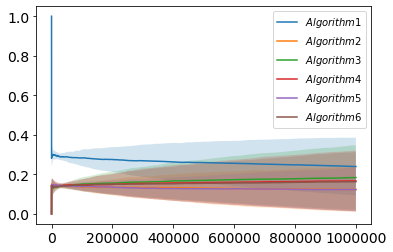}
  \caption{Algorithm~\ref{alg:ucb-c-gen} distribution}
\end{subfigure}
\caption{TS contains best arm,$\Delta_i=0.02,\textsc{ALG}_{1:2} = \textsc{TS},\textsc{ALG}_{3:4} = \textsc{UCB-I},\textsc{ALG}_{5:6} = \textsc{Tsallis-INF}$.}
\label{fig:small_gap_TS}
\end{figure*}
\newpage

\section{Proofs from Section~\ref{sec:lower_bounds}}
\label{app:lower_bounds}
We first introduce the formal construction briefly 
described in Section~\ref{sec:lower_bounds}.

\subsection{First lower bound}

Assume that the corralling algorithm can play one of two
algorithms, $\cA_1$ or $\cA_2$, with the rewards of each arm played by
these algorithms distributed according to a Bernoulli random
variable. Algorithm $\cA_1$ plays a single arm with expected reward
$\mu_1$ and algorithm $\cA_2$ is defined as follows.

Let $\beta$ be drawn according to the Bernoulli distribution
$\beta \sim \Ber(\frac{1}{2})$ and let $\alpha$ be drawn uniformly
over the unit interval, $\alpha \sim \Unif[0, 1]$. If $\beta = 1$,
$\cA_2$ alternates between playing an arm with mean $\mu_2$ and an arm
with mean $\mu_3$ every round, so that the algorithm incurs linear
regret. We set $\mu_i$s such that
$\mu_2 > \mu_1 > \frac{\mu_2 + \mu_3}{2}$. 
If $\beta = 0$, then $\cA_2$ behaves in the same way as if
$\beta = 1$ for the first $T^{(1 - \alpha)}$ rounds and for the remaining
$T - T^{(1 - \alpha)}$ rounds $\cA_2$ only pulls the arm with mean
$\mu_2$. Notice that, in this setting, $\cA_2$ admits sublinear regret
almost surely. 

We denote by
$\bP(\cdot |r_1(a_{i_1, j_1}), \ldots,r_t(a_{i_t, j_t}), \beta = i)$ the natural measure on the $\sigma$-algebra generated by the observed rewards under the
environment $\beta = i$ and all the randomness of the player's
algorithm. To simplify the notation, we denote by $r_{1:t}$ the sequence $\{r_s(a_{i_s,j_s})\}_{s=1}^t$. Let $N$ denote the random variable 
counting the number of
times the corralling strategy selected $\cA_1$. Information-theoretically, 
the player can obtain a good approximation of $\mu_1$
in time $O(\log{T})$ and, therefore, for simplicity, we assume that
the player knows $\mu_1$ exactly.  Note that this can only make the
problem easier for the player. Given this information, we can assume
that the player begins by playing algorithm $\cA_2$ for $T - N + 1$
rounds and then switches to $\cA_1$ for the rest of the game. In
particular, we assume that $T - N + 1$ is the time when the player can
figure out that $\beta = 1$. We note that at time $T^{(1 - \alpha)}$
we have
$\bP(\cdot |r_{1:T^{(1 -
    \alpha)}},\beta = 1) = \bP(\cdot |r_{1:T^{(1 -
    \alpha)}},\beta = 0)$,
as the distribution of the rewards provided by $\cA_2$ do not differ
between $\beta = 1$ and $\beta = 0$. Furthermore, any random strategy
would also need to select algorithm $\cA_2$ at least
$T^{(1 - \alpha)} + 1$ rounds before it is able to distinguish between
$\beta = 1$ or $\beta = 0$. It is also important to note that under the event that $\beta = 1$, the corralling algorithm does not receive any information about the value of $\alpha$.
This allows us to show that in the setting constructed above, with at least constant probability the best algorithm i.e., $\cA_1$ when $\beta = 1$ and $\cA_2$ when $\beta=0$, has sublinear regret.
Finally, a direct computation of the
regret of this corralling strategy gives the following result.

\begin{theorem}
\label{thm:lower_bound}
Let algorithms $\cA_1$ and $\cA_2$ follow the construction in Section~\ref{sec:lower_bounds}. Then, with probability at least 
$1/2$ over the random choice of $\alpha$, any corralling strategy 
incurs regret at least $\tilde\Omega(T)$, while the regret of the 
best algorithm is at most $O(\sqrt{T})$.
\end{theorem}
\begin{proof}
Let $R(T)$ denote the regret of the corralling algorithm. Direct computation shows that if $\beta=1$ the corralling regret is
\begin{align*}
    &\mathbb{E}[R(T)|\beta=1,r_{1:T^{(1-\alpha)}},\alpha]
    = \mathbb{E}\left[\left(\mu_1-\frac{\mu_2+\mu_3}{2}\right)(T-N)|\beta = 1,r_{1:T^{(1-\alpha)}},\alpha \right]
\end{align*}
Further if $\beta = 0$ and $\cA_2$ is the best algorithm the regret of corralling is
\begin{align*}
    &\mathbb{E}[R(T)|\beta=0,r_{1:T^{(1-\alpha)}},\alpha] = \mathbb{E}\left[T^{(1-\alpha)}\mu_2 + (T-T^{(1-\alpha)})\mu_2|\beta=0,r_{1:T^{(1-\alpha)}},\alpha\right]\\
    \geq& \mathbb{E}\left[\frac{\mu_2+\mu_3}{2} T^{(1-\alpha)} + \mu_2(T-T^{(1-\alpha)})\right.\\
    -&\left. \mu_1 N - \chi_{(N \leq T - T^{(1-\alpha)})}\left(\frac{\mu_2+\mu_3}{2} T^{(1-\alpha)} + \mu_2(T-T^{(1-\alpha)} - N)\right) \right.\\
    -&\left. \chi_{(N > T - T^{(1-\alpha)})}\frac{\mu_2+\mu_3}{2}(T-N)|\beta=0,r_{1:T^{(1-\alpha)}},\alpha\right],
\end{align*}
where the characteristic functions describe the event in which we pull $\cA_1$ less times than is needed for $\cA_2$ to switch to playing the best action. Notice that the total regret for corralling is at least the above as we also need to add the regret of the best algorithm to the above.

We first consider the case $\beta=1$. Notice that in this case the corralling algorithm does not receive any information about $\alpha$ because $\cA_2$ alternates between $\mu_2$ and $\mu_3$ at all rounds. This implies $\mathbb{E}[R(T)|\beta=1,\alpha] = \mathbb{E}[R(T)|\beta=1]$. Condition on the event $N\leq T - T^{(1-\alpha)}$. We have
\begin{align*}
    \mathbb{E}[R(T)|\beta=1, N\leq T - T^{(1-\alpha)},r_{1:T^{(1-\alpha)}},\alpha] &=\\
    \mathbb{E}[R(T)|\beta=1, N\leq T - T^{(1-\alpha)},r_{1:T^{(1-\alpha)}}] &\geq \left(\mu_1-\frac{\mu_2+\mu_3}{2}\right) \mathbb{E}[T^{(1-\alpha)}|\beta=1] \\
    &=\left(\mu_1-\frac{\mu_2+\mu_3}{2}\right)\mathbb{E}[T^{(1-\alpha)}]\\
    &= \left(\mu_1-\frac{\mu_2+\mu_3}{2}\right)\frac{T - 1}{\log{T}},
\end{align*}
where in the first inequality we have replaced $N$ by $T-T^{1-\alpha}$.
Next consider the case $\beta=0$. Condition on the event $N > T-T^{(1-\alpha)}$. We have
\begin{align*}
    &\mathbb{E}[R(T)|\beta=0, N> T - T^{(1-\alpha)},r_{1:T^{(1-\alpha)}},\alpha]\\
    &= \mathbb{E}\left[\frac{\mu_2-\mu_3}{2}(T-T^{(1-\alpha)}) - \left(\mu_1-\frac{\mu_2+\mu_3}{2}\right)N|\beta=0, N> T - T^{(1-\alpha)},r_{1:T^{(1-\alpha)}},\alpha\right]\\
    &\geq \mathbb{E}\left[(\mu_2-\mu_1)T - \frac{\mu_2}{2}T^{(1-\alpha)}|\beta=0, N> T - T^{(1-\alpha)},r_{1:T^{(1-\alpha)}},\alpha\right]\\
    &= \mathbb{E}\left[(\mu_2-\mu_1)T - \frac{\mu_2}{2}T^{1-\alpha}|\alpha\right],
\end{align*}
where in the inequality we have used the fact that $N>T-T^{1-\alpha}$ to bound $-\mu_1 N$ and $T\geq N$ to bound $\frac{\mu_1+\mu_3}{2} N$.
Let $A$ denote the event $N\leq T - T^{(1-\alpha)}$. We are now ready to lower bound the regret of the player's strategy as follows.
\begin{align*}
    \mathbb{E}\left[R(T)|\alpha\right] &= \frac{1}{2}\mathbb{E}\left[ \mathbb{E}[R(T)|r_{1:T^{(1-\alpha)}},\beta=1,\alpha]+ \mathbb{E}[R(T)|r_{1:T^{(1-\alpha)}},\beta=0,\alpha]|\alpha\right]\\
    &\geq \frac{1}{2}\mathbb{E}\left[ \bP(A |r_{1:T^{(1-\alpha)}},\beta=1,\alpha)\mathbb{E}[R(T)|r_{1:T^{(1-\alpha)}},\beta=1,A,\alpha] \right.\\
    &\left.+ \bP(A^c |r_{1:T^{(1-\alpha)}},\beta=1,\alpha)\mathbb{E}[R(T)|r_{1:T^{(1-\alpha)}},\beta=0,A^c,\alpha]|\alpha\right]\\
    &\geq \frac{1}{2}\mathbb{E}\left[\bP(A |r_{1:T^{(1-\alpha)}},\beta=1,\alpha) \left(\mu_1-\frac{\mu_2+\mu_3}{2}\right)\frac{T - 1}{2\log{T}}\right.\\
    &\left.+ (1-\bP(A |r_{1:T^{(1-\alpha)}},\beta=1),\alpha)(\mu_2-\mu_1)T - \frac{\mu_2}{2}T^{1-\alpha}|\alpha\right],
\end{align*}
where in the first inequality we have used the fact that the conditional measures induced by $\beta=1$ and $\beta=0$ are equal for the first $T^{1-\alpha}$ rounds. Because $\alpha \geq 1/2$ with probability at least $1/2$ it holds that the random variable $\mathbb{E}[R(T)|\alpha] > \tilde \Omega(T)$ with probability at least $1/2$ and that the regret of $\cA_2$ when $\beta=1$ is at most $O(\sqrt{T})$.
\end{proof}

\subsection{A realistic setting for Algorithm 2}
\label{sec:realistic_lower_bound}
The behavior of $\cA_2$ for the setting given by $\beta = 0$, in the
construction above, may seem somewhat artificial: a stochastic bandit
algorithm may not be expected to behave in that manner when the gap
between $\mu_2$ and $\mu_3$ is large enough.  Here, we describe how to
set $\mu_1$, $\mu_2$ and $\mu_3$ such that the successive elimination algorithm
\citep{even2002pac} admits a similar behavior to $\cA_2$ with
$\beta = 0$. Recall that successive elimination needs at least
$1/\Delta^2$ rounds to distinguish between the arm with mean $\mu_2$
and the arm with mean $\mu_3$. In other words, for at least
$1/\Delta^2$ rounds, it will alternate between the two
arms. Therefore, we set $\frac{1}{\Delta^2} = T^{(1 - \alpha)}$ or,
equivalently, $\Delta = \frac{1}{T^{(1 - \alpha)/2}}$, and
$\mu_1 = \mu_2 - \frac{1}{4T^{(1 - \alpha)/2}}$ to yield behavior
similar to $\cA_2$. For this construction, we show the following lower
bound.

\begin{theorem}[Theorem~\ref{thm:lower_bound2} formal]
Let algorithms $\cA_1$ and $\cA_2$ follow the construction in Section~\ref{sec:realistic_lower_bound}. With probability at least $1/4$ over the random choice of $\alpha$ any corralling strategy will incur regret at least $\tilde\Omega(\sqrt{T})$ while the gap between $\mu_2$ and $\mu_3$ is such that $\Delta > \omega(T^{-1/4})$ and hence the regret of the best algorithm is at most $o(T^{1/4})$.
\end{theorem}
\begin{proof}
From the proof of Theorem~\ref{thm:lower_bound} we can compute, when $\beta=1$, we can directly compute
\begin{align*}
    &\mathbb{E}\left[R(T)|\beta=1, N\leq T - T^{(1-\alpha)},r_{1:T^{(1-\alpha)}},\alpha\right]\\
    \geq&\mathbb{E}\left[\left(\mu_1 - \frac{\mu_2 + \mu_3}{2}\right)T^{(1-\alpha)}|\beta=1, N\leq T - T^{(1-\alpha)},r_{1:T^{(1-\alpha)}},\alpha\right]\\
    =&\mathbb{E}\left[\frac{1}{4T^{(1-\alpha)/2}}T^{(1-\alpha)}|\beta=1, N\leq T - T^{(1-\alpha)},r_{1:T^{(1-\alpha)}}\right] = \frac{\sqrt{T}-1}{2\log{T}},
\end{align*}
Where in the equality we again used the fact that if $\beta=1$, the corralling algorithm receives no information about $\alpha$. Further when $\beta=0$ we have
\begin{align*}
    &\mathbb{E}\left[R(T)|\beta=0, N> T - T^{(1-\alpha)},r_{1:T^{(1-\alpha)}}(a_{T^{(1-\alpha)}}),\alpha\right]\\
    \geq&\mathbb{E}\left[(\mu_2-\mu_1)T - \frac{\mu_2}{2}T^{(1-\alpha)}|\beta=0, N> T - T^{(1-\alpha)},r_{1:T^{(1-\alpha)}},\alpha\right]\\
    =&\mathbb{E}\left[\frac{T^{(1+\alpha)/2}}{4}|\alpha\right] - \mathbb{E}\left[\frac{T^{(1-\alpha)}}{2}|\alpha\right].
\end{align*}
Again we note that with probability $1/2$ we have $\alpha\geq 1/2$ and the above expression becomes asymptotically larger than $\sqrt{T}$. The same computation as in the proof of Theorem~\ref{thm:lower_bound} finishes the proof.
\end{proof}

We note that, in our construction, if $\beta = 1$, then the inequality
$\Delta \gg \frac{1}{\sqrt{T}}$ holds almost surely. In this setting,
the instance-dependent regret bound for $\cA_2$ and successive
elimination is asymptotically smaller compared to the worst-case
instance-independent regret bounds for stochastic bandit algorithms,
which scale as $\tilde O(\sqrt{T})$ with the time horizon. This
suggests that, even though $\cA_2$ enjoys asymptotically better regret
bounds than $\tilde O(\sqrt{T})$, the corralling algorithm will
necessarily incur $\tilde \Omega(\sqrt{T})$ regret.

\subsection{A lower bound when a worst case regret bound is known}
\label{sec:lower_bound_known_regret}

Next, suppose that we know a worst case regret bound of $R_{2}(T)$
for algorithm $\cA_2$. As before, we sample $\beta$ according to a
Bernoulli distribution. If $\beta = 1$, then algorithm $\cA_2$ has a
single arm with reward distributed as $\Ber((\mu_2 + \mu_3)/2)$; in that
case, $\cA_2$ admits a regret equal to $0$. If $\beta = 0$, then $\cA_2$
has two arms distributed according to $\Ber(\mu_2)$ and $\Ber(\mu_3)$,
respectively. We sample $\alpha \sim \Unif[0,1]$, and let $\cA_2$ play
an arm uniformly at random for the first $R_{2}(T)^{(1 - \alpha)}$
rounds. In particular, during each of the first
$R_{2}(T)^{(1-\alpha)}$ rounds, $\cA_2$ plays with equal
probability the arm with mean $\mu_2$ and the arm with mean
$\mu_3$. On round $R_{2}(T)^{(1-\alpha)}$, the algorithm switches
to playing $\mu_1$ until the rest of the game. Notice that the rewards
up to time $R_{2}(T)^{(1-\alpha)}$, whether $\beta = 1$ or $\beta
= 0$, have the same distribution. Hence, $\bP(\cdot
|r_{1:R_2(T)^{(1-\alpha)}}, \beta = 1)
= \bP(\cdot
|r_{1:R_2(T)^{(1-\alpha)}}, \beta = 0)$. Then,
following the arguments in the proof of
Theorem~\ref{thm:lower_bound}, we can prove the following lower bound.

\begin{theorem}
\label{thm:lower_bound3}
Let algorithms $\cA_1$ and $\cA_2$ follow the construction in Section~\ref{sec:lower_bound_known_regret}. Suppose that the worst case known regret bound for Algorithm is $R_{2}(T)$. With probability at least $1/2$ over the random choice of $\alpha$ any corralling strategy will incur regret at least $\tilde\Omega(R_{2}(T))$ while the regret of $\cA_2$ is at most $O(\sqrt{R_{2}(T)})$.
\end{theorem}

\section{Proofs from Section~\ref{sec:ucb_boost}}

\begin{lemma}
Suppose we run $2\log{1/\delta}$ copies of algorithm $\cA_i$ which satisfies Equation~\ref{eq:pseudo_regret_assumption}. Let $\cA_{med_i}$ denote the algorithm with median reward at time $t$. Then,
\begin{align*}
    \prob{t\mu_{med_i,1} -\sum_{s=1}^t r_{s}(a_{med_i,j_s}) \geq 2\bar R_{med_i}(t)} \leq \delta.
\end{align*}
\end{lemma}
\begin{proof}[Proof of Lemma~\ref{lem:boosted_regret_bound}]
First note that $\mu_{med_i,1} = \mu_{i_s,1}$ and $\bar R_{i_s}(t) = \bar R_{med_i}(t) $for all $s$ and $t$. The assumption in Equation~\ref{eq:pseudo_regret_assumption} together with Markov's inequality implies that for every copy $\cA_{i_s}$ of $\cA_i$ at time $t$ it holds that 
\begin{align*}
    \prob{t\mu_{med_i, 1} -\sum_{s=1}^t r_{s}(a_{i,j_s}) \geq 2\bar R_{med_i}(t)} \leq \frac{1}{2}.
\end{align*}
Let $\cA_{i_1},\ldots,\cA_{i_n}$ be the algorithms which have reward smaller than $\cA_{med_i}$ at time $t$. We have
\begin{align*}
     \prob{t\mu_{med_i, 1} -\sum_{s=1}^t r_{s}(a_{med_i,j_s}) \geq 2\bar R_{med_i}(t)} &\leq \prob{\bigcap_{l\in[n]}\left\{ t\mu_{l,1} -\sum_{s=1}^t r_s(a_{i_l,j_s}) \geq 2\bar R_{med_i}(t) \right\}}\\
     &\leq \left(\frac{1}{2}\right)^{\log{1/\delta}} \leq \delta,
\end{align*}
where the first inequality follows from the definition of $\cA_{med_i}$ and $\cA_{i_l}$ for $l\in[n]$.
\end{proof}

\begin{theorem}
Suppose that algorithms $\cA_1,\ldots,\cA_k$ satisfy the following regret bound $\mathbb{E}[R_{i}(t)] \leq \sqrt{\alpha k_i t \log{t}}$. Then after $T$ rounds, Algorithm~\ref{alg:ucb-c-gen} produces a sequence of actions $a_1,\ldots,a_T$, such that
\begin{align*}
    T\mu_{1,1} - \mathbb{E}\left[\sum_{t=1}^T r_t(a_{i_t,j_t})\right] &\leq O\left(\sum_{i \neq i^*} \frac{k_i\log{T}^2}{\Delta_i} + \log{T}\mathbb{E}\left[R_{i^*}(T)\right]\right),\\
    T\mu_{1,1} - \mathbb{E}\left[\sum_{t=1}^T r_t(a_{i_t,j_t})\right] &\leq O\left(\log{T}\sqrt{K T\log{T} \max_{i\in [K]} (k_i)}\right).
\end{align*}
\end{theorem}
\begin{proof}[Proof of Theorem~\ref{thm:regret_bound_ucbc}]
For simplicity we assume that $\lceil\log{T}\rceil =\log{T}$. For the rest of the proof we let $t_\ell = T_{\ell}(t)$ to simplify notation. Further, since $\bar R_{\ell_s} = \bar R_{\ell}, \forall s \in [\log{T}]$, we use $\bar R_{\ell}$ as the upper bound on the regret for all algorithms in $\mathbb{A}_\ell$. Let $\psi_\ell(t) = 2\sqrt{\frac{2\log{t}}{t_\ell}} + \frac{\sqrt{2\bar R_{\ell}(t_\ell)}}{t_\ell}$. The proof follows the standard ideas behind analyses of UCB type algorithms. If at time $t$ algorithm $\ell \neq 1$ is selected then one of the following must hold true:
\begin{equation}
    \label{eq:ucb_cond_1}
    \mu_{1,1} \geq \hat\mu_{\bar 1}(t_1) + \frac{\sqrt{2\bar R_{1}(t_1)} + \sqrt{2t_1\log{t}}}{t_1},
\end{equation}
\begin{equation}
    \label{eq:ucb_cond_2}
    \hat\mu_{med_\ell}(t_\ell) > \mu_{1,1} + \sqrt{\frac{2\log{t}}{t_\ell}},
\end{equation}
\begin{equation}
\label{eq:ucb_cond_3}
\Delta_\ell < 2\sqrt{\frac{2\log{t}}{t_\ell}} + \frac{\sqrt{2\bar R_{\ell}(t_\ell)}}{t_\ell}.
\end{equation}
The above conditions can be derived by considering the case when the UCB for $\cA_1$ is smaller than the UCB for $\cA_\ell$ and every algorithm has been selected a sufficient number of times. Suppose that the three conditions above are false at the same time. Then we have
\begin{align*}
    \hat\mu_{\bar 1}(t_1) &+ \frac{\sqrt{2\bar R_{1}(t_1)} + \sqrt{2t_1\log{t}}}{t_1} > \mu_{1,1} = \mu_{1,\ell} + \Delta_\ell\\
    &\geq \Delta_\ell + \hat\mu_{med_\ell}(t_\ell) - \sqrt{\frac{2\log{t}}{t_\ell}}\\
    &\geq \hat\mu_{med_\ell}(t_\ell) + \frac{\sqrt{2\bar R_{\ell}(t_\ell)} + \sqrt{2t_\ell\log{t}}}{t_\ell},
\end{align*}
which contradicts the assumption that algorithm $\cA_\ell$ was selected. With slight abuse of notation we use $[k_{\ell}]$ to denote the set of arms belonging to algorithm $\cA_\ell$. Next we bound the expected number of times each sub-optimal algorithm is played up to time $T$. Let $\delta$ be an upper bound on the probability of the event that $\hat \mu_{\bar 1}(s)$ exceeds the UCB for $\cA_1$.
\begin{align*}
    \mathbb{E}[T_\ell] &= \sum_{t=1}^T \mathbb{E}[\chi_{a_{i_t,j_t} \in [k_{\ell}]}] \leq \psi_\ell^{-1}(\Delta_\ell) + \sum_{t > \psi_\ell^{-1}(\Delta_\ell)} \prob{\text{Equation~\ref{eq:ucb_cond_1} or Equation~\ref{eq:ucb_cond_2} hold}}\\
    &\leq \psi_\ell^{-1}(\Delta_\ell) + \sum_{t > \psi_\ell^{-1}(\Delta_\ell)} \prob{\exists s \in [t] : \mu_{1,1} \geq \hat\mu_{\bar 1}(s) + \frac{\sqrt{2\bar R_{1}(s)} + \sqrt{2s\log{t}}}{s}}\\
    &+ \sum_{t > \psi_\ell^{-1}(\Delta_\ell)} \prob{\exists s \in [t] : \hat\mu_{med_\ell}(s) > \mu_{1,1} + \sqrt{\frac{2\log{t}}{s}}}\\
    &\leq \sum_{t > \psi_\ell^{-1}(\Delta_\ell)} t\delta + \sum_{t > \psi_\ell^{-1}(\Delta_\ell)} \frac{1}{t} + \psi^{-1}(\Delta_\ell),
\end{align*}
where the last inequality follows from the definition of $\delta$ and the fact that $\hat \mu_{med_\ell}(s) \leq \hat\mu_{med_\ell,1}(s)$ (empirical mean of arm $1$ for algorithm $\cA_{med_\ell}$ at time $s$) and the standard argument in the analysis of UCB-I. Setting $\delta = \frac{1}{T^2}$ finishes the bound on the number of suboptimal algorithm pulls. Next we consider bounding the regret incurred only by playing the median algorithms $\cA_{med_\ell}$
\begin{align*}
    t\mu_{1,1} - \mathbb{E}\left[\sum_{s=1}^t r_s(a_{i_s,j_s})\right] &= t\mu_{1,1} - \mathbb{E}\left[\sum_\ell t_\ell\mu_{\ell,1} + \sum_{\ell}\sum_i T_{med_\ell,i}(t_\ell)\mu_{\ell,i} - \sum_\ell t_\ell\mu_{\ell,1}\right]\\
    &= \mathbb{E}\left[\sum_{\ell \neq 1}t_\ell\Delta_\ell\right] + \sum_{\ell \neq 1}\mathbb{E}[R_{med_\ell}(t_\ell)] + \mathbb{E}\left[R_{med_1}(t)\right]\\
    &\leq \sum_{\ell\neq 1}\Delta_\ell\psi^{-1}(\Delta_\ell) + 2\log{T} + \sum_{\ell \neq 1} \mathbb{E}[\sqrt{\alpha k_\ell t_\ell \log{t}}] + \mathbb{E}\left[R_{med_1}(t)\right]\\
    &\leq \sum_{\ell\neq 1}\Delta_\ell\psi^{-1}(\Delta_\ell) + 2\log{T} + \sum_{\ell \neq 1}\sqrt{\alpha k_\ell \mathbb{E}[t_\ell]\log{t}} + \mathbb{E}\left[R_{med_1}(t)\right]\\
    &\leq \sum_{\ell\neq 1}\Delta_\ell\psi^{-1}(\Delta_\ell) + 2\log{T} + \sum_{\ell \neq 1}\sqrt{\alpha k_\ell \psi^{-1}(\Delta_\ell)\log{t}}\\ &+\mathbb{E}\left[R_{med_1}(t)\right]
    + \sqrt{\alpha k_\ell}\log{t}.
\end{align*}
Now for the assumed regret bound on the algorithms, we have $\psi_\ell(t) = 2\sqrt{\frac{2\log{t}}{t_\ell}} + \sqrt{2\frac{\alpha k_\ell \log{t_\ell}}{t_\ell}}$. This implies that $\psi_{\ell}^{-1}(\Delta_\ell) \leq \frac{\alpha'k_\ell\log{t}}{\Delta_\ell^2}$, for some other constant $\alpha'$. 
To get the instance independent bound we first notice that by Jensen's inequality we have
\begin{align*}
    \sum_{\ell}\sqrt{\alpha' k_\ell \mathbb{E}[t_\ell]\log{t}} &\leq K\sqrt{\frac{1}{K} \sum_{\ell}\alpha'\mathbb{E}[t_\ell] k_\ell \log{t}}\\
    &\leq \sqrt{\alpha'K t\log{t} \max_{\ell} (k_\ell)}.
\end{align*}
Next we can bound $\mathbb{E}\left[\sum_{\ell \neq 1}t_\ell\Delta_\ell\right]$ in the following way
\begin{align*}
    \mathbb{E}\left[\sum_{\ell \neq 1}t_\ell\Delta_\ell\right] &\leq \sum_{\ell}\Delta_\ell \sqrt{\mathbb{E}[t_\ell]}\sqrt{\mathbb{E}[t_\ell]} = \sum_{\ell}\sqrt{\Delta_\ell^2\mathbb{E}[t_\ell]}\sqrt{\mathbb{E}[t_\ell]}\\
    &= \sum_{\ell}\sqrt{\alpha'k_\ell\mathbb{E}[t_\ell] \log{t}} \leq \sqrt{\alpha'K t\log{t}\max_{\ell} (k_\ell)}
\end{align*}
The theorem now follows.
\end{proof}

\subsection{Proof of Theorem~\ref{thm:ucb_lower_bound_boosting}}
\label{sec:ucb_lower_bound}
Consider an instance of Algorithm~\ref{alg:ucb-c-gen}, except that it runs a single copy of each base learner $\cA_i$. 
Let $\cA_1$ be a UCB algorithm with two arms with means $\mu_1 > \mu_2$, respectively. The arm with mean $\mu_1$ is set according to a Bernoulli random variable, and the arm with mean $\mu_2$ is deterministic. Let algorithm $\cA_2$ have a single deterministic arm with mean $\mu_3$, such that $\mu_1 > \mu_3$ and $\mu_3 > \mu_2$. Let $\Delta = \mu_1 - \mu_3$. We now follow the lower bounding technique of~\cite{audibert2009exploration}. 

Consider the event that in the first $q$ pulls of arm $a_1^{\cA_1}$, we have $r_t(a_{1,1}) = 0$, i.e. $\mathscr{E} = \{r_1(a_{1,1}) = 0, r_2(a_{1,1}) = 0,\ldots, r_q(a_{1,1}) = 0\}$. This event occurs with probability $(1-\mu_1)^q$. 
Notice that on event $\mathscr{E}$, the upper confidence bound for $\mu_1$ as per $\cA_1$ is $\sqrt{\frac{\alpha\log{T_{1}(t)}}{q}}$ during time $t$. This implies that for $a_{1,1}$ to be pulled again we need $\sqrt{\frac{\alpha\log{T_{1}(t)}}{q}} > \mu_2$ and hence for the first $\exp{q\mu_2^2/\alpha}$ rounds in which $\cA_1$ is selected by the corralling algorithm, $a_{1,1}$ is only pulled $q$ times. 
Further, on $\mathscr{E}$, the upper confidence bound for $\cA_1$ as per the corralling algorithm is of the form $\sqrt{\frac{2\beta\log{t}}{T_{1}(t)}}$. This implies that for $\cA_1$ to be selected again we need $\mu_2 + \sqrt{\frac{2\beta\log{t}}{T_{1}(t)}} > \mu_3$. Let $\tilde\Delta = \mu_3 - \mu_2$. Then, the above implies that in the first $t \leq \exp{T_{1}(t)\tilde\Delta^2/(2\beta)}$ rounds, $\cA_1$ is pulled at most $T_{1}(t)$ times.
Combining with the bound for the number of pulls of $a_{1,1}$ we arrive at the fact that on $\mathscr{E}$, $a_{1,1}$ can not be pulled more than $q$ times in the first $\exp{\frac{\tilde\Delta^2 \exp{q\mu_2^2/\alpha}}{2\beta}}$ rounds. Let $q$ be large enough so that $q \leq \frac{1}{2}\exp{\frac{\tilde\Delta^2 \exp{q\mu_2^2/\alpha}}{2\beta}}$. Then, for large enough $T$, we have that the pseudo-regret of the corralling algorithm is $\hat R(T) \geq \frac{1}{2} \Delta \exp{\frac{\tilde\Delta^2 \exp{q\mu_2^2/\alpha}}{2\beta}}$. 
Taking $q = \log{\frac{2\beta}{\tilde\Delta^2}\log{\tau\frac{\alpha}{\mu_2^2}}}$, we get 
\vspace*{-5pt}
\begin{align*}
    \prob{\hat R(T) \geq \frac{1}{2}\Delta \tau} \geq \prob{\mathscr{E}} = (1-\mu_1)^q = \frac{1}{\exp{q}^{\log{1/(1-\mu_1)}}} = \Big(\frac{\tilde\Delta^2}{2\beta\log{\tau}}\Big)^{\frac{\alpha}{\mu_2^2}\log{1/(1-\mu_1)}}.
\end{align*}
Let $\gamma = \frac{\alpha}{\mu_2^2}\log{1/(1-\mu_1)}$. We can now bound the expected pseudo-regret of the algorithm by integrating over $2 \leq \tau \leq T$, to get 
\begin{align*}
    \mathbb{E}[\hat R(T)] &\geq \frac{1}{2}\Delta \int_2^T \Big(\frac{\tilde\Delta^2}{2\beta\log{\tau}}\Big)^{\frac{\alpha}{\mu_2^2}\log{1/(1-\mu_1)}} d \tau = \frac{1}{2}\Delta\Big(\frac{\tilde\Delta^2}{2\beta}\Big)^\gamma \int_2^T \left(\frac{1}{\log{\tau}}\right)^\gamma d \tau\\
    &\geq \frac{1}{2}\Delta\Big(\frac{\tilde\Delta^2}{2\beta}\Big)^\gamma \frac{T-2}{\left(\log{\frac{T+2}{2}}\right)^\gamma},
\end{align*}
where the last inequality follows from the Hermite-Hadamart inequality.

It is important to note that the above reasoning will fail if $\gamma$ is a function of $T$. This might occur if in the UCB for $\cA_1$ we have $\alpha = \log{T}$. In such a case the lower bounds become meaningless as $\frac{1}{\log{(T+2)/2}}^\gamma \leq o(1/T)$. Further, it should actually be possible to avoid boosting in this case as the tail bound of the regret will now be upper bounded as $\mathbb{P}[R_1(t) \geq \Delta\tau] \leq \frac{1}{T\tau^c}$.

\paragraph{General Approach if Regret has a Polynomial Tail.} Assume that, in general, the best algorithm has the following regret tail:
\vspace*{-15pt}
\begin{align*}
    \prob{R_{1}(t) \geq \frac{1}{2}\Delta_{1,1}\tau} \geq \frac{1}{\tau^c},
\end{align*}
for some constant $c$. Results in \cite{salomon2011deviations} suggest that for stochastic bandit algorithms which enjoy anytime regret bounds we can not have a much tighter high probability regret bound. Let $\mathscr{E}_{T_{1}(t)} = \{R_{1}(T_{1}(t)) \geq T_{1}(t)(\mu_1 - \frac{1}{\sqrt{2}}\mu_3)\}$. 
After $T_{1}(t)$ pulls of $\cA_1$ the reward plus the UCB for $\cA_1$ is at most $\frac{\sum_{s=1}^{T_{1}(t)}r_s(a_{1,j_s})}{T_{1}(t)} + \sqrt{\frac{\alpha k_1 \log{t} }{T_{1}(t)}}$, and on  $\mathscr{E}_{T_{1}(t)}$, we have $\frac{\sum_{s=1}^{T_{1}(t)}r_s(a_{1,j_s})}{T_{1}(t)} \leq \frac{1}{\sqrt{2}} \mu_3$. This implies that in the first $t$ rounds, $\cA_1$ could not have been pulled more than 
\begin{align*}
    \frac{\alpha k_1 \log{t}}{\left(\mu_3 - \frac{\sum_{s=1}^{T_{1}(t)}r_s(a_{1,j_s})}{T_{1}(t)}\right)^2} \leq \frac{2\alpha k_1 \log{t}}{\mu_3^2}.
\end{align*}
Setting $T_{1}(t) = \frac{2\alpha k_1\log{T}}{\mu_3^2}$, we have that $\mathscr{E}_{T_{1}(t)}$ occurs with probability at least $\left(\frac{\Delta_{1,2}\mu_3^2}{4\alpha k_1\log{T}}\right)^c$ and hence the expected regret of the corralling algorithm is at least
\begin{align*}
    \left(\frac{\Delta_{1,2}\mu_3^2}{4\alpha k_1\log{T}}\right)^c\Delta\left(T - \frac{2\alpha k_1\log{T}}{\mu_3^2}\right).
\end{align*}
We have just showed the following.
\begin{theorem}
There exist instances $\cA_1$ and $\cA_2$ of UCB-I and a reward distribution, such that if Algorithm~\ref{alg:ucb-c-gen} runs a single copy of $\cA_1$ and $\cA_2$ the expected regret of the algorithm is at least
\begin{align*}
    \mathbb{E}[R(T)] \geq \tilde \Omega(\Delta T).
\end{align*}
Further, for any algorithm $\cA_1$ such that $\prob{R_{1}(t) \geq \frac{1}{2}\Delta_{1,1}\tau} \geq \frac{1}{\tau^c}$, there exists a reward distribution such that if Algorithm~\ref{alg:ucb-c-gen} runs a single copy of $\cA_1$ the expected regret of the algorithm is at least
\begin{align*}
    \mathbb{E}[R(T)] \geq \tilde \Omega((\Delta_{1,2})^c\Delta T).
\end{align*}
\end{theorem}

\section{Proof of Theorem~\ref{thm:regret_bound}}
\label{app:tsallis-inf}

\subsection{Potential function and auxiliary lemmas}
First we recall the definition of conjugate of a convex function $f$, denoted as $f^*$ 
\begin{align*}
    f^*(y) = \max_{x \in \mathbb{R}^d} \langle x,y \rangle - f(x).
\end{align*}
In our algorithm, we are going to use the following potential at time $t$
\begin{equation}
\label{eq:potential_identities}
\begin{aligned}
    \Psi_t(w) &= -4\sum_{i=1}^K \frac{\sqrt{w_i} - \frac{1}{2}w_i}{\eta_{t,i}}\\
    \nabla \Psi_t(w)_i &= -2 \frac{\frac{1}{\sqrt{w_i}} - 1}{\eta_{t,i}}\\
    \nabla^2\Psi_t(w)_{i,i} &= \frac{1}{w_i^{3/2}\eta_{t,i}},\nabla^2\Psi_t(w)_{i,j} = 0\\
    \nabla \Psi_t^*(Y)_i &= \frac{1}{\left(-\frac{\eta_{t,i}}{2}Y_i+1\right)^2}\\
    \Phi_t (Y) &= \max_{w \in \Delta^{K-1}} \langle Y,w \rangle - \Psi_t(w) = \left(\Psi_t + \I_{\Delta^{K-1}}\right)^*(Y).
\end{aligned}
\end{equation}
Further for a function $f$ we use $D_f(x,y)$ to denote the Bregman divergence between $x$ and $y$ induced by $f$ equal to
\begin{align*}
    D_f(x,y) = f(x) - f(y) - \langle \nabla f(y),x-y \rangle = f(x) + f^*(\nabla f(y)) - \langle \nabla f(y), x\rangle,
\end{align*}
where the second inequality follows by the Fenchel duality equality $f^*(\nabla f(y)) + f(y) = \langle \nabla f(y),y \rangle$.
We now present a bandit algorithm is going to be the basis for the corralling algorithm. Let $\eta_t = \begin{pmatrix}\eta_{t,1}\\ \eta_{t,2}\\ \vdots \\ \eta_{t,K}\end{pmatrix}$ be the step size schedule for time $t$. The algorithm proceeds in epochs. Each epoch is twice as large as the preceding and the step size schedule remains non-increasing throughout the epochs, except when an OMD step is taken. In each epoch the algorithm makes a choice to either take two mirror descent steps, while increasing the step size:
\begin{equation}
\label{eq:omd_step}
    \begin{aligned}
       \hat w_{t+1} &= \argmin_{w\in \Delta^{K-1}} \langle \hat \ell_t, w \rangle + D_{\Psi_t}(w,w_t),\\
       \eta_{t+1,i} &= \beta\eta_{t,i} & \text{for $i : w_{t,i} \leq 1/\rho_{s_i}$},\\
       \hat w_{t+2} &= \argmin_{w\in \Delta^{K-1}} \langle \hat \ell_{t+1}, w \rangle + D_{\Psi_{t+1}}(w,\hat w_{t+1}),\\
       \rho_{s_i} &= 2\rho_{s_i},
    \end{aligned}
\end{equation}
or the algorithm takes a FTRL step
\begin{equation}
\label{eq:ftrl_step}
    w_{t+1} = \argmin_{w\in \Delta^{K-1}} \langle \hat L_t, w \rangle + \Psi_{t+1}(w),
\end{equation}
where $\hat L_t = \hat L_{t-1} + \hat \ell_t$ unless otherwise specified by the algorithm. We note that the algorithm can only increase the step size during the OMD step. For technical reasons we require a FTRL step after each OMD step. Further we require that the second step of each epoch be an OMD step, if there exists at least one $w_{t,i} \leq \frac{1}{\rho_1}$. The algorithm also can enter an OMD step during an epoch if at least one $w_{t,i} \leq \frac{1}{\rho_{n_i}}$. The intuition behind this behavior is as follows. Increasing the step size and doing an OMD step will give us negative regret during that round and we only require negative regret for a certain arm if the probability of pulling said arm becomes smaller than some threshold. The pseudo-code can be found in Algorithm~\ref{alg:tsallis_inf_incr_dcr}. For the rest of the proofs and discussion we denote an iterate from the FTRL update as $w_t$ and an iterate from the OMD update as $\hat w_t$. Further, intermediate iterates of OMD are denotes as $\tilde w_t$. We now present a couple of auxiliary lemmas useful for analyzing the OMD and FTRL updates.

\begin{lemma}
\label{lem:duality_eq}
For any $x,y \in \Delta^{K-1}$ it holds
\begin{align*}
    D_{\Psi_t}(x,y) = D_{\Phi_t}(\nabla \Phi_t^*(y), \nabla\Phi_t^*(x)).
\end{align*}

\end{lemma}
\begin{proof}
Since $\Psi_t+\I_{\Delta^{K-1}}$ is a convex, closed function on $\Delta^{K-1}$ it holds that $\Psi_t+\I_{\Delta^{K-1}} = ((\Psi_t+\I_{\Delta^{K-1}})^*)^*$ (see for e.g. \citep{brezis2010functional} Theorem 1.11). Further, $\Phi_t^*(x) = ((\Psi_t+\I_{\Delta^{K-1}})^*)^*(x) = \Psi_t(x)$. The above implies
\begin{align*}
    D_{\Psi_t}(x,y) = D_{\Phi_t^*}(x,y) = D_{\Phi_t}(\nabla \Phi_t^*(y), \nabla\Phi_t^*(x)).
\end{align*}
\end{proof}

\begin{lemma}
\label{lem:ftrl_lagrange_mult}
For any positive $\hat L_t$ and $w_{t+1}$ generated according to update~\ref{eq:ftrl_step} we have
\begin{align*}
    w_{t+1} = \nabla\Phi_{t+1}(-\hat L_{t}) = \nabla\Psi_{t+1}^*(-\hat L_{t}+\nu_{t+1}\pmb{1}),
\end{align*}
for some scalar $\nu_t$. Further $(\hat L_t - \nu_{t+1}\pmb{1})_i > 0$ for all $i \in [K]$.
\end{lemma}
\begin{proof}
The proof is contained in Section 4.3 in~\cite{zimmert2018optimal}.
\end{proof}

\begin{lemma}[Lemma 16~\cite{zimmert2018optimal}]
\label{lem:lemma_16_zimmer}
Let $w \in \Delta^{K-1}$ and $\tilde w = \nabla \Psi_t^*(\nabla \Psi_t(w) - \ell)$. If $\eta_{t,i}\leq \frac{1}{4}$, then for all $\ell > -1$ it holds that $\tilde w_i^{3/2} \leq 2w_i^{3/2}$.
\end{lemma}

\begin{algorithm}[t]
\caption{Corralling with Tsallis-INF}
\begin{algorithmic}[1]
\REQUIRE{Mult. constant $\beta$, thresholds $\{\rho_i\}_{i=1}^n$, initial step size $\eta$, epochs $\{\tau_i\}_{i=1}^m$, algorithms $\{\cA_i\}_{i=1}^K$.}
\ENSURE{Algorithm selection sequence $(i_t)_{t=1}^T$.}
\STATE Initialize $t=1$, $w_1 = Unif(\Delta^{K-1})$, $\eta_1 = \eta$
\STATE Initialize current threshold list $\theta \in [n]^K$ to $\bf 1$
\WHILE{$t\leq \log{T}+1$}
\FOR{$i\in [K]$}
\STATE Algorithm $i$ plays action $a_{i,j_t}$ and $\hat L_{1,i} += \ell_{t}(a_{i,j_t})$
\ENDFOR
\STATE $t+=1$
\ENDWHILE
\STATE $t=2, w_2 = \nabla \Phi_2(-\hat L_1),1/\eta_{t+1}^2 = 1/\eta_{t}^2 +1$
\WHILE{$j \leq m$}
\FOR{$t\in \tau_j$}
\STATE $\mathcal{R}_t = \emptyset$,$\hat\ell_t = \texttt{PLAY-ROUND}(w_t)$
\IF{$t$ is the first round of epoch $\tau_j$ and $\exists w_{t,i} \leq \frac{1}{\rho_1}$}
\FOR{$i \colon w_{t,i} \leq \frac{1}{\rho_1}$}
\STATE $\theta_i = \min\{s \in [n]\colon w_{t,i} > \frac{1}{\rho_s}\}$, $\mathcal{R}_t = \mathcal{R}_t\bigcup \{i\}$.
\ENDFOR
\STATE $(w_{t+3},\hat L_{t+2}) = \texttt{NEG-REG-STEP}(w_{t}, \hat\ell_{t}, \eta_{t}, \mathcal{R}_t, \hat L_{t-1})$, $t = t+2$, $\hat\ell_t = \texttt{PLAY-ROUND}(w_t)$
\ENDIF
\IF{$\exists i\colon w_{t,i} \leq \frac{1}{\rho_{\theta_i}}$ and prior step was not \texttt{NEG-REG-STEP}}
\FOR{$i \colon w_{t,i} \leq \frac{1}{\rho_{\theta_i}}$}
\STATE $\theta_i +=1$, $\mathcal{R}_t = \mathcal{R}_t\bigcup \{i\}$.
\ENDFOR
\STATE $(w_{t+3},\hat L_{t+2}) = \texttt{NEG-REG-STEP}(w_{t}, \hat\ell_{t}, \eta_{t}, \mathcal{R}_t, \hat L_{t-1})$, $t = t+2$, $\hat\ell_t = \texttt{PLAY-ROUND}(w_t)$
\ELSE
\STATE $1/\eta_{t+1}^2 = 1/\eta_{t}^2 +1$, $w_{t+1} = \nabla \Phi_{t+1}(-\hat L_t)$
\ENDIF
\ENDFOR
\ENDWHILE
\end{algorithmic}
\end{algorithm}

\begin{algorithm}[t]
\caption{$\texttt{NEG-REG-STEP}$}
\begin{algorithmic}[1]
\REQUIRE{Previous iterate $w_t$, current loss $\hat\ell_t$, step size $\eta_t$, set of rescaled step-sizes $\mathcal{R}_t$, cumulative loss $\hat L_{t-1}$}
\ENSURE{Plays two rounds of the game and returns distribution $w_{t+3}$ and cumulative loss $\hat L_{t+2}$}
\STATE $(w_{t+1},\hat L_{t}) = \texttt{OMD-STEP}(w_{t}, \hat\ell_{t}, \eta_{t}, \mathcal{R}_t, \hat L_{t-1})$
\STATE $\hat\ell_{t+1} = \texttt{PLAY-ROUND}(w_{t+1})$, $\hat L_{t+1} = \hat L_t + \hat\ell_{t+1}$
\FOR{$i \in \mathcal{R}_t$}
\STATE $\eta_{t+2,i} = \beta\eta_{t,i}$ and restart $\cA_i$ with updated environment $\theta_i = \frac{2}{w_{t,i}}$
\ENDFOR
\STATE $w_{t+2} = \nabla \Phi_{t+2}(-\hat L_{t+1})$
\STATE $\hat\ell_{t+2} = \texttt{PLAY-ROUND}(w_{t+2})$
\STATE $\hat L_{t+2} = \hat L_{t+1} + \hat\ell_{t+2}, \eta_{t+3} = \eta_{t+2}, t = t+2$
\STATE $w_{t+1} = \nabla \Phi_{t+1}(-\hat L_t),t=t+1$
\end{algorithmic}
\end{algorithm}

\begin{algorithm}[t]
\caption{$\texttt{OMD-STEP}$}
\label{alg:omd_step}
\begin{algorithmic}[1]
\REQUIRE{Previous iterate $w_t$, current loss $\hat\ell_t$, step size $\eta_t$, set of rescaled step-sizes $\mathcal{R}_t$, cumulative loss $\hat L_{t-1}$}
\ENSURE{New iterate $w_{t+1}$, cumulative loss $\hat L_t$}
\STATE $\nabla\Psi_t (\tilde w_{t+1}) = \nabla\Psi_t(w_t) - \hat\ell_t$
\STATE $w_{t+1} = \argmin_{w \in \Delta^{K-1}}D_{\Phi_t}(w,\tilde w_{t+1})$.
\STATE $\e = \sum_{i\in \mathcal{R}_t} \e_i$
\STATE $\tilde L_{t-1} =  (\pmb 1_k - \e)\odot(\hat L_{t-1} - (\nu_{t-2} + \nu_{t-1})\pmb{1}_k) + \frac{1}{\beta}\e\odot((\hat L_{t-1} - (\nu_{t-2}+\nu_{t-1})\pmb{1}_k))$ // $\nu_{t-2}$ and $\nu_{t-1}$ are the Lagrange multipliers from the previous two FTRL steps.
\STATE $\hat L_{t} = \tilde L_{t-1} + \hat\ell_t$ 
\end{algorithmic}
\end{algorithm}

\begin{algorithm}[t]
\caption{$\texttt{PLAY-ROUND}$}
\label{alg:play_round}
\begin{algorithmic}[1]
\REQUIRE{Sampling distribution $w_t$}
\ENSURE{Loss vector $\hat\ell_t$}
\STATE Sample algorithm $i_t$ according to $\bar w_t = \left(1-\frac{1}{Tk}\right)w_t + \frac{1}{Tk}Unif(\Delta^{K-1})$.
\STATE Algorithm $i_t$ plays action $a_{i_t,j_t}$. Observe loss $\ell_t(a_{i_t,j_t})$ and construct unbiased estimator $\hat\ell_t= \frac{\ell_t(a_{i_t,j_t})}{\bar w_{t,i_t}}\e_{i_t}$ of $\ell_t$.
\STATE Give feedback to $i$-th algorithm as $\hat \ell_t(a_{i,j_t})$, where $a_{i,j_t}$ was action provided by $\cA_i$
\end{algorithmic}
\end{algorithm}

\subsection{Regret bound}
We begin by studying the instantaneous regret of the FTRL update. The bound follows the one in~\cite{zimmert2018optimal}. Let $u = e_{i^*}$ be the unit vector corresponding to the optimal algorithm $\cA_{i^*}$. First we decompose the regret into a stability term and a penalty term:
\begin{align*}
    \langle \hat \ell_t, w_t - u \rangle &= \langle \hat \ell_t, w_t \rangle + \Phi_t(-\hat L_t) - \Phi_t(-\hat L_{t-1}) \textit{ (Stability)}\\
    &-\Phi_t(-\hat L_t) + \Phi_t(-\hat L_{t-1}) - \langle \hat\ell_t, u\rangle \textit{ (Penalty)}.
\end{align*}
The bound on the stability term follows from Lemma 11 in~\cite{zimmert2018optimal}, however, we will show this carefully, since parts of the proof will be needed to bound other terms. Recall the definition of $\Phi_t(Y) = \max_{w \in \Delta^{K-1}} \langle Y,w \rangle - \Psi_t(w)$. Since $w$ is in the simplex we have $\Phi_t(Y+\alpha\pmb{1}_k) = \max_{w \in \Delta^{K-1}} \langle Y,w \rangle + \langle \alpha\pmb{1},w \rangle- \Psi_t(w) = \Phi_t(Y) + \alpha$. We also note that from Lemma~\ref{lem:ftrl_lagrange_mult} it follows that we can write $\nabla\Psi_t(w_t) = -\hat L_{t-1} + \nu_t\pmb{1}$. Combining the two facts we have
\begin{align*}
     \langle \ell_t, w_t \rangle + \Phi_t(-\hat L_t) - \Phi_t(-\hat L_{t-1}) &= \langle \ell_t, w_t \rangle + \Phi_t(\nabla\Psi_t(w_{t}) -\hat\ell_t - \nu_t\pmb{1}) - \Phi_t(\nabla\Psi_t(w_{t}) - \nu_t\pmb{1})\\
     &= \langle \ell_t - \alpha\pmb{1}_k, w_t  \rangle + \Phi_t(\nabla\Psi_t(w_{t}) -\hat\ell_t + \alpha\pmb{1}_k) - \Phi_t(\nabla\Psi_t(w_{t}))\\
     &\leq \langle \ell_t- \alpha\pmb{1}_k, w_t   \rangle + \Psi_t^*(\nabla\Psi_t(w_{t}) -\hat\ell_t + \alpha\pmb{1}_k) - \Psi_t^*(\nabla\Psi_t(w_{t}))\\
     &= D_{\Psi_t^*}(\nabla\Psi_t(w_{t}) -\hat\ell_t + \alpha\pmb{1}_k, \nabla\Psi_t(w_{t}))\\
     &\leq \max_{z \in [\nabla\Psi_t(w_{t}) -\hat\ell_t + \alpha\pmb{1}_k, \nabla\Psi_t(w_{t})]} \frac{1}{2}\|\hat\ell_t - \alpha\pmb{1} \|_{\nabla^2\Psi_t*(z)}^2\\
     &=\max_{w \in [w_{t},\nabla\Psi_t^*(\nabla\Psi_t(w_{t}) -\hat\ell_t + \alpha\pmb{1}_k)]} \frac{1}{2}\|\hat\ell_t - \alpha\pmb{1} \|_{\nabla^2\Psi_t^{-1}(w)}^2,
\end{align*}
where the first inequality holds since $\Psi_t^* \geq \Phi_t$ and $\Psi_t^*(\nabla \Psi(w_t)) = \langle\nabla\Psi(w_t), w_t \rangle - \Psi_t(w_t) = \Phi_t(\nabla\Psi(w_t))$ and the second inequality follows since by Taylor's theorem there exists a $z$ on the line segment between $\nabla \Psi_t(w_t) - \hat\ell_t + \alpha\pmb{1}_k$ and $\Psi_t(w_t)$ such that $D_{\Psi_t^*}(\nabla\Psi_t(w_{t}) -\hat\ell_t + \alpha\pmb{1}_k, \nabla\Psi_t(w_{t})) = \frac{1}{2}\|\hat\ell_t - \alpha\pmb{1} \|_{\nabla^2\Psi_t*(z)}^2$.
\begin{lemma}
\label{lem:lemma_11_zimmert}
Let $w_{t} \in \Delta^{K-1}$ and let $i_t \sim w_{t}$. Let $\hat\ell_{t,i_t} = \frac{\ell_{t,i_t}}{w_{t,i_t}}$ and $\hat\ell_{t,i} = 0$ for all $i\neq i_t$. It holds that 
\begin{align*}
    \mathbb{E}\left[\max_{w \in [w_{t},\nabla\Psi_t^*(\nabla\Psi_t(w_{t}) -\hat\ell_t + \alpha\pmb{1}_k)]} \|\hat\ell_t\|_{\nabla^2\Psi_t^{-1}(w)}^2\right] &\leq \sum_{i=1}^K \frac{\eta_{t,i}}{2}\sqrt{\mathbb{E}[w_{t,i}]}\\
    \mathbb{E}\left[\max_{w \in [w_{t},\nabla\Psi_t^*(\nabla\Psi_t(w_{t}) -\hat\ell_t + \alpha\pmb{1}_k)]} \|\hat\ell_t - \chi_{(i_t = j)}\ell_{t,j}\pmb{1}\|_{\nabla^2\Psi_t^{-1}(w)}^2\right] &\leq \sum_{i\neq j} \frac{\eta_{t,i}}{2}\sqrt{\mathbb{E}[w_{t,i}]} + \frac{\eta_{t,i} + \eta_{t,j}}{2}\mathbb{E}[w_{t,i}].
\end{align*}
\end{lemma}
\begin{proof}
First notice that:
\begin{align*}
    &\mathbb{E}\left[\max_{w \in [w_{t},\nabla\Psi_t^*(\nabla\Psi_t(w_{t}) -\hat\ell_t + \alpha\pmb{1}_k)]} \|\hat\ell_t -\alpha\pmb{1}_k\|_{\nabla^2\Psi_t^{-1}(w)}^2\right]\\
    \leq &\mathbb{E}\left[\sum_{i=1}^K \max_{w_i \in [w_{t,i},\nabla\Psi_t^*(\nabla\Psi_t(w_{t}) -\hat\ell_t + \alpha\pmb{1}_k)_i]}\frac{\eta_{t,i}}{2} w_i^{3/2}(\hat\ell_{t,i} - \alpha)^2 \right]
\end{align*}
From the definition of $\nabla \Psi^*(Y)_i$ (Equation~\ref{eq:potential_identities}) we know that $\nabla \Psi^*(Y)_i$ is increasing on $(-\infty, 0]$ and hence for $\alpha = 0$ we have $w_{t,i} \geq \nabla\Psi_t^*(\nabla\Psi_t(w_{t}) -\hat\ell_t)_i$. This implies the maximum of each of the terms is attained at $w_i = w_{t,i}$. Thus
\begin{align*}
    &\mathbb{E}\left[\sum_{i=1}^K \max_{w_i \in [w_{t,i},\nabla\Psi_t^*(\nabla\Psi_t(w_{t}) -\hat\ell_t + \alpha\pmb{1}_k)_i]}\frac{\eta_{t,i}}{2} w_i^{3/2}(\hat\ell_{t,i})^2 \right]\\
    =&\mathbb{E}\left[\sum_{i=1}^K \frac{\eta_{t,i}}{2} w_{t,i}^{3/2}\chi_{(i_t = i)}\frac{\ell_{t,i}^2}{w_{t,i}^2}\right] = \mathbb{E}\left[\sum_{i=1}^K \frac{\eta_{t,i}}{2} w_{t,i}^{3/2} \frac{\ell_{t,i}^2}{w_{t,i}}\right] \leq \sum_{i=1}^K \frac{\eta_{t,i}}{2}\sqrt{\mathbb{E}[w_{t,i}]}.
\end{align*}
When $\alpha = \chi_{(i_t = j)}\ell_{t,j}$ we consider several cases. First if $i_t \neq j$ the same bound as above holds. Next if $i_t = j$ for all $i\neq j$ we have $\hat\ell_{t,i} - \alpha = -\alpha = -\ell_{t,j} \geq -1$ and for $\nabla \Phi_t^*(\nabla \Phi_t(w_t) -\hat\ell_t+ \ell_{t,j}) = \nabla \Phi_t^*(\nabla \Phi_t(w_t) + \ell_{t,j}) \leq 2^{2/3}w_{t,i}$ by Lemma~\ref{lem:lemma_16_zimmer}. This implies that in this case the maximum in the terms is bounded by $2w_{t,i}^{3/2}\ell_{t,j}^2$. Finally if $i_t=j$ for the $j$-th term we again use the fact that $w_{t,j} \geq \nabla\Psi_t^*(\nabla\Psi_t(w_{t}) -\hat\ell_t + \ell_{t,j})_j$ since $-\hat\ell_{t,j} + \ell_{t,j} \leq 0$. Combining all of the above we have
\begin{align*}
    &\mathbb{E}\left[\max_{w \in [w_{t},\nabla\Psi_t^*(\nabla\Psi_t(w_{t}) -\hat\ell_t + \alpha\pmb{1}_k)]} \|\hat\ell_t - \chi_{(i_t = j)}\ell_{t,j}\pmb{1}\|_{\nabla^2\Psi_t^{-1}(w)}^2\right] \leq \sum_{i\neq j} \frac{\eta_{t,i}}{2}\sqrt{\mathbb{E}[w_{t,i}]}\\
    + &\mathbb{E}\left[\chi_{(i_t=j)}\left(\frac{\eta_{t,j}}{2}\left(\frac{\ell_{t,j}}{w_{t,j}} - \ell_{t,j}\right)^2 w_{t,j}^{3/2} +\sum_{i\neq j} \ell_{t,j}^2 \frac{\eta_{t,i}}{2}w_{t,i}^{3/2}\right)\right]\\
    =&\sum_{i\neq j} \frac{\eta_{t,i}}{2}\sqrt{\mathbb{E}[w_{t,i}]} + \mathbb{E}\left[\frac{\eta_{t,j}}{2}\left(\ell_{t,j}(1 - w_{t,j})\right)^2 w_{t,j}^{1/2} +\sum_{i\neq j} \ell_{t,j}^2 \frac{\eta_{t,i}}{2}w_{t,i}^{3/2}w_{t,j}\right]\\
    \leq&\sum_{i\neq j} \frac{\eta_{t,i} + \eta_{t,j}}{2}\left(\sqrt{\mathbb{E}[w_{t,i}]} + \mathbb{E}[w_{t,i}]\right).
\end{align*}
\end{proof}
Now the stability term is bounded by Lemma~\ref{lem:lemma_11_zimmert}.
Next we proceed to bound the penalty term in a slightly different way. Direct computation yields
\begin{equation}
    \label{eq:stability_dual}
    \begin{aligned}
    D_{\Phi_t}(-\hat L_{t-1}, \nabla \Phi_t^*(u)) - D_{\Phi_t}(-\hat L_{t}, \nabla \Phi_t^*(u)) &=-\Phi_t(-\hat L_t) + \Phi_t(-\hat L_{t-1}) - \langle -\hat L_{t-1} + \hat L_t, u\rangle\\
    &+ \Phi_t(\nabla\Phi_t^*(u)) - \Phi_t(\nabla\Phi_t^*(u))\\
    &=-\Phi_t(-\hat L_t) + \Phi_t(-\hat L_{t-1}) - \langle \hat\ell_t, u\rangle.
    \end{aligned}
\end{equation}
Using the next lemma and telescoping will result in a bound for the sum of the penalty terms
\begin{lemma}
\label{lem:lemma_12_zimmert}
Let $u = e_{i^*}$ be the optimal algorithm. For any $w_{t+1}$ such that $w_{t+1} = \nabla\Phi_{t+1}(-\hat L_t)$ and $\eta_{t+1} \leq \eta_t$ it holds that
\begin{align*}
    D_{\Phi_{t+1}}(-\hat L_{t}, \nabla \Phi_{t+1}^*(u)) - D_{\Phi_t}(-\hat L_{t}, \nabla \Phi_t^*(u)) \leq 4\sum_{i\neq i^*}\left(\frac{1}{\eta_{t+1,i}} - \frac{1}{\eta_{t,i}}\right)\left(\sqrt{w_{t+1,i}} - \frac{1}{2}w_{t+1,i}\right).
\end{align*}
\end{lemma}
\begin{proof}
\begin{align*}
D_{\Phi_{t+1}}(-\hat L_{t}, \nabla \Phi_{t+1}^*(u)) - D_{\Phi_t}(-\hat L_{t}, \nabla \Phi_t^*(u)) &= \Phi_{t+1}(-\hat L_t) - \Phi_{t}(-\hat L_t) + \Phi_{t+1}^*(u) - \Phi_t^*(u)\\
&-\langle u, \hat L_t - \hat L_t \rangle\\
&=\Phi_{t+1}(-\hat L_t) - \Phi_{t}(-\hat L_t) + \Psi_{t+1}(u) - \Psi_t(u)\\
&= \Phi_{t+1}(-\hat L_t) - \Phi_{t}(-\hat L_t) - 2\left(\frac{1}{\eta_{t+1,i^*}} - \frac{1}{\eta_{t,i^*}}\right)\\
&= \langle w_{t+1}, -\hat L_t\rangle - \Psi_{t+1}(w_{t+1}) - \Phi_{t}(-\hat L_t)\\
&- 2\left(\frac{1}{\eta_{t+1,i^*}} - \frac{1}{\eta_{t,i^*}}\right)\\
&\leq \langle w_{t+1}, -\hat L_t\rangle - \Psi_{t+1}(w_{t+1}) - \langle w_{t+1}, -\hat L_t \rangle + \Psi_t(w_{t+1})\\
&- 2\left(\frac{1}{\eta_{t+1,i^*}} - \frac{1}{\eta_{t,i^*}}\right)\\
&\leq 4\sum_{i\neq i^*}\left(\frac{1}{\eta_{t+1,i}} - \frac{1}{\eta_{t,i}}\right)\left(\sqrt{w_{t+1,i}} - \frac{1}{2}w_{t+1,i}\right).
\end{align*} 
The first equality holds by Fenchel duality and the definition of Bregman divergence. The second equality holds by the fact that on the simplex $\Phi_t^*(\cdot) = \Psi_t(\cdot)$. The third equality holds because $\Psi_t(u) = -4(\sqrt{1} - \frac{1}{2})$. The fourth equality holds because $w_{t+1}$ is the maximizer of $\langle-\hat L_t,w \rangle + \Psi_{t+1}(w)$ and this is exactly how $\Phi_{t+1}(-\hat L_t)$ is defined. The first inequality holds because 
\begin{align*}
    -\Phi_t(-\hat L_t) &= \max_{w \in \Delta^{K-1}} \langle-\hat L_t,w \rangle + \Psi_{t}(w)\\
    &\leq \langle-\hat L_t,w_{t+1} \rangle + \Psi_t(w_{t+1}).
\end{align*}
The final inequality holds because $\Psi_t(w_{t+1}) - \Psi_{t+1}(w_{t+1}) = 4\sum_{i}(1/\eta_{t+1,i} - 1/\eta_{t,i})(\sqrt{w_{t+1,i}} - w_{t+1,i}/2)$ and the fact that
$\sqrt{w_{t+1,i^*}} - \frac{1}{2}w_{t+1,i^*} \leq \frac{1}{2}$.
\end{proof}

Next we focus on the OMD update. By the 3-point rule for Bregman divergence we can write
\begin{align*}
    \langle \hat \ell_{t}, w_{t} - u \rangle &= \langle \nabla \Psi_t(w_t) - \nabla \Psi_t(\tilde w_{t+1}), w_{t} - u\rangle = D_{\Psi_{t}}(u,w_t) - D_{\Psi_{t}}(u,\tilde w_{t+1}) + D_{\Psi_{t}}(w_t,\tilde w_{t+1})\\
    &\leq D_{\Psi_{t}}(u,w_t) - D_{\Psi_{t}}(u,\hat w_{t+1}) + D_{\Psi_{t}}(w_t,\tilde w_{t+1}),\\
     \langle \hat \ell_{t+1}, \hat w_{t+1} - u \rangle &\leq D_{\Psi_{t+1}}(u,\hat w_{t+1}) - D_{\Psi_{t+1}}(u,\hat w_{t+2}) + D_{\Psi_{t+1}}(\hat w_{t+1},\tilde w_{t+2}),
\end{align*}
where the first inequality follows from the fact that $D_{\Psi_t}(u,\tilde w_{t+1}) \leq D_{\Psi_t}(u,\hat w_{t+1})$ as $\hat w_{t+1}$ is the projection of $\tilde w_{t+1}$ with respect to the Bregman divergence onto $\Delta^{K-1}$.

We now explain how to control each of the terms. First we begin by matching $D_{\Psi_{t+1}}(u, \hat w_{t+1})$ with $-D_{\Psi_t}(u, \hat w_{t+1})$.
\begin{eqnarray}
    D_{\Psi_{t+1}}(u, \hat w_{t+1})-D_{\Psi_t}(u, \hat w_{t+1}) &=& \Psi_{t+1}(u) - \Psi_t(u) + \Psi_t(\hat w_{t+1}) - \Psi_{t+1}(\hat w_{t+1})\nonumber\\
    &&+ \langle \nabla \Psi_t(\hat w_{t+1}),u-\hat w_{t+1}\rangle - \langle \nabla \Psi_{t+1}(\hat w_{t+1}), u - \hat w_{t+1}\rangle\nonumber\\
    &=& -2\left(\frac{1}{\eta_{t+1,i^*}}- \frac{1}{\eta_{t,i^*}}\right) \nonumber\\
    &&- 4\sum_{i} \left(\sqrt{\hat w_{t+1,i}} - \frac{1}{2} \hat w_{t+1,i}\right) \left(\frac{1}{\eta_{t,i}} - \frac{1}{\eta_{t+1,i}} \right)\nonumber\\
    &&-2\left(\frac{1}{\sqrt{\hat w_{t+1,i^*}}} - 1\right)\left(\frac{1}{\eta_{t,i^*}} - \frac{1}{\eta_{t+1,i^*}}\right) \nonumber\\
    &&+ 2\sum_{i} \hat w_{t+1,i} \left(\frac{1}{\sqrt{\hat w_{t+1,i}}} - 1 \right) \left(\frac{1}{\eta_{t,i}} - \frac{1}{\eta_{t+1,i}} \right),\nonumber\\
    &=& 2\left(\frac{1}{\eta_{t,i^*}}- \frac{1}{\eta_{t+1,i^*}}\right) \nonumber\\
    &&-2\left(\frac{1}{\sqrt{\hat w_{t+1,i^*}}} - 1\right)\left(\frac{1}{\eta_{t,i^*}} - \frac{1}{\eta_{t+1,i^*}}\right) \nonumber\\
    &&- 2\sum_{i} \sqrt{\hat w_{t+1,i}}
    \left(\frac{1}{\eta_{t,i}} - \frac{1}{\eta_{t+1,i}} \right),\nonumber
\end{eqnarray}
where we have set $u = e_{i^*}$. Since the step size schedule is non-decreasing during OMD updates, we have that the above is bounded by
\begin{equation}
\label{eq:neg_regret}
    \begin{aligned}
        D_{\Psi_{t+1}}(u, \hat w_{t+1})-D_{\Psi_t}(u, \hat w_{t+1}) &\leq 2\left(\frac{1}{\eta_{t,i^*}}- \frac{1}{\eta_{t+1,i^*}}\right) - 2\left(\frac{1}{\sqrt{\hat w_{t+1,i^*}}} - 1\right)\left(\frac{1}{\eta_{t,i^*}} - \frac{1}{\eta_{t+1,i^*}}\right)\\
        &\leq -2\left(\frac{1}{\sqrt{\hat w_{t+1,i^*}}} - 2\right)\left(\frac{1}{\eta_{t,i^*}} - \frac{1}{\eta_{t+1,i^*}}\right).
    \end{aligned}
\end{equation}

Next we explain how to control the terms $D_{\Psi_t}(w_t,\tilde w_{t+1})$ and $D_{\Psi_{t+1}}(\hat w_{t+1}, \tilde w_{t+2})$. These can be thought of as the stability terms in the FTRL update.

\begin{lemma}
\label{lem:omd_stability}
For iterates generated by the OMD step in Equation~\ref{eq:omd_step} and any $j$ it holds that
\begin{align*}
    \mathbb{E}[D_{\Psi_t}(w_t,\tilde w_{t+1})] &\leq \sum_{i=1}^K \frac{\eta_{t,i}}{2}\sqrt{\mathbb{E}[w_{t,i}]},\\
    \mathbb{E}[D_{\Psi_t}(w_t,\tilde w_{t+1})]&\leq \sum_{i\neq j} \frac{\eta_{t,i}}{2}\sqrt{\mathbb{E}[w_{t,i}]} + \frac{\eta_{t,i} + \eta_{t,j}}{2}\mathbb{E}[w_{t,i}],\\
    \mathbb{E}[D_{\Psi_{t+1}}(\hat w_{t+1},\tilde w_{t+2})] &\leq \sum_{i=1}^K \frac{\eta_{t+1,i}}{2}\sqrt{\mathbb{E}[\hat w_{t+1,i}]},\\
    \mathbb{E}[D_{\Psi_{t+1}}(\hat w_{t+1},\tilde w_{t+2})]&\leq \sum_{i\neq j} \frac{\eta_{t+1,i}}{2}\sqrt{\mathbb{E}[w_{t,i}]} + \frac{\eta_{t+1,i} + \eta_{t+1,j}}{2}\mathbb{E}[w_{t,i}],
\end{align*}
where $\tilde w_{t+1}$ is any iterate such that $\hat w_{t+1} = \argmin_{w \in \Delta^{K-1}}D_{\Psi_t}(w, \tilde w_{t+1})$.
\end{lemma}
\begin{proof}
We show the first two inequalities. The second couple of inequalities follow similarly. First we notice that we have
\begin{align*}
    \hat w_{t+1} = \argmin_{w \in \Delta^{K-1}} \langle w,\hat \ell_t \rangle + D_{\Psi_t}(w,w_{t+1}) = \argmin_{w \in \Delta^{K-1}} \langle w,\hat \ell_t - \alpha \pmb{1}_k \rangle + D_{\Psi_t}(w,w_{t+1}),
\end{align*}
for any $\alpha$. This implies that $\hat w_{t+1} = \argmin_{w \in \Delta^{K-1}}D_{\Psi_t}(w, \tilde w_{t+1})$ for $\tilde w_{t+1} = \argmin_{w \in \mathbb{R}^K} \langle w,\hat \ell_t - \alpha \pmb{1}_k \rangle + D_{\Psi_t}(w,w_{t+1})$. We can now write
\begin{align*}
    D_{\Psi_{t}^*}(\nabla \Psi_{t}(\tilde w_{t+1}) , \nabla \Psi_{t}(w_t)) &=  D_{\Psi_{t}^*}( \nabla\Psi_t(w_t) - \hat \ell_t + \alpha \pmb{1}_k, \nabla \Psi_{t}(w_t))\\
    &\leq \max_{w \in [w_t, \nabla \Psi_t^*(\nabla\Psi_t(w_t) -\hat\ell_t + \alpha \pmb{1}_k)]}\|\hat\ell_t - \alpha \pmb{1}_k \|_{\nabla^2 \Psi_t^{-1}(w)}^2.
\end{align*}
The proof is finished by Lemma~\ref{lem:lemma_11_zimmert}.
\end{proof}

Finally we explain how to control $D_{\Psi_t}(u,w_t)$ and $D_{\Psi_{t+1}}(u,\hat w_{t+2})$. First by Lemma~\ref{lem:duality_eq} it holds that $$D_{\Psi_t}(u,w_t) = D_{\Phi_t}(-L_{t-1}, \nabla \Phi_t^*(u)).$$ This term can now be combined with the term $-D_{\Phi_{t-1}}(-L_{t-1}, \nabla \Phi_{t-1}^*(u))$ coming from the prior FTRL update and both terms can be controlled through Lemma~\ref{lem:lemma_12_zimmert}. To control $-D_{\Psi_{t+1}}(u,\hat w_{t+2})$ we show that $-D_{\Psi_{t+1}}(u,\hat w_{t+2}) = -D_{\Phi_{t+1}}(-\hat L_{t+1}, \nabla \Phi_{t+1}^*(u))$. This is done by showing that if $\hat w_{t+1}$ and $\hat w_{t+2}$ are defined as in Equation~\ref{eq:omd_step} we can equivalently write $\hat w_{t+2}$ as an FTRL step coming from a slightly different loss.

\begin{lemma}
\label{lem:ftrl_omd_equiv}
Let $\hat w_{t+2}$ be defined as in Equation~\ref{eq:omd_step}. Let $\nu_{t+1}$ be the constant such that $\nabla \Phi_{t+1}(-\hat L_t) = \nabla \Psi_{t+1}^*(-\hat L_t + \nu_t\pmb{1}_k)$. Let $\hat L_{t+1} = (\pmb 1_k - \e)\odot(\hat L_{t} - (\nu_{t-1} + \nu_t)\pmb{1}_k) + \frac{1}{\beta}\e\odot((\hat L_t - (\nu_{t-1}+\nu_t)\pmb{1}_k)) + \hat\ell_{t+1}$ and $\eta_{t+2} = \eta_{t+1}$. Then $(\hat L_{t+1})_i \geq 0$ for all $i\in[K]$ and $\hat w_{t+2} = w_{t+2} = \nabla \Phi_{t+2}(-\hat L_{t+1})$.
\end{lemma}
\begin{proof}
By the definition of the update we have 
\begin{align*}
    \hat w_{t+1} &= \nabla \Phi_t(\nabla\Psi_t(w_t) - \hat \ell_t) = \nabla \Phi_t(-\hat L_t + \nu_{t-1}\pmb{1}_k)\\
    &= \nabla \Psi_t^*(-\hat L_t + (\nu_{t-1}+\nu_t)\pmb{1}_k),\\
    \hat w_{t+2} &= \nabla \Phi_{t+1} (\nabla \Psi_{t+1}(\hat w_{t+1}) - \hat\ell_{t+1}),
\end{align*}
where in the first equality we have used the fact that $\nabla \Psi_t(w_t) = -L_{t-1} + \nu_{t-1}\pmb{1}_k$.
For any $i$ such that the OMD update increased the step size, i.e. $\eta_{t+1,i} = \beta \eta_{t,i}$ it holds from the definition of $\nabla \Psi_{t+1}(\cdot)$ that $\nabla \Psi_{t+1}(w)_i = \frac{1}{\beta}\nabla\Psi_{t}(w)_i$. Since $\nabla \Psi_t^*$ inverts $\nabla \Psi_t$ coordinate wise, we can write
\begin{align*}
    \nabla \Psi_{t+1}(\hat w_{t+1})_i = \frac{1}{\beta}\nabla \Psi_{t}(\hat w_{t+1})_i = \frac{1}{\beta}(-\hat L_t + (\nu_{t-1}+\nu_t)\pmb{1}_k)_i.
\end{align*}
If we let $e$ be the the sum of all $e_i$'s such that $\eta_{t+1,i} = \beta \eta_{t,i}$ we can write 
\begin{align*}
    \hat w_{t+2} = \nabla \Phi_{t+1}\left( (\pmb{1}_k - e)\odot(-\hat L_{t} + (\nu_{t-1} + \nu_t)\pmb{1}_k) + \frac{1}{\beta}e\odot((-\hat L_t + (\nu_{t-1}+\nu_t)\pmb{1}_k)) - \hat \ell_{t+1}\right).
\end{align*}
The fact that $\hat L_{t+1,i} \geq 0$ for any $i$ follows since any coordinate $\nabla \Psi_{t}(\hat w_{t+1})_i \leq 0$ which implies that any coordinate of $(-\hat L_t + (\nu_{t-1} + \nu_t)\pmb{1}_k)_i \leq 0$.
\end{proof}
We can finally couple $-D_{\Phi_{t+1}}(-\hat L_{t+1}, \nabla \Phi_{t+1}^*(u))$ with the term from the next FTRL step which is $D_{\Phi_{t+2}}(-\hat L_{t+1}, \nabla \Phi_{t+2}^*(u))$ and use Lemma~\ref{lem:lemma_12_zimmert} to bound the sum of this two terms. Putting everything together we arrive at the following regret guarantee.

\begin{theorem}
\label{thm:corral_bound_stoch}
The regret bound for Algorithm~\ref{alg:tsallis_inf_incr_dcr} for any step size schedule which is non-increasing on the FTRL steps and any $T_0$ satisfies
\begin{align*}
    \mathbb{E}\left[\sum_{t=1}^T \langle \hat \ell_t, w_t - u \rangle \right] &\leq  \sum_{t=T_0+1}^T \sum_{i\neq i^*} \mathbb{E}[\frac{3}{2}\eta_{t,i}\sqrt{ w_{t,i}} + \frac{\eta_{t,i}+\eta_{t,i^*}}{2}w_{t,i}] + \sum_{t=1}^{T_0} \sum_{i=1}^K \mathbb{E}\left[\frac{\eta_{t,i}}{2} \sqrt{w_{t,i}}\right]\\
    &+ \sum_{t\in\cT_{OMD}} \mathbb{E}\left[-2\left(\frac{1}{\sqrt{\hat w_{t+1,i^*}}} - 3\right)\left(\frac{1}{\eta_{t,i^*}} - \frac{1}{\eta_{t+1,i^*}}\right)\right]\\
    &+ \mathbb{E}\left[\Psi_1(u) - \Psi_{1}(w_1)\right] + \mathbb{E}\left[\sum_{t \in [T]\setminus \cT_{OMD}} 4\sum_{i\neq i^*}\left(\frac{1}{\eta_{t,i}} - \frac{1}{\eta_{t-1,i}}\right)\left( \sqrt{w_{t,i}}\right) \right].
\end{align*}
\end{theorem}
\begin{proof}
Let $\cT_{FTRL}$ be the set of all rounds in which the FTRL step is taken except for all rounds immediately before the OMD step and immediately after the OMD step. Let $\cT_{OMD}$ be the set of all round immediately before the OMD step. The regret is bounded as follows:
\begin{align*}
    \mathbb{E}\left[\sum_{t=1}^T \langle \hat \ell_t, w_t - u \rangle \right]&= \sum_{t\in \cT_{FTRL}}\mathbb{E}\left[\langle \hat \ell_t, w_t - u \rangle\right] + \sum_{t\in [T]\setminus \cT_{FTRL}}\mathbb{E}\left[\langle \hat \ell_t, w_t - u \rangle\right]\\
    &= \sum_{t\in [T]\setminus \cT_{FTRL}}\mathbb{E}\left[\langle \hat \ell_t, w_t - u \rangle\right] +\sum_{t \in \cT_{FTRL}} \mathbb{E}\left[\langle \hat \ell_t,w_t \rangle + \Phi_t(-\hat L_t) - \Phi_t(-\hat L_{t-1})\right.\\
    &\left.+ D_{\Phi_t}(-\hat L_{t-1}, \nabla \Phi_t^*(u)) - D_{\Phi_t}(-\hat L_{t}, \nabla \Phi_t^*(u))\right].
\end{align*}
For any $T_0$, by the stability bound in Lemma~\ref{lem:lemma_11_zimmert} we have 
\begin{align*}
    \sum_{t \in \cT_{FTRL}} \mathbb{E}\left[\langle \hat \ell_t,w_t \rangle + \Phi_t(-\hat L_t) - \Phi_t(-\hat L_{t-1})\right] &\leq \sum_{t \in \cT_{FTRL}\bigcap\{[T_0]\}} \sum_{i=1}^K \frac{\eta_{t,i}}{2}\sqrt{\mathbb{E}[w_{t,i}]}\\
    &+ \sum_{t \in \cT_{FTRL}\setminus\{[T_0]\}}\sum_{i\neq i^*} \mathbb{E}[\frac{\eta_{t,i}}{2} (\sqrt{w_{t,i}} + w_{t,i})].
\end{align*}
Next we consider the penalty term
\begin{align*}
    &\sum_{t\in \cT_{FTRL}}\mathbb{E}\left[D_{\Phi_t}(-\hat L_{t-1}, \nabla \Phi_t^*(u)) - D_{\Phi_t}(-\hat L_{t}, \nabla \Phi_t^*(u))\right] = \mathbb{E}\left[D_{\Phi_1}(0, \nabla \Phi_1^*(u))\right]\\
    + &\sum_{t+1\in \cT_{FTRL}}\mathbb{E}\left[D_{\Phi_{t+1}}(-\hat L_{t}, \nabla \Phi_t^*(u)) - D_{\Phi_t}(-\hat L_{t}, \nabla \Phi_t^*(u))\right] - \mathbb{E}\left[\sum_{t \in \cT_{OMD}} D_{\Phi_{t-1}}(-\hat L_{t-1}, \nabla \Phi_{t-1}^*(u))\right]\\
    +&\mathbb{E}\left[\sum_{t \in \cT_{OMD}}D_{\Phi_{t+2}}(-\hat L_{t+1}, \nabla \Phi_{t+2}^*(u))\right] - \mathbb{E}\left[D_{\Phi_{T}}(-\hat L_{T}, \nabla \Phi_{T}^*(u))\right].
\end{align*}
We are now going to complete the penalty term by considering the extra terms which do not bring negative regret from $\sum_{t \in [T]\setminus \cT_{FTRL}} \mathbb{E}[\langle \hat \ell_t,w_t-u\rangle]$.
\begin{align*}
    &\sum_{t \in [T]\setminus \cT_{FTRL}} \mathbb{E}[\langle \hat \ell_t,w_t-u\rangle] \leq \sum_{t \in \cT_{OMD}} \mathbb{E}\left[D_{\Psi_{t}}(u,w_t) - D_{\Psi_{t}}(u,\hat w_{t+1}) + D_{\Psi_{t}}(w_t,\tilde w_{t+1})\right]\\
    +&\sum_{t \in \cT_{OMD}} \mathbb{E}\left[D_{\Psi_{t+1}}(u,\hat w_{t+1}) - D_{\Psi_{t+1}}(u,\hat w_{t+2}) + D_{\Psi_{t+1}}(\hat w_{t+1},\tilde w_{t+2})\right]\\
    +&\sum_{t \in\cT_{OMD}} \mathbb{E}\left[ \langle\hat \ell_{t+2},w_{t+2} \rangle + \Phi_{t+2}(-\hat L_{t+2}) - \Phi_{t+2}(-\hat L_{t+1})\right]\\
    +&\sum_{t \in \cT_{OMD}} \mathbb{E}\left[ D_{\Phi_{t+2}}(-\hat L_{t+1}, \nabla \Phi_{t+2}^*(u)) - D_{\Phi_{t+2}}(-\hat L_{t+2}, \nabla \Phi_{t+2}^*(u)) \right]\\
    =&\sum_{t \in \cT_{OMD}} \mathbb{E}\left[\langle\hat \ell_{t+2},w_{t+2} \rangle + \Phi_{t+2}(-\hat L_{t+2}) - \Phi_{t+2}(-\hat L_{t+1}) +  D_{\Psi_{t}}(w_t,\tilde w_{t+1}) + D_{\Psi_{t+1}}(\hat w_{t+1},\tilde w_{t+2})\right]\\
    +&\sum_{t \in \cT_{OMD}} \mathbb{E}\left[D_{\Psi_{t+1}}(u,\hat w_{t+1}) - D_{\Psi_{t}}(u,\hat w_{t+1})\right]\\
    +&\sum_{t \in \cT_{OMD}} \mathbb{E}\left[ D_{\Phi_t}(-\hat L_{t-1},\nabla \Phi_t^*(u)) - D_{\Phi_{t+2}}(-\hat L_{t+2}, \nabla \Phi_{t+2}^*(u))\right]\\
    +&\sum_{t \in \cT_{OMD}} \mathbb{E}\left[D_{\Phi_{t+2}}(-\hat L_{t+1}, \nabla \Phi_{t+2}^*(u)) - D_{\Psi_{t+1}}(u,\hat w_{t+2})\right],
\end{align*}
where in the first inequality we have used the 3-point rule for Bregman divergence and the definition of the set $\tau_{FTRL}$. For any $T_0$ the term 
\begin{align*}
    \sum_{t \in \cT_{OMD}} \mathbb{E}\left[\langle\hat \ell_{t+2},w_{t+2} \rangle + \Phi_{t+2}(-\hat L_{t+2}) - \Phi_{t+2}(-\hat L_{t+1}) +  D_{\Psi_{t}}(w_t,\tilde w_{t+1}) + D_{\Psi_{t+1}}(\hat w_{t+1},\tilde w_{t+2})\right]
\end{align*} 
is bounded by Lemma~\ref{lem:lemma_11_zimmert} and Lemma~\ref{lem:omd_stability} as follows
\begin{align*}
    &\sum_{t \in \cT_{OMD}} \mathbb{E}\left[\langle\hat \ell_{t+2},w_{t+2} \rangle + \Phi_{t+2}(-\hat L_{t+2}) - \Phi_{t+2}(-\hat L_{t+1}) +  D_{\Psi_{t}}(w_t,\tilde w_{t+1}) + D_{\Psi_{t+1}}(\hat w_{t+1},\tilde w_{t+2})\right]\\
    \leq &\sum_{t \in \cT_{OMD}\setminus \{[T_0]\}} \sum_{i\neq i^*} \mathbb{E}[\eta_{t+2,i}\sqrt{w_{t+2,i}} + \frac{\eta_{t+2,i}+\eta_{t+2,i^*}}{2}w_{t+2,i}] + \sum_{t \in \cT_{OMD}\bigcap \{[T_0]\}}\sum_{i=1}^K \frac{\eta_{t,i}}{2}\sqrt{\mathbb{E}[w_{t+2,i}]},
\end{align*}
where we have used the $i\neq i^*$ bound from the above lemmas for all terms past $T_0$ and the bound which includes all $i \in [K]$ for the first $T_0$ terms. The term $\sum_{t \in \cT_{OMD}} \mathbb{E}\left[D_{\Psi_{t+1}}(u,\hat w_{t+1}) - D_{\Psi_{t}}(u,\hat w_{t+1})\right]$ is bounded from Equation~\ref{eq:neg_regret} as follows
\begin{align*}
    \sum_{t \in \cT_{OMD}} \mathbb{E}\left[D_{\Psi_{t+1}}(u,\hat w_{t+1}) - D_{\Psi_{t}}(u,\hat w_{t+1})\right] \leq \sum_{t\in\cT_{OMD}} \mathbb{E}\left[-2\left(\frac{1}{\sqrt{\hat w_{t+1,i^*}}} - 2\right)\left(\frac{1}{\eta_{t,i^*}} - \frac{1}{\eta_{t+1,i^*}}\right)\right].
\end{align*}
By Lemma~\ref{lem:ftrl_omd_equiv} and Lemma~\ref{lem:duality_eq}
\begin{align*}
&\sum_{t \in \cT_{OMD}} \mathbb{E}\left[D_{\Phi_{t+2}}(-\hat L_{t+1}, \nabla \Phi_{t+2}^*(u)) - D_{\Psi_{t+1}}(u,\hat w_{t+2})\right]\\
= &\sum_{t \in \cT_{OMD}} \mathbb{E}\left[D_{\Phi_{t+2}}(-\hat L_{t+1}, \nabla \Phi_{t+2}^*(u)) - D_{\Phi_{t+1}}(-\hat L_{t+1}, \nabla \Phi_{t+1}^*(u))\right].
\end{align*}
Combining all of the above we have
\begin{equation}
\label{eq:half_way_there}
\begin{aligned}
    \mathbb{E}\left[\sum_{t=1}^T \langle \hat \ell_t, w_t - u \rangle \right] &\leq  \sum_{t=T_0+1}^T \sum_{i\neq i^*} \mathbb{E}[\frac{3}{2}\eta_{t,i}\sqrt{ w_{t,i}} + \frac{\eta_{t,i}+\eta_{t,i^*}}{2}w_{t,i}] + \sum_{t=1}^{T_0} \sum_{i=1}^K \mathbb{E}\left[\frac{\eta_{t,i}}{2} \sqrt{w_{t,i}}\right]\\
    &+ \sum_{t\in\cT_{OMD}} \mathbb{E}\left[-2\left(\frac{1}{\sqrt{\hat w_{t+1,i^*}}} - 2\right)\left(\frac{1}{\eta_{t,i^*}} - \frac{1}{\eta_{t+1,i^*}}\right)\right]\\
    &+\sum_{t \in [T]\setminus \cT_{OMD}}\mathbb{E}\left[D_{\Phi_{t+1}}(-\hat L_{t}, \nabla \Phi_{t+1}^*(u)) - D_{\Phi_t}(-\hat L_{t}, \nabla \Phi_t^*(u))\right]\\
    &+ \mathbb{E}[D_{\Phi_1}(0,\nabla\Phi_1^*(u))] - \mathbb{E}[D_{\Phi_T}(-\hat L_T, \nabla \Phi_T^*(u))].
\end{aligned}
\end{equation}
Using Lemma~\ref{lem:lemma_12_zimmert} we have that
\begin{align*}
    &\sum_{t \in [T]\setminus \cT_{OMD}}\mathbb{E}\left[D_{\Phi_{t+1}}(-\hat L_{t}, \nabla \Phi_{t+1}^*(u)) - D_{\Phi_t}(-\hat L_{t}, \nabla \Phi_t^*(u))\right] \leq\\
    &\sum_{t \in [T]\setminus \cT_{OMD}}\mathbb{E}\left[4\sum_{i\neq i^*} \left(\frac{1}{\eta_{t+1,i}} - \frac{1}{\eta_{t,i}}\right)\left(\sqrt{w_{t+1,i}} - \frac{1}{2}w_{t+1,i}\right)\right].
\end{align*}
By definition of $w_1$ we have $D_{\Phi_1}(0,\nabla \Phi_1^*(u)) = \Psi_1(u) - \Psi_1(w_1)$.
Plugging back into Equation~\ref{eq:half_way_there} we have
\begin{align*}
    \mathbb{E}\left[\sum_{t=1}^T \langle \hat \ell_t, w_t - u \rangle \right] &\leq  \sum_{t=T_0+1}^T \sum_{i\neq i^*} \mathbb{E}[\frac{3}{2}\eta_{t,i}\sqrt{ w_{t,i}} + \frac{\eta_{t,i}+\eta_{t,i^*}}{2}w_{t,i}] + \sum_{t=1}^{T_0} \sum_{i=1}^K \mathbb{E}\left[\frac{\eta_{t,i}}{2} \sqrt{w_{t,i}}\right]\\
    &+ \sum_{t\in\cT_{OMD}} \mathbb{E}\left[-2\left(\frac{1}{\sqrt{\hat w_{t+1,i^*}}} - 3\right)\left(\frac{1}{\eta_{t,i^*}} - \frac{1}{\eta_{t+1,i^*}}\right)\right]\\
    &+ \mathbb{E}\left[\Psi_1(u) - \Psi_{1}(w_1)\right] + \mathbb{E}\left[\sum_{t \in [T]\setminus \cT_{OMD}} 4\sum_{i\neq i^*}\left(\frac{1}{\eta_{t,i}} - \frac{1}{\eta_{t-1,i}}\right)\left( \sqrt{w_{t,i}}\right) \right].
\end{align*}
\end{proof}
The algorithm begins by running each algorithm for $\log{T}+1$ rounds. We set the probability thresholds so that $\rho_1 = 36$, $\rho_j = 2 \rho_{j-1}$ and $\frac{1}{\rho_n} \geq \frac{1}{KT}$, because we mix each $w_t$ with the uniform distribution weighted by $1/KT$. This implies $n \leq \operatorname{log}_2(T)$. The algorithm now proceeds in epochs.
The sizes of the epochs are as follows. The first epoch was of size $K\log{T}+K$, each epoch after doubles the size of the preceding one so that the number of epochs is bounded by $\log{T}$. In the beginning of each epoch, except for the first epoch we check if $w_{t,i} < \frac{1}{\rho_1}$. If it is we increase the step size $\eta_{t+1,i} = \beta \eta_{t,i}$ and run the OMD step. Let the $\tau$-th epoch have size $s_{\tau}$. Let $\frac{1}{\rho_\tau}$ be the largest threshold which was not exceeded during epoch $\tau$. We require that each of the algorithms have the following expected regret bound under the unbiased rescaling of the losses $\bar R_{i}(t)$: $\mathbb{E}[\bar R_{i}(\sum_{\tau=1}^S s_{\tau})] \leq \sum_{\tau=1}^S \mathbb{E}[\sqrt{\rho_\tau} R(s_\tau)]$. This can be ensured by restarting the algorithms in the beginning of the epochs if at the beginning of epoch $\tau$ it happens that $w_{t,i} > \frac{1}{\rho_{\tau-1}}$. Let $\ell_t$ be the loss over all possible actions. Let $i_t$ be the algorithm selected by the corralling algorithm at time $t$. Let $a^*$ be the best overall action.

\begin{lemma}
\label{lem:regret_bound}
Let $\bar R_{i^*}(\cdot)$ be a function upper bounding the expected regret of $\cA_{i^*}$, $\mathbb{E}[R_{i^*}(\cdot)]$. For any $\eta$ such that $\eta_{1,i} \leq \min_{t\in [T]} \frac{\left(1-\exp{-\frac{1}{\log{T}^2}}\right)\sqrt{t}}{50\bar R_{i}(t)}, \forall i \in [K]$ it holds that
\begin{align*}
    &\mathbb{E}\left[\sum_{t=1}^T \ell_t(a_{i_t,j_t}) - \ell_t(a^*) \right] \leq  \sum_{t=T_0+1}^T \sum_{i\neq i^*} \mathbb{E}[\frac{3}{2}(\eta_{t,i}+\eta_{t,i^*})(\sqrt{ w_{t,i}} + w_{t,i})] + \sum_{t=1}^{T_0} \sum_{i=1}^K \mathbb{E}\left[\frac{\eta_{t,i}}{2} \sqrt{w_{t,i}}\right]\\
    + &\mathbb{E}\left[\Psi_1(u) - \Psi_{1}(w_1)\right] + \mathbb{E}\left[\sum_{t \in [T]\setminus \cT_{OMD}} 4\sum_{i\neq i^*}\left(\frac{1}{\eta_{t,i}} - \frac{1}{\eta_{t-1,i}}\right)\left( \sqrt{w_{t,i}}\right) \right] + 1 + 36\mathbb{E}[R_{i^*}(T)].
\end{align*}
\end{lemma}
\begin{proof}
First we note that $\mathbb{E}[\hat \ell_t (i^*)] = \mathbb{E}\left[w_{i^*,t}\frac{\ell_t(a_{i^*,j_t})}{w_{i^*,t}}\right] =\mathbb{E}[\ell_t(a_{i^*,j_t})]$. Using Theorem~\ref{thm:corral_bound_stoch} we have
\begin{align*}
    \sum_{t=1}^T \mathbb{E}\left[ \ell_t(a_{i_t,j_t}) - \ell_t(a^*) \right] &= \sum_{t=1}^T\mathbb{E}\left[ \ell_t(a_{i^*,j_t}) - \ell_t(a^*) \right] + \sum_{t=1}^T\mathbb{E}\left[\langle \hat \ell_t, \bar w_t - u \rangle\right] \leq \sum_{t=1}^T\mathbb{E}\left[ \hat \ell_t(i^*) -  \ell_t(a^*) \right]\\
    &+ \sum_{t=T_0+1}^T \sum_{i\neq i^*} \mathbb{E}[\frac{3}{2}(\eta_{t,i}+\eta_{t,i^*})(\sqrt{ w_{t,i}} + w_{t,i})] + \sum_{t=1}^{T_0} \sum_{i=1}^K \mathbb{E}\left[\frac{\eta_{t,i}}{2} \sqrt{w_{t,i}}\right]\\
    &+ \sum_{t\in\cT_{OMD}} \mathbb{E}\left[-2\left(\frac{1}{\sqrt{\hat w_{t+1,i^*}}} - 3\right)\left(\frac{1}{\eta_{t,i^*}} - \frac{1}{\eta_{t+1,i^*}}\right)\right]\\
    &+ \mathbb{E}\left[\Psi_1(u) - \Psi_{1}(w_1)\right] + \mathbb{E}\left[\sum_{t \in [T]\setminus \cT_{OMD}} 4\sum_{i\neq i^*}\left(\frac{1}{\eta_{t,i}} - \frac{1}{\eta_{t-1,i}}\right)\left( \sqrt{w_{t,i}}\right) \right] + 1.
\end{align*}
Let us focus on $\sum_{t=1}^T\mathbb{E}\left[ \hat \ell_t(i^*) - \ell_t(a^*) \right] - 2\sum_{t\in\cT_{OMD}} \mathbb{E}\left[\left(\frac{1}{\sqrt{\hat w_{t+1,i^*}}} - 3\right)\left(\frac{1}{\eta_{t,i^*}} - \frac{1}{\eta_{t+1,i^*}}\right)\right]$. 
By our assumption on $\cA_{i^*}$ it holds that
\begin{align*}
    \sum_{t=1}^T\mathbb{E}\left[ \hat \ell_t(i^*) - \ell_t(a^*) \right] \leq \mathbb{E}\left[ R_{i^*}(\sum_{\tau=1}^{\log{T}} s_{\tau})\right] \leq \sum_{\tau=1}^{\log{T}} \mathbb{E}[\sqrt{\rho_\tau} R_{i^*}(s_\tau)].
\end{align*} 
We now claim that during epoch $\tau$ there is a $t$ in that epoch such that also $t \in \cT_{OMD}$ and for which $w_{t,i^*} \leq \frac{1}{\rho_{\tau-1}}$. We consider two cases, first if OMD was invoked because at least one of the probability thresholds $\rho_s$ was passed by a $w_{t_s,i^*}$, we must have $\rho_s \leq \rho_\tau$. Also by definition of $\rho_\tau$ as the largest threshold not passed by any $w_{t,i^*}$ there exists at least one $t' \geq t_s$ for which $\frac{1}{\rho_{\tau-1}} \geq w_{t',i^*} > \frac{1}{\rho_{\tau}}$. 
This implies that we have subtracted at least $ 2\mathbb{E}\left[\left(\frac{1}{\sqrt{\hat w_{t'+1,i^*}}} - 3\right)\left(\frac{1}{\eta_{t',i^*}} - \frac{1}{\eta_{t'+1,i^*}}\right)\right] \geq 2\mathbb{E}\left[\left(\sqrt{\rho_{\tau-1}} - 3\right)\left(\frac{1}{\eta_{t',i^*}} - \frac{1}{\eta_{t'+1,i^*}}\right)\right]$. In the second case we have that for all $t$ in epoch $\tau$ it holds that $\frac{1}{\rho_{\tau-1}} \geq w_{t,i^*} > \frac{1}{\rho_\tau}$ or $w_{t,i^*} > \frac{1}{\rho_1}$. In the second case we only incur regret $\mathbb{E}[R_{1}(t)]$ scaled by $36$ and in the first case the OMD played in the beginning of the epoch has resulted in at least $-2\mathbb{E}\left[\left(\sqrt{\rho_{\tau-1}} - 3\right)\left(\frac{1}{\eta_{t,i^*}} - \frac{1}{\eta_{t+1,i^*}}\right)\right]$ negative contribution, where $t$ indexes the beginning of the epoch. 
We set $\beta = e^{1/\log{T}^2}$ and now evaluate the difference 
$\frac{1}{\eta_{t,i^*}} - \frac{1}{\eta_{t+1,i^*}} \geq \left(1 - \frac{1}{\beta}\right)\frac{\sqrt{t}}{25\eta_{1,i^*}}$. Where we have used the fact that $\eta_{t,i^*} \leq \frac{\eta_{1,i^*} \beta^{\operatorname{log}_2(T)^2}}{\sqrt{t}} \leq \frac{25\eta_{1,i^*}}{\sqrt{t}}$. This follows by noting that there are $\operatorname{log}_2(T)$ epochs and during each epoch one can call the OMD step only $\operatorname{log}_2(T)$ times. Let $\beta' = \left(1- \frac{1}{\beta}\right)$. 
Thus if $t_{\tau}$ is the beginning of epoch $\tau$ we subtract at least $\frac{\beta'\sqrt{t_\tau \rho_{\tau-1}}}{25\eta_{1,i^*}}$. Notice that the length of each epoch $s_{\tau}$ does not exceed $2t_{\tau}$, thus we have
\begin{align*}
    \sum_{\tau=1}^{\log{T}} \mathbb{E}[\sqrt{\rho_\tau} R_{i^*}(s_\tau)] \leq \sum_{\tau=1}^{\log{T}} \mathbb{E}[\sqrt{\rho_\tau} R_{i^*}(2t_\tau)],
\end{align*}
and so as long as we set $\eta_{1,i^*} \leq \frac{\beta'\sqrt{2t_{\tau}}}{50 \bar R_{i^*}(2t_\tau)}$, where $\mathbb{E}[R_{{i^*}}(2t_{\tau})]\leq \bar R_{i^*}(2t_\tau)$ we have
\begin{align*}
    \sum_{\tau=1}^{\log{T}} \mathbb{E}[\sqrt{\rho_\tau} R_{i^*}(2t_\tau)] &- 2\sum_{t\in\cT_{OMD}} \mathbb{E}\left[\left(\frac{1}{\sqrt{\hat w_{t+1,i^*}}} - 3\right)\left(\frac{1}{\eta_{t,i^*}} - \frac{1}{\eta_{t+1,i^*}}\right)\right]\\
    &\leq \sum_{\tau=1}^{\log{T}} \mathbb{E}[\sqrt{\rho_\tau} R_{i^*}(2t_\tau) - \sqrt{\rho_\tau} R_{i^*}(2t_\tau)] \leq 0.
\end{align*}
\end{proof}
We can now use the self-bounding trick of the regret as in~\cite{zimmert2018optimal} to finish the proof. Let $\mu^*$ denote the reward of the best arm. First note that we can write
\begin{eqnarray}
    \mathbb{E}\left[\sum_{t=1}^T\ell_t(a_{i_t,j_t}) - \ell_t(a^*)\right] &=& \mathbb{E}\left[\sum_{t=1}^T \chi_{i_t\neq i^*}(\ell_t(a_{i_t,j_t}) - \mu^*)\right] + \mathbb{E}\left[R_{i^*}(T_{i^*}(T))\right] \nonumber \\
    &\geq& \mathbb{E}\left[\sum_{t=1}^T\sum_{i=1}^K w_{t,i}\chi_{i_t\neq i^*}(\ell_t(a_{i_t,j_t}) - \mu^*)\right] \nonumber \\
    &\geq& \mathbb{E}\left[\sum_{t=1}^T\sum_{i\neq i^*} w_{t,i}\Delta_i\right].\nonumber
\end{eqnarray}
\begin{theorem}
Let $\bar R_{i^*}(\cdot)$ be a function upper bounding the expected regret of $\cA_{i^*}$, $\mathbb{E}[R_{i^*}(\cdot)]$. For any $\eta$ such that $\eta_{1,i} \leq \min_{t\in [T]} \frac{\left(1-\exp{-\frac{1}{\log{T}^2}}\right)\sqrt{t}}{50\bar R_{i}(t)}, \forall i \in [K]$ and $\beta = e^{1/\log{T}^2}$ it holds that the expected regret of Algorithm~\ref{alg:tsallis_inf_incr_dcr} is bounded as
\begin{align*}
    \mathbb{E}\left[R(T)\right] & \leq \sum_{i\neq i^*}\frac{1500(1/\eta_{1,i}+\eta_{1,i})^2}{\Delta_i}\left(\log{\frac{T\Delta_i - 15\eta_{1,i} }{T_0\Delta_i - 15\eta_{1,i}}} + \log{225\eta_{1,i}^2\Delta_i/\Delta_1}\right) \nonumber \\
    & + \sum_{i\in [K]}\frac{8}{\eta_{1,i}\sqrt{K}} + 2 +72R_{i^*}(T),
\end{align*}
where $T_0 = \max_{i\neq i^*}\frac{225\eta_{1,i}^2}{\Delta_i}$.
\end{theorem}
\begin{proof}[Proof of Theorem~\ref{thm:regret_bound}]
By Lemma~\ref{lem:regret_bound} we have that the overall regret is bounded by
\begin{eqnarray}
    \mathbb{E}[R(T)] &\leq& \sum_{t=T_0+1}^T \sum_{i\neq i^*} \mathbb{E}[\frac{3}{2}(\eta_{t,i}+\eta_{t,i^*})(\sqrt{ w_{t,i}} + w_{t,i})] + \sum_{t=1}^{T_0} \sum_{i=1}^K \mathbb{E}\left[\frac{\eta_{t,i}}{2} \sqrt{w_{t,i}}\right] \nonumber \\
    &&+ \mathbb{E}\left[\Psi_1(u) - \Psi_{1}(w_1)\right] + \mathbb{E}\left[\sum_{t \in [T]\setminus \cT_{OMD}} 4\sum_{i\neq i^*}\left(\frac{1}{\eta_{t,i}} - \frac{1}{\eta_{t-1,i}}\right)\left( \sqrt{w_{t,i}}\right) \right] \nonumber \\
    &&+ 1 + 36\mathbb{E}[R_{{i^*}}(T)]\nonumber \\
    &\leq& \sum_{t=T_0+1}^T \sum_{i\neq i^*} \mathbb{E}[\frac{75\eta_{1,i}}{2\sqrt{t}}(\sqrt{ w_{t,i}} + w_{t,i})] + \sum_{t=1}^{T_0} \sum_{i=1}^K \mathbb{E}\left[\frac{25\eta_{1,i}}{2\sqrt{t}} \sqrt{w_{t,i}}\right]\nonumber \\
    &&+ \mathbb{E}\left[\Psi_1(u) - \Psi_{1}(w_1)\right] + \mathbb{E}\left[\sum_{t=1}^T \sum_{i\neq i^*}\frac{10}{\eta_{1,i}\sqrt{t}}\left( \sqrt{w_{t,i}}\right) \right] \nonumber \\
    &&+ 1 + 36\mathbb{E}[R_{{i^*}}(T)]\nonumber \\
    &\leq& \sum_{t=T_0+1}^T \sum_{i\neq i^*} \mathbb{E}[\frac{75\eta_{1,i}}{2\sqrt{t}}(\sqrt{ w_{t,i}} + w_{t,i})] + \sum_{t=1}^{T_0} \sum_{i=1}^K \mathbb{E}\left[\frac{25\eta_{1,i}}{2\sqrt{t}} \sqrt{w_{t,i}}\right]\nonumber \\
    &&+ \mathbb{E}\left[\Psi_1(u) - \Psi_{1}(w_1)\right] + \mathbb{E}\left[\sum_{t=1}^T \sum_{i\neq i^*}\frac{10}{\eta_{1,i}\sqrt{t}}\left( \sqrt{w_{t,i}}\right) \right]\nonumber  \\
    && + 1 + 36\mathbb{E}[R_{{i^*}}(T)] + \mathbb{E}[R(T)] - \mathbb{E}\left[\sum_{t=1}^T \sum_{i\neq i^*} w_{t,i}\Delta_i\right]\nonumber \\
    &\leq& \sum_{t=T_0+1}^T \sum_{i\neq i^*} \mathbb{E}[\frac{75\eta_{1,i}}{\sqrt{t}}(\sqrt{ w_{t,i}} + w_{t,i})] + \sum_{t=1}^{T_0} \sum_{i=1}^K \mathbb{E}\left[\frac{25\eta_{1,i}}{\sqrt{t}} \sqrt{w_{t,i}}\right]\nonumber \\
    &&+ 2\mathbb{E}\left[\Psi_1(u) - \Psi_{1}(w_1)\right] + \mathbb{E}\left[\sum_{t=1}^T \sum_{i\neq i^*}\frac{20}{\eta_{1,i}\sqrt{t}}\left( \sqrt{w_{t,i}}\right) \right]\nonumber  \\
    && + 2 + 3672\mathbb{E}[R_{{i^*}}(T)] - \mathbb{E}\left[\sum_{t=1}^T \sum_{i\neq i^*} w_{t,i}\Delta_i\right].\nonumber
\end{eqnarray}
In the first inequality we used the fact that for any $i$ we have $\eta_{t,i} \leq 25\eta_{1,i}/\sqrt{t}$, in the second inequality we have used the self bounding property derived before the statement of the theorem and in the third inequality we again used the bound on the expected regret $\mathbb{E}[R(T)]$ from the first inequality. We are now going to use the fact that for any $w > 0$ it holds that $2\alpha\sqrt{w} - \beta w \leq \frac{\alpha^2}{\beta}$. For $t\leq T_0$ we have
\begin{align*}
\sum_{t=1}^{T_0}\sum_{i\neq i^*}\left( 20\frac{\sqrt{w_{t,i}}}{\sqrt{t}}\left(\frac{1}{\eta_{1,i}}+\eta_{1,i}\right) - \Delta_i w_{t,i}\right) \leq \sum_{t=1}^{T_0} \sum_{i\neq i^*} \frac{1500(1/\eta_{1,i}+\eta_{1,i})^2}{t\Delta_i}.
\end{align*}
For $t > T_0$ we have 
\begin{align*}
    \sum_{T_0+1}^T \sum_{i\neq i^*}\left( \frac{\sqrt{w_{t,i}}}{\sqrt{t}}\left(\frac{20}{\eta_{1,i}} + 75\eta_{1,i}\right) - \left(\!\!\Delta_i - \frac{15\eta_{1,i}}{\sqrt{t}}\right)w_{t,i}\!\!\right) &\leq \sum_{t=T_0+1}^T \sum_{i\neq i^*}\frac{1500(1/\eta_{1,i}+\eta_{1,i})^2}{t\Delta_i - 15\eta_{1,i}\sqrt{t}}\\
    &\leq \sum_{i\neq i^*}\int_{T_0}^T \frac{1500(1/\eta_{1,i}+\eta_{1,i})^2}{t\Delta_i - 15\eta_{1,i}\sqrt{t}} dt\\
    &= \frac{1500(1/\eta_{1,i}+\eta_{1,i})^2}{\Delta_i}\log{\frac{15\eta_{1,i} - T\Delta_i}{15\eta_{1,i} - T_0\Delta_i}}.
\end{align*}
We now choose $T_0 = \max_{i\neq i^*}\frac{225\eta_{1,i}^2}{\Delta_i}$. To bound $\mathbb{E}\left[\Psi_1(u) - \Psi_{1}(w_1)\right]$ we have set $w_1$ to be the uniform distribution over the $K$ algorithms and recall that $\Psi_1(w) = -4\sum_{i}\frac{\sqrt{w_i} - \frac{1}{2}w_i}{\eta_{1,i}}$. This implies $\Psi_1(u) - \Psi_{1}(w_1) \leq \sum_{i\in [K]}\frac{4}{\eta_{1,i}\sqrt{K}}$. Putting everything together we have 
\begin{align*}
    \mathbb{E}\left[R(T)\right] & \leq \sum_{i\neq i^*}\frac{1500(1/\eta_{1,i}+\eta_{1,i})^2}{\Delta_i}\left(\log{\frac{T\Delta_i - 15\eta_{1,i} }{T_0\Delta_i - 15\eta_{1,i}}} + \log{225\eta_{1,i}^2/\Delta_i}\right) \nonumber \\
    & + \sum_{i\in [K]}\frac{8}{\eta_{1,i}\sqrt{K}} + 2 +72R_{i^*}(T)
\end{align*}
\end{proof}

To parse the above regret bound in the stochastic setting we note that the min-max regret bound for the $k$-armed problem is $\Theta(\sqrt{k T})$. Most popular algorithms like UCB, Thompson sampling and mirror descent have a regret bound which is (up to poly-logarithmic factors) $O(\sqrt{k T})$. If we were to corral only such algorithms, the condition of the theorem implies that $\frac{1}{\eta_{1,i}} \in \tilde O(\sqrt{k})$ as $\frac{\sqrt{t}}{\bar R_i(t)} \leq O(1),\forall t\in [T]$. What happens, however, if algorithm $\cA_i$ has a worst case regret bound of the order $\omega(\sqrt{T})$? For the next part of the discussion we only focus on time horizon dependence. As a simple example suppose that $\cA_i$ has worst case regret of $T^{2/3}$ and that $\cA_{i^*}$ has a worst case regret of $\sqrt{T}$. In this case Theorem~\ref{thm:corral_bound_stoch} tells us that we should set $\eta_{1,i} = \tilde O(1/T^{1/6})$ and hence the regret bound scales at least as $\Omega(T^{1/3}/\Delta_i + \mathbb{E}[R_{i^*}(T)])$. In general if the worst case regret bound of $\cA_i$ is in the order of $T^{\alpha}$ we have a regret bound scaling at least as $T^{2\alpha-1}/\Delta_{i}$.

\subsection{Stability of UCB and UCB-like algorithms under a change of environment}
In this section we discuss how the regret bounds for UCB and similar algorithms change whenever the variance of the stochastic losses is rescaled by Algorithm~\ref{alg:tsallis_inf_incr_dcr}. Assume that the UCB algorithm plays against stochastic rewards bounded in $[0,1]$. We begin by noting that after every call to $\texttt{OMD-STEP}$ (Algorithm~\ref{alg:omd_step}) the UCB algorithm should be restarted with a change in the environment which reflects that the variance of the losses has now been rescaled. Let the UCB algorithm of interest be $\cA_i$. If the OMD step occurred at time $t'$ and it was the case that $\frac{1}{\rho_{s-1}} \geq w_{t',i} > \frac{1}{\rho_s}$, then we know that the rescaled rewards will be in $[0,\rho_s]$ until the next time the UCB algorithm is restarted. This suggests that the confidence bound for arm $j$ at time $t$ should become $\sqrt{\frac{\rho_s^2\log{t}}{T_{i,j}(t)}}$. However, we note that the second moment of the rescaled rewards is only $\frac{\ell_t(a_{i,j_t})^2}{w_{t,i}}$. A slightly more careful analysis using Bernstein's inequality for martingales (e.g. Lemma 10~\cite{bartlett2008high}) allows us to show the following.
\begin{theorem}[Theorem \ref{thm:ucb_stable} formal]
Suppose that during epoch $\tau$ of size $\cT$ UCB-I is restarted and its environment was changed by $\rho_s$ so that the upper confidence bound is changed to $\sqrt{\frac{4\rho_s\log{t}}{T_{i,j}(t)}} +\frac{4\rho_s\log{t}}{3T_{i,j}(t)}$ for arm $j$ at time $t$. Then the expected regret of the algorithm is bounded by
\begin{align*}
    \mathbb{E}[R_{i}(\cT)] \leq \sqrt{8\rho_s k_i\cT\log{\cT}}
\end{align*}
\end{theorem}
\begin{proof}[Proof of Theorem~\ref{thm:ucb_stable}]
Let the reward of arm $j$ at time $t$ be $r_{t,j}$ and the rescaled reward be $\hat r_{t,j}$. Without loss of generality assume that the arm with highest reward is $j=1$. Denote the mean of arm $j$ as $\mu_j$ and denote the mean of the best arm as $\mu^*$. During this run of UCB we know that each $|\hat r_{t,j}| \leq \rho_s$. Further if we denote the probability with which the algorithm is sampled at time $t$ as $w_{t,i}$ we have $\mathbb{E}[\hat r_{t,j} -  \mu_j |w_{1:t-1,i}] = 0$ and hence $r_{t,j} -  \mu_j$ is a martingale difference. Further notice that the conditional second moment of $r_{t,j}$ is $\mathbb{E}[\hat r_{t,j}^2|w_{1:t-1,i}] = \mathbb{E}[w_{t,i}\frac{r_{t,j}^2}{w_{t,i}^2} + 0|w_{1:t-1,i}] \leq \rho$. Let $Y_t = (\hat r_{\tau,j} - \mu_{j})$. Bernstein's inequality for martingales (\cite{bartlett2008high}[Lemma 10]) now implies that $\prob{\sum_{t=1}^\cT Y_t > \sqrt{2\cT\rho\log{1/\delta}} + \frac{2}{3}\rho\log{1/\delta}} \leq \delta$. This implies that the confidence bound should be changed to
\begin{align*}
    \sqrt{\frac{4\rho_s\log{t}}{T_{i,j}(t)}} +\frac{4\rho_s\log{t}}{3T_{i,j}(t)}.
\end{align*}
Following the standard proof of UCB we can now conclude that a suboptimal arm can be pulled at most $T_{i,j}(t)$ times up to time $t$ where
\begin{align*}
    2\Delta_j \geq \sqrt{\frac{4\rho_s\log{t}}{T_{i,j}(t)}} +\frac{4\rho_s\log{t}}{3T_{i,j}(t)}.
\end{align*}
This implies that 
\begin{align*}
    \mathbb{E}\left[T_{i,j}(t)\right] \leq \frac{8\rho_s\log{t}}{\Delta_j^2}.
\end{align*}
Next we bound the regret of the algorithm up to time $t$ as follows:
\begin{align*}
    \mathbb{E}[R_{i}(t)] &\leq \sum_{j\neq 1} \Delta_j \mathbb{E}\left[T_{i,j}(t)\right] = \sum_{j\neq j^*}\sqrt{\mathbb{E}\left[T_{i,j}(t)\right]}\sqrt{\Delta_j^2\mathbb{E}\left[T_{i,j}(t)\right]}\\
    &\leq\sum_{j\neq j^*}\sqrt{\mathbb{E}\left[T_{i,j}(t)\right]}\sqrt{8\rho_s\log{t}} \leq k_i\sqrt{\frac{1}{k_i}\sum_{j}\mathbb{E}[T_{i,j}(t)]} = \sqrt{8\rho_s k_i t \log{t}}.
\end{align*}
\end{proof}
In general the argument can be repeated for other UCB-type algorithms (e.g. Successive Elimination) and hinges on the fact that the rescaled rewards $\hat r_{t,j}$ have second moment bounded by $\rho$ since with probability $w_{t,i}$ we have $\hat r_{t,j}^2 = \frac{r_{t,j}^2}{w_{t,i}^2}$ and with probability $1-w_{t,i}$ it equals $\hat r_{t,j}^2 = 0$. We are not sure if similar arguments can be carried out for more delicate versions of UCB, like KL-UCB and leave it as future work to check.

\section{Regret bound in the adversarial setting}
\label{app:adv_regret_bounds}
We now consider the setting in which the best overall arm does not maintain a gap at every round. 
Following the proof of Theorem~\ref{thm:corral_bound_stoch} we are able to show the following.
\begin{theorem}
\label{thm:corral_bound_adv}
The regret bound for Algorithm~\ref{alg:tsallis_inf_incr_dcr} for any step size schedule which is non-increasing on the FTRL steps satisfies
\begin{align*}
    \mathbb{E}\left[\sum_{t=1}^T \langle \hat \ell_t, w_t - u \rangle \right] &\leq 4\max_{w\in \Delta^{K-1}}\sqrt{T}\sum_{i=1}^K\left(\eta_{1,i}+\frac{1}{\eta_{1,i}}\right)\sqrt{w_i}\\
    &+\sum_{t\in\cT_{OMD}} \mathbb{E}\left[-2\left(\frac{1}{\sqrt{\hat w_{t+1,i^*}}} - 3\right)\left(\frac{1}{\eta_{t,i^*}} - \frac{1}{\eta_{t+1,i^*}}\right)\right].
\end{align*}
\end{theorem}
\begin{proof}
From the proof of Theorem~\ref{thm:corral_bound_stoch} we have
\begin{align*}
    \mathbb{E}\left[\sum_{t=1}^T \langle \hat \ell_t, w_t - u \rangle \right]&= \sum_{t\in [T]\setminus \cT_{FTRL}}\mathbb{E}\left[\langle \hat \ell_t, w_t - u \rangle\right] +\sum_{t \in \cT_{FTRL}} \mathbb{E}\left[\langle \hat \ell_t,w_t \rangle + \Phi_t(-\hat L_t) - \Phi_t(-\hat L_{t-1})\right.\\
    &\left.+ D_{\Phi_t}(-\hat L_{t-1}, \nabla \Phi_t^*(u)) - D_{\Phi_t}(-\hat L_{t}, \nabla \Phi_t^*(u))\right].
\end{align*}
Lemma~\ref{lem:lemma_11_zimmert} implies 
\begin{align*}
    \sum_{t \in \cT_{FTRL}} \mathbb{E}\left[\langle \hat \ell_t,w_t \rangle + \Phi_t(-\hat L_t) - \Phi_t(-\hat L_{t-1})\right] &\leq \sum_{t \in \cT_{FTRL}} \sum_{i=1}^K \frac{\eta_{t,i}}{2}\sqrt{\mathbb{E}[w_{t,i}]}.
\end{align*}
As before the penalty term is decomposed as follows
\begin{align*}
    &\sum_{t\in \cT_{FTRL}}\mathbb{E}\left[D_{\Phi_t}(-\hat L_{t-1}, \nabla \Phi_t^*(u)) - D_{\Phi_t}(-\hat L_{t}, \nabla \Phi_t^*(u))\right] = \mathbb{E}\left[D_{\Phi_1}(0, \nabla \Phi_1^*(u))\right]\\
    + &\sum_{t+1\in \cT_{FTRL}}\mathbb{E}\left[D_{\Phi_{t+1}}(-\hat L_{t}, \nabla \Phi_t^*(u)) - D_{\Phi_t}(-\hat L_{t}, \nabla \Phi_t^*(u))\right] - \mathbb{E}\left[\sum_{t \in \cT_{OMD}} D_{\Phi_{t-1}}(-\hat L_{t-1}, \nabla \Phi_{t-1}^*(u))\right]\\
    +&\mathbb{E}\left[\sum_{t \in \cT_{OMD}}D_{\Phi_{t+2}}(-\hat L_{t+1}, \nabla \Phi_{t+2}^*(u))\right] - \mathbb{E}\left[D_{\Phi_{T}}(-\hat L_{T}, \nabla \Phi_{T}^*(u))\right].
\end{align*}
Next the term $\sum_{t\in [T]\setminus \cT_{FTRL}}\mathbb{E}[\langle \hat \ell_t,w_t-u\rangle]$ is again decomposed as in the proof of Theorem~\ref{thm:corral_bound_stoch}
\begin{align*}
    &\sum_{t \in [T]\setminus \cT_{FTRL}} \mathbb{E}[\langle \hat \ell_t,w_t-u\rangle]\\
    \leq&\sum_{t \in \cT_{OMD}} \mathbb{E}\left[\langle\hat \ell_{t+2},w_{t+2} \rangle + \Phi_{t+2}(-\hat L_{t+2}) - \Phi_{t+2}(-\hat L_{t+1}) +  D_{\Psi_{t}}(w_t,\tilde w_{t+1}) + D_{\Psi_{t+1}}(\hat w_{t+1},\tilde w_{t+2})\right]\\
    +&\sum_{t \in \cT_{OMD}} \mathbb{E}\left[D_{\Psi_{t+1}}(u,\hat w_{t+1}) - D_{\Psi_{t}}(u,\hat w_{t+1})\right]\\
    +&\sum_{t \in \cT_{OMD}} \mathbb{E}\left[ D_{\Phi_t}(-\hat L_{t-1},\nabla \Phi_t^*(u)) - D_{\Phi_{t+2}}(-\hat L_{t+2}, \nabla \Phi_{t+2}^*(u))\right]\\
    +&\sum_{t \in \cT_{OMD}} \mathbb{E}\left[D_{\Phi_{t+2}}(-\hat L_{t+1}, \nabla \Phi_{t+2}^*(u)) - D_{\Psi_{t+1}}(u,\hat w_{t+2})\right].
\end{align*}
Using Lemma~\ref{lem:lemma_11_zimmert} and Lemma~\ref{lem:omd_stability} we bound the first term of the above inequality as
\begin{align*}
    &\sum_{t \in \cT_{OMD}} \mathbb{E}\left[\langle\hat \ell_{t+2},w_{t+2} \rangle + \Phi_{t+2}(-\hat L_{t+2}) - \Phi_{t+2}(-\hat L_{t+1}) +  D_{\Psi_{t}}(w_t,\tilde w_{t+1}) + D_{\Psi_{t+1}}(\hat w_{t+1},\tilde w_{t+2})\right]\\
    \leq &\sum_{t \in \cT_{OMD}}\sum_{i=1}^K \frac{\eta_{t,i}}{2}\sqrt{\mathbb{E}[w_{t+2,i}]}
\end{align*}
The term $\sum_{t \in \cT_{OMD}} \mathbb{E}\left[D_{\Psi_{t+1}}(u,\hat w_{t+1}) - D_{\Psi_{t}}(u,\hat w_{t+1})\right]$ is bounded from Equation~\ref{eq:neg_regret} as follows
\begin{align*}
    \sum_{t \in \cT_{OMD}} \mathbb{E}\left[D_{\Psi_{t+1}}(u,\hat w_{t+1}) - D_{\Psi_{t}}(u,\hat w_{t+1})\right] \leq \sum_{t\in\cT_{OMD}} \mathbb{E}\left[-2\left(\frac{1}{\sqrt{\hat w_{t+1,i^*}}} - 2\right)\left(\frac{1}{\eta_{t,i^*}} - \frac{1}{\eta_{t+1,i^*}}\right)\right].
\end{align*}
By Lemma~\ref{lem:ftrl_omd_equiv} and Lemma~\ref{lem:duality_eq}
\begin{align*}
&\sum_{t \in \cT_{OMD}} \mathbb{E}\left[D_{\Phi_{t+2}}(-\hat L_{t+1}, \nabla \Phi_{t+2}^*) - D_{\Psi_{t+1}}(u,\hat w_{t+2})\right]\\
= &\sum_{t \in \cT_{OMD}} \mathbb{E}\left[D_{\Phi_{t+2}}(-\hat L_{t+1}, \nabla \Phi_{t+2}^*) - D_{\Phi_{t+1}}(-\hat L_{t+1}, \nabla \Phi_{t+2}^*)\right].
\end{align*}
Combining all of the above we have
\begin{equation}
\label{eq:half_way_there2}
\begin{aligned}
    \mathbb{E}\left[\sum_{t=1}^T \langle \hat \ell_t, w_t - u \rangle \right] &\leq \sum_{t=1}^{T} \sum_{i=1}^K \mathbb{E}\left[\frac{\eta_{t,i}}{2} \sqrt{w_{t,i}}\right] + \sum_{t\in\cT_{OMD}} \mathbb{E}\left[-2\left(\frac{1}{\sqrt{\hat w_{t+1,i^*}}} - 2\right)\left(\frac{1}{\eta_{t,i^*}} - \frac{1}{\eta_{t+1,i^*}}\right)\right]\\
    &+\sum_{t \in [T]\setminus \cT_{OMD}}\mathbb{E}\left[D_{\Phi_{t+1}}(-\hat L_{t}, \nabla \Phi_t^*(u)) - D_{\Phi_t}(-\hat L_{t}, \nabla \Phi_t^*(u))\right]\\
    &+ \mathbb{E}[D_{\Phi_1}(0,\nabla\Phi_1^*(u))] - \mathbb{E}[D_{\Phi_T}(-\hat L_T, \nabla \Phi_T^*(u))].
\end{aligned}
\end{equation}
The last two terms are bounded in the same way as in the proof of Theorem~\ref{thm:corral_bound_stoch}
\begin{align*}
    &\!\!\!\!\!\!\!\!\!\!\!\!\!\!\!\!\sum_{t \in [T]\setminus \cT_{OMD}}\mathbb{E}\left[D_{\Phi_{t+1}}(-\hat L_{t}, \nabla \Phi_t^*(u)) - D_{\Phi_t}(-\hat L_{t}, \nabla \Phi_t^*(u))\right]\\
     &\qquad\qquad\qquad + \mathbb{E}[D_{\Phi_1}(0,\nabla\Phi_1^*(u))] - \mathbb{E}[D_{\Phi_T}(-\hat L_T, \nabla \Phi_T^*(u))]\\
    \leq& \mathbb{E}\left[\Psi_1(u) - \Psi_{1}(w_1)\right] + \mathbb{E}\left[\sum_{t \in [T]\setminus \cT_{OMD}} 4\sum_{i\neq i^*}\left(\frac{1}{\eta_{t,i}} - \frac{1}{\eta_{t-1,i}}\right)\left( \sqrt{w_{t,i}}\right) \right]
\end{align*}
Plugging back into Equation~\ref{eq:half_way_there2} we have
\begin{align*}
    \mathbb{E}\left[\sum_{t=1}^T \langle \hat \ell_t, w_t - u \rangle \right] &\leq \sum_{t=1}^{T} \sum_{i=1}^K \mathbb{E}\left[\frac{\eta_{t,i}}{2} \sqrt{w_{t,i}}\right] + \mathbb{E}\left[\Psi_1(u) - \Psi_{1}(w_1)\right]\\
    &+ 4\mathbb{E}\left[\sum_{t \in [T]\setminus \cT_{OMD}} \sum_{i=1}^K\left(\frac{1}{\eta_{t,i}} - \frac{1}{\eta_{t-1,i}}\right)\left( \sqrt{w_{t,i}}\right) \right]\\
    &+ \sum_{t\in\cT_{OMD}} \mathbb{E}\left[-2\left(\frac{1}{\sqrt{\hat w_{t+1,i^*}}} - 3\right)\left(\frac{1}{\eta_{t,i^*}} - \frac{1}{\eta_{t+1,i^*}}\right)\right]\\
    &\leq \sum_{t=1}^T 4\sum_{i=1}^K\left(\eta_{1,i}+\frac{1}{\eta_{1,i}}\right)\sqrt{\frac{w_{t,i}}{t}}\\
    &+\sum_{t\in\cT_{OMD}} \mathbb{E}\left[-2\left(\frac{1}{\sqrt{\hat w_{t+1,i^*}}} - 3\right)\left(\frac{1}{\eta_{t,i^*}} - \frac{1}{\eta_{t+1,i^*}}\right)\right]\\
    &\leq 4\max_{w\in \Delta^{K-1}}\sqrt{T}\sum_{i=1}^K\left(\eta_{1,i}+\frac{1}{\eta_{1,i}}\right)\sqrt{w_i}\\
    &+\sum_{t\in\cT_{OMD}} \mathbb{E}\left[-2\left(\frac{1}{\sqrt{\hat w_{t+1,i^*}}} - 3\right)\left(\frac{1}{\eta_{t,i^*}} - \frac{1}{\eta_{t+1,i^*}}\right)\right],
\end{align*}
where the last inequality follows from the fact that the maximizer of the function $\sum_{i=1}^K\sqrt{\frac{w_i}{t}}\alpha_i$ over the simplex, for $\alpha_i \geq 0$ is the same for all $t\in[T]$.
\end{proof}

Following the proof of Lemma~\ref{lem:regret_bound} and replacing the bound on $\mathbb{E}\left[\sum_{t=1}^T\langle \hat\ell_t,w_t - u \rangle\right]$ from Theorem~\ref{thm:corral_bound_stoch} with the one from Theorem~\ref{thm:corral_bound_adv} yields the next result.
\begin{theorem}[Theorem~\ref{thm:regret_bound_adv}]
Let $\bar R_{i^*}(\cdot)$ be a function upper bounding the expected regret of $\cA_{i^*}$, $\mathbb{E}[R_{i^*}(\cdot)]$. For any $\eta_{1,i^*} \leq \min_{t\in [T]} \frac{\left(1-\exp{-\frac{1}{\log{T}^2}}\right)\sqrt{t}}{50\bar R_{i^*}(t)}$ and $\beta = e^{1/\log{T}^2}$ it holds that the expected regret of Algorithm~\ref{alg:tsallis_inf_incr_dcr} is bounded as
\begin{align*}
    \mathbb{E}\left[\sum_{t=1}^T \ell_t(a_t) - \ell_t(a^*)\right] \leq 4\max_{w\in \Delta^{K-1}}\sqrt{T}\sum_{i=1}^K\left(\eta_{1,i}+\frac{1}{\eta_{1,i}}\right)\sqrt{w_i} + 36R_{i^*}(T).
\end{align*}
\end{theorem}

A few remarks are in order. First, when the rewards obey the stochastically constrained adversarial setting i.e., there exists a gap $\Delta_i$ at every round between the best action and every other action during all rounds $t\in [T]$, then the regret for corralling bandit algorithms with worst case regret bounds of the order $\tilde O(\sqrt{T})$ in time horizon is at most $\tilde O(\sum_{i\neq i^*} \frac{\log{T}^5}{\Delta_i} + R_{i^*}(T))$. On the other hand, if there is no gap in the rewards then a worst case regret bound is still $\tilde O(\max\{\sqrt{KT}, \max_{i}\bar R_i(T)\} + R_{i^*}(T))$. This implies that Algorithm~\ref{alg:tsallis_inf_incr_dcr} can be used as a model selection tool when we are not sure what environment we are playing against. For example, if we are not sure if we should use a contextual bandit algorithm, a linear bandit algorithm or a stochastic multi-armed bandit algorithm, we can corral all of them and Algorithm~\ref{alg:tsallis_inf_incr_dcr} will perform almost as well as the algorithm for the best environment. Further, if we are in a distributed setting where we have access to multiple algorithms of the same type but not the arms they are playing, we can do almost as well as an algorithm which plays on all the arms simultaneously. We believe that our algorithm will have numerous other applications outside of the scope of the above examples.

\section{Proof of Theorem~\ref{thm:model_selection}}
\label{app:model_selection}
Recall the gap assumption made in Theorem~\ref{thm:model_selection}:
\begin{assumption}
\label{assump:feature_gap}
For any $i < i^*$ it holds that for all $(x,a)\in \mathcal{X}\times\mathcal{A}$
    \begin{align*}
        \mathbb{E}[\langle \beta_i, \phi_i(x,a)\rangle] - \min_{a\in\mathcal{A}}\mathbb{E}[\langle \beta^*, \phi_{i^*}(x,a) \rangle] \geq 2\frac{d_{i^*}^{2\alpha} - d_{i}^{2\alpha}}{\sqrt{T}}.
    \end{align*}
\end{assumption}

Since the losses might not be bounded in $[0,1]$ as $d_K = \Theta(T)$ we need to slightly modify the bound for the Stability term in Lemma~\ref{lem:lemma_11_zimmert} and the term $D_{\Psi_t}(w_t,\tilde w_{t+1})$ in Lemma~\ref{lem:omd_stability}. Recall that we need to bound the term $\mathbb{E}\left[\max_{w \in [w_{t},\nabla\Psi_t^*(\nabla\Psi_t(w_{t}) -\hat\ell_t + \alpha\pmb{1}_k)]} \|\hat\ell_t\|_{\nabla^2\Psi_t^{-1}(w)}^2\right]$. The argument is the same as in~\ref{lem:lemma_11_zimmert} up to 
\begin{align*}
    \mathbb{E}\left[\max_{w \in [w_{t},\nabla\Psi_t^*(\nabla\Psi_t(w_{t}) -\hat\ell_t + \alpha\pmb{1}_k)]} \|\hat\ell_t\|_{\nabla^2\Psi_t^{-1}(w)}^2\right] \leq \mathbb{E}\left[\sum_{i=1}^K \frac{\eta_{t,i}}{2}w_{t,i}^{3/2}(\hat\ell_{t,i})^2\right].
\end{align*}
Let $\ell_{t,i} = \langle \beta_{i_t}, \phi_{i_t}(x_t,a_{i_t,j_t} \rangle + \xi_t)$, then we have
\begin{align*}
    \mathbb{E}\left[\frac{\eta_{t,i}}{2}w_{t,i}^{3/2}(\hat\ell_{t,i})^2\right] \leq  \mathbb{E}\left[\eta_{t,i}\chi_{(i_t=i)}w_{t,i}^{3/2}\frac{\ell_{t,i}^2}{w_{t,i}^2} + \eta_{t,i}w_{t,i}^{3/2}\frac{d_i^{4\alpha}}{T}\right] \leq \mathbb{E}\left[\eta_{t,i}w_{t,i}\frac{d_i^{4\alpha}}{T}\right] + 2\mathbb{E}\left[\eta_{t,i}\sqrt{w_{t,i}}\right],
\end{align*}
where in the last inequality we have used the fact that $w_{t,i} \geq w_{t,i}^{3/2}$ together with the our assumption that $\xi_{t}$ is zero-mean with variance proxy $1$. Following the proof of Lemma~\ref{lem:regret_bound} with the bound on the stability term we can bound
\begin{align*}
    \mathbb{E}\left[\sum_{t=1}^T \ell_t(a_{i_t,j_t}) - \ell_t(a^*)\right] &= \sum_{t=1}^T \mathbb{E}\left[\ell_t(a_{i^*,j_t}) - \ell_t(a^*)\right] + \sum_{t=1}^T \mathbb{E}\left[\langle\hat\ell_t + \mathbf{d}, w_t - u\rangle\right] - \sum_{t=1}^T\mathbb{E}\left[\langle\mathbf{d}, w_t - u\rangle\right] + 1\\
    &\leq  \sum_{t=1}^T \mathbb{E}\left[\hat\ell_t(i^*) - \ell_t(a^*)\right] + 2\sum_{t=1}^T \sum_{i=1}^K \mathbb{E}\left[\eta_{t,i}\sqrt{w_{t,i}}+ \eta_{t,i}w_{t,i}\frac{d_{i}^{4\alpha}}{T}\right] + 4\sum_{t=1}^T\sum_{i=1}^K \left(\frac{1}{\eta_{t,i}} - \frac{1}{\eta_{t-1,i}}\right)\sqrt{w_{t,i}}\\
    &-\sum_{t=1}^T\mathbb{E}[\langle \mathbf{d}, w_t - u\rangle] -\sum_{t\in\cT_{OMD}} \mathbb{E}\left[2\left(\frac{1}{\sqrt{\hat w_{t+1,i^*}}} - 3\right)\left(\frac{1}{\eta_{t,i^*}} - \frac{1}{\eta_{t+1,i^*}}\right)\right] + \sqrt{K}+1\\
    &\leq 4\sum_{t=1}^T\sum_{i=1}^K \sqrt{\frac{w_{t,i}}{t}}\left(\eta_{1,i}+\frac{1}{\eta_{t,i}}\right) + 2\sum_{t=1}^T\sum_{i=1}^K\mathbb{E}\left[\eta_{t,i}w_{t,i}\frac{d_i^{4\alpha}}{T}\right]-\sum_{t=1}^T\mathbb{E}[\langle \mathbf{d}, w_t - u\rangle] + 36\mathbb{E}[R_{i^*}(T)].
\end{align*}
For a fixed $t$ we have
\begin{align*}
    -\langle \mathbf{d},w_t - u \rangle = \frac{d_{i^*}^{2\alpha}}{\sqrt{T}}(1-w_{t,{i^*}}) - \sum_{i\neq i^*}w_{t,i}\frac{d_i^{2\alpha}}{\sqrt{T}} = \sum_{i<i^*} w_{t,i}\frac{d_{i^*}^{2\alpha}-d_{i}^{2\alpha}}{\sqrt{T}} - \sum_{i>i^*} w_{t,i} \frac{d_{i}^{2\alpha} - d_{i^*}^{2\alpha}}{\sqrt{T}}.
\end{align*}
First we consider the terms $i>i^*$. Assume WLOG that $d_K^{2\alpha} \leq T/4$, as otherwise the learning guarantees are trivial. For these terms we have
\begin{align*}
    \sqrt{\frac{w_{t,i}}{t}}\frac{1}{\eta_{1,i}} + w_{t,i}\left(\frac{\eta_{1,i}}{\sqrt{t}}\frac{d_{i}^{4\alpha}}{T} - \frac{d_i^{2\alpha}}{\sqrt{T}}\right) \leq \sqrt{\frac{w_{t,i}}{t}}\frac{1}{\eta_{1,i}} - w_{t,i}\frac{d_i^{2\alpha}}{2\sqrt{T}} \leq \frac{\sqrt{T}}{t d_i^{2\alpha} \eta_{1,i}^2}.
\end{align*}
Since $\eta_{1,i} = \tilde\Theta(1/d_{i}^{\alpha})$ we have that the above is further bounded by $\tilde O(\sqrt{T}/t)$.

Next we consider the terms for $i<i^*$ given by $w_{t,i}\frac{d_{i^*}^{2\alpha}-d_{i}^{2\alpha}}{\sqrt{T}}$. Here we use our assumption that the regret $\mathbb{E}\left[\sum_{t=1}^T \ell_t(a_{i_t,j_t}) - \ell_t(a^*)\right] \geq \mathbb{E}\left[w_{t,i}\Delta_i\right]$, where $\Delta_i = \mathbb{E}[\langle \beta_i, \phi_i(x,a)\rangle] - \min_{a\in\mathcal{A}}\mathbb{E}[\langle \beta^*, \phi_{i^*}(x,a) \rangle]$. Using the self-bounding trick we can cancel out the terms $w_{t,i}\frac{d_{i^*}^{2\alpha}-d_i^{2\alpha}}{\sqrt{T}}$ as soon as $\Delta_i \geq 2w_{t,i}\frac{d_{i^*}^{2\alpha}-d_i^{2\alpha}}{\sqrt{T}}$, which holds by Assumption~\ref{assump:feature_gap}. All other terms in the regret bound are bounded by $\tilde O(d_{i^*}^{\alpha}\sqrt{T})$. Thus we have shown that the regret of the corralling algorithm is bounded as
\begin{align*}
    \mathbb{E}\left[\sum_{t=1}^T \ell_t(a_{i_t,j_t}) - \ell_t(a^*)\right] \leq \tilde O\left(\mathbb{E}[R_{i^*}(T)] + K\sqrt{T}\right).
\end{align*}

\end{document}